\documentclass{article}

\PassOptionsToPackage{numbers}{natbib}
\usepackage[final]{neurips_2020}

\usepackage[utf8]{inputenc} 
\usepackage[T1]{fontenc}    
\usepackage{hyperref}       
\usepackage{url}            
\usepackage{booktabs}       
\usepackage{amsfonts}       
\usepackage{nicefrac}       
\usepackage{microtype}      
\usepackage{subfigure}

\usepackage{enumitem}

\let\zz\[\let\zzz\]

\usepackage{enumitem}
\usepackage{url}
\usepackage{graphicx} 
\usepackage{bbm}

\usepackage{ifthen}
\usepackage{xargs}

\usepackage{nicefrac}
\usepackage{amssymb,mathrsfs} 

\usepackage{yfonts}
\usepackage{bbm}
\usepackage{amsmath}
\usepackage[utf8]{inputenc} 
\usepackage[T1]{fontenc}    
\usepackage{upgreek}
\usepackage{url}            
\usepackage{booktabs}       
\usepackage{amsfonts}       
\usepackage{nicefrac}       
\usepackage{microtype}      
\usepackage{amsmath}
\usepackage{graphicx}
\usepackage{hyperref}
\usepackage{wrapfig}

\usepackage{aliascnt}
\usepackage{cleveref}
\usepackage{autonum}

\let\[\zz\let\]\zzz

\usepackage{amsthm}

\makeatletter
\newtheorem{theorem}{Theorem}
\crefname{theorem}{theorem}{Theorems}
\Crefname{Theorem}{Theorem}{Theorems}

\newtheorem{proposition}{Proposition}
\crefname{proposition}{proposition}{propositions}
\Crefname{Proposition}{Proposition}{Propositions}

\newtheorem{lemma}{Lemma}
\crefname{lemma}{lemma}{lemmas}
\Crefname{Lemma}{Lemma}{Lemmas}

\newtheorem{corollary}{Corollary}
\crefname{corollary}{corollary}{corollaries}
\Crefname{Corollary}{Corollary}{Corollaries}

\newtheorem{definition}{Definition}
\crefname{definition}{definition}{definitions}
\Crefname{Definition}{Definition}{Definitions}

\newtheorem{remark}{Remark}
\crefname{remark}{remark}{remarks}
\Crefname{Remark}{Remark}{Remarks}

\crefname{figure}{figure}{figures}
\Crefname{Figure}{Figure}{Figures}

\newtheorem{assumptionC}{\; \textbf{C}\hspace{-3pt}}
\crefformat{assumptionC}{{\textbf{C}}#2#1#3}

\newcommand{\Rd}{\mathbb{R}^d}

\newcommand{\ie}{\textit{i.e.}}

\newcommand{\calA}{\mathcal{A}}

\newcommand{\calF}{\mathcal{F}}

\newcommand{\calN}{\mathcal{N}}

\newcommand{\calP}{\mathcal{P}}

\newcommand{\Var}{\mathrm{Var}}

\newcommand{\bbM}{\mathbb{M}}
\newcommand{\bbN}{\mathbb{N}}

\newcommand{\bbR}{\mathbb{R}}

\newcommand{\rmi}{\mathrm{i}}
\newcommand{\rme}{\mathrm{e}}

\newcommand{\bfD}{\mathbf{D}}

\newcommand{\bfI}{\mathbf{I}}

\def\msa{\mathsf{A}}

\def\msk{\mathsf{K}}

\def\msf{\mathsf{F}}
\def\msg{\mathsf{G}}
\def\msh{\mathsf{H}}

\def\msx{\mathsf{X}}

\def\msy{\mathsf{Y}}

\def\mcbb{\mathcal{B}}  
\newcommand{\mcb}[1]{\mathcal{B}(#1)}

\def\mcp{\mathcal{P}}

\def\mbm{\mathbb{M}}
\def\rset{\mathbb{R}}

\def\nset{\mathbb{N}}
\def\nsets{\mathbb{N}^*}

\def\Nset{\mathbb{N}}

\def\rmd{\mathrm{d}}

\def\rml{\mathrm{L}}

\def\rme{\mathrm{e}}

\def\mrC{\mathrm{C}}

\newcommand{\abs}[1]{\left\vert #1 \right\vert}
\newcommand{\absLigne}[1]{\vert #1 \vert}

\newcommandx{\Vnorm}[2][1=V]{\| #2 \|_{#1}}
\newcommandx{\normpi}[2][2=2]{\left\Vert  #1 \right\Vert_{#2}}
\newcommandx{\normH}[2][2=2]{\left\Vert  #1 \right\Vert}
\newcommandx{\normHLigne}[2][2=2]{\Vert  #1 \Vert}
\newcommandx{\normHLine}[2][2=2]{\Vert  #1 \Vert}
\newcommandx{\normmu}[2][2=2]{\left\Vert  #1 \right\Vert_{#2}}
\newcommandx{\normopmu}[2][2=2]{\left\vvvert  #1 \right\vvvert_{#2}}
\newcommandx{\normopH}[2][2=2]{\left\vvvert  #1 \right\vvvert}
\newcommandx{\normop}[2][2=2]{\left\vvvert  #1 \right\vvvert}
\newcommand{\ps}[2]{\left\langle#1,#2 \right\rangle}

\newcommandx{\normpiLine}[2][2=2]{\Vert  #1 \Vert_{#2}}
\newcommandx{\normmuLine}[2][2=2]{\Vert  #1 \Vert_{#2}}
\newcommandx{\normopmuLine}[2][2=2]{\vvvert  #1 \vvvert_{#2}}
\newcommandx{\normopHLine}[2][2=2]{\vvvert  #1 \vvvert}
\newcommandx{\normopLine}[2][2=2]{\vvvert  #1 \vvvert}

\newcommandx{\VnormEq}[2][1=V]{\left\| #2 \right\|_{#1}}
\newcommandx{\norm}[2][1=]{\ifthenelse{\equal{#1}{}}{\left\Vert #2 \right\Vert}{\left\Vert #2 \right\Vert^{#1}}}
\newcommandx{\normLigne}[2][1=]{\ifthenelse{\equal{#1}{}}{\Vert #2 \Vert}{\Vert #2\Vert^{#1}}}
\newcommandx{\norminf}[2][1=]{\ifthenelse{\equal{#1}{}}{\left\Vert #2 \right\Vert_{\infty}}{\left\Vert #2 \right\Vert^{#1}_{\infty}}}

\newcommand{\parenthese}[1]{\left(#1 \right)}
\newcommand{\parentheseLigne}[1]{(#1 )}
\newcommand{\parentheseDeux}[1]{\left[ #1 \right]}
\newcommand{\defEns}[1]{\left\lbrace #1 \right\rbrace }

\newcommandx\probaMarkovTilde[2][2=]
{\ifthenelse{\equal{#2}{}}{{\widetilde{\mathbb{P}}_{#1}}}{\widetilde{\mathbb{P}}_{#1}\left[ #2\right]}}

\newcommand{\PE}{\mathbb{E}}

\newcommand{\expe}[1]{\PE \left[ #1 \right]}

\newcommandx{\expeE}[2][2=]{\mathbb{E}^{#2}\left[ #1 \right]}
\newcommand{\expeLigne}[1]{\PE [ #1 ]}

\newcommand{\plusinfty}{+\infty}
\def\ie{\textit{i.e.}}

\def\eqsp{\;}
\newcommand{\coint}[1]{\left[#1\right)}

\newcommand{\ccint}[1]{\left[#1\right]}

\newcommandx{\ball}[3][3=]{\ifthenelse{\equal{#3}{}}{\operatorname{B}(#1,#2)}{\operatorname{B}_{#3}(#1,#2)}}

\def\TV{\mathbf{TV}}
\newcommandx\sequence[3][2=,3=]
{\ifthenelse{\equal{#3}{}}{\ensuremath{\{ #1_{#2}\}}}{\ensuremath{\{ #1_{#2}, \eqsp #2 \in #3 \}}}}
\newcommandx\sequenceD[3][2=,3=]
{\ifthenelse{\equal{#3}{}}{\ensuremath{\{ #1_{#2}\}}}{\ensuremath{( #1)_{ #2 \in #3} }}}

\newcommandx{\sequencen}[2][2=n\in\nset]{\ensuremath{( #1)_{#2}}}
\newcommandx{\sequencek}[2][2=k\in\nset]{\ensuremath{( #1)_{#2}}}
\newcommandx{\sequencens}[2][2=n\in\nsets]{\ensuremath{( #1)_{#2}}}
\newcommandx{\sequenceks}[2][2=k\in\nsets]{\ensuremath{( #1)_{#2}}}
\newcommandx\sequenceDouble[4][3=,4=]
{\ifthenelse{\equal{#3}{}}{\ensuremath{\{ (#1_{#3},#2_{#3}) \}}}{\ensuremath{\{  (#1_{#3},#2_{#3}), \eqsp #3 \in #4 \}}}}
\newcommandx{\sequencenDouble}[3][3=n\in\nset]{\ensuremath{\{ (#1_{n},#2_{n}), \eqsp #3 \}}}

\newcommand{\wrt}{w.r.t.}

\def\iid{i.i.d.}

\newcommand{\ensemble}[2]{\left\{#1\,:\eqsp #2\right\}}
\newcommand{\ensembleLigne}[2]{\{#1\,:\eqsp #2\}}
\newcommand{\set}[2]{\ensemble{#1}{#2}}
\newcommand{\setLigne}[2]{\ensembleLigne{#1}{#2}}

\def\veps{\varepsilon}

\def\vareps{\varepsilon}
\def\sphere{\mathbb{S}}
\def\sphereD{\mathbb{S}^{d-1}}

\newcommandx{\CPE}[4][1=,4=]{{\mathbb E}^{#4}_{#1}\left[#2 \, \middle | #3 \right]} %
\newcommandx{\CPELigne}[4][1=,4=]{{\mathbb E}^{#4}_{#1}[#2\, | #3 ]} %
\newcommandx{\CPVar}[3][1=]{\mathrm{Var}^{#3}_{#1}\left\{ #2 \right\}}
\newcommand{\CPP}[3][]
{\ifthenelse{\equal{#1}{}}{{\mathbb P}\left(\left. #2 \, \right| #3 \right)}{{\mathbb P}_{#1}\left(\left. #2 \, \right | #3 \right)}}

\def\distance{\ell}
\def\entropyH{\mathbf{H}}
\newcommandx{\wasserstein}[3][1=\distance,3=]{\mathbf{W}_{#1}^{#3}\left(#2\right)}
\newcommandx{\wassersteinLigne}[3][1=\distance,3=]{\mathbf{W}_{#1}^{#3}(#2)}
\newcommandx{\wassersteinD}[1][1=\distance]{\mathbf{W}_{#1}}
\newcommandx{\wassersteinDLigne}[1][1=\distance]{\mathbf{W}_{#1}}

\newcommandx{\swasserstein}[3][1=\distance,3=]{\mathbf{SW}_{#1}^{#3}\left(#2\right)}
\newcommandx{\swassersteinLigne}[3][1=\distance,3=]{\mathbf{SW}_{#1}^{#3}(#2)}
\newcommandx{\swassersteinD}[1][1=\distance]{\mathbf{SW}_{#1}}
\newcommandx{\swassersteinDLigne}[1][1=\distance]{\mathbf{SW}_{#1}}
\newcommandx{\hatswassersteinD}[1][1=\distance]{\widehat{\mathbf{SW}}_{#1}}

\newcommandx{\sinkD}[1][1=\distance]{\overline{\mathbf{W}}_{#1}}
\newcommandx{\ssinkD}[1][1=\distance]{\overline{\mathbf{SW}}_{#1}}
\newcommandx{\sdelta}[1][1=\distance]{\mathbf{S\Delta}_{#1}}
\newcommandx{\hatsdelta}[1][1=\distance]{\widehat{\mathbf{S\Delta}}_{#1}}
\newcommandx{\maxSdelta}[1][1=\distance]{\max \text{--}~\mathbf{S\Delta}_{#1}}
\newcommandx{\sIpm}[1][1=\distance]{\mathbf{S}\boldsymbol{\gamma}_{#1}}
\newcommandx{\sTV}[1][1=\distance]{\mathbf{S}\TV_{#1}}
\newcommandx{\sMMD}[1][1=\distance]{\mathbf{SMMD}_{#1}}
\newcommandx{\MMD}[1][1=\distance]{\mathbf{MMD}_{#1}}

\newcommandx{\cramerD}[1][1=\distance]{\mathbf{C}_{#1}}
\newcommandx{\scramerD}[1][1=\distance]{\mathbf{SC}_{#1}}

\def\tf{\tilde{f}}
\def\tg{\tilde{g}}
\def\tF{\widetilde{\msf}}
\def\tG{\widetilde{\msg}}

\def\dist{\mathbf{d}}

\def\Leb{\mathrm{Leb}}

\newcommand{\E}{\mathbb{E}}

\def\argmax{\mathrm{argmax}}

\def\ths{\theta}
\def\thss{\theta^{\star}}
\def\thsss{\theta^{\star}_{\sharp}}

\def\thsssl{\theta_{l\sharp}^{\star}}

\def\unifS{\boldsymbol{\sigma}}

\newcommand\fracb[2]{#1/#2}

\def\hmu{\hat{\mu}}
\def\hnu{\hat{\nu}}

\def\bDelta{\mathbf{\Delta}}
\def\bgamma{\boldsymbol{\gamma}}

\def\Lip{\mathrm{Lip}}

\def\supp{\mathrm{supp}}

\def\tu{\tilde{u}}
\def\tv{\tilde{v}}
\def\tX{\tilde{X}}
\def\tY{\tilde{Y}}
\def\tx{\tilde{x}}
\def\ty{\tilde{y}}

\def\diam{\mathrm{diam}}
\def\Lipg{\mathtt{L}}

\title{Statistical and Topological Properties of Sliced Probability Divergences}

\author{%
  Kimia Nadjahi$^1$\thanks{Corresponding author: \texttt{kimia.nadjahi@telecom-paris.fr}}\hspace{1.4mm}, \quad Alain Durmus$^2$, \quad L\'ena\"ic Chizat$^3$, \\
  \textbf{Soheil Kolouri$^4$, \qquad Shahin Shahrampour$^5$, \quad Umut \c{S}im\c{s}ekli$^{1, 6}$} \\[2mm]
  1: LTCI, Télécom Paris, Institut Polytechnique de Paris, France \\
  2: Centre Borelli, ENS Paris-Saclay, CNRS, Université Paris-Saclay, France\\
  3: Laboratoire de Mathématiques d’Orsay, CNRS, Université Paris-Saclay, France \\
  4: HRL Laboratories, LLC., Malibu, CA, USA \\
  5: Texas A\&M University, College Station, TX, USA\\
  6: Department of Statistics, University of Oxford, UK \\
}

\begin{document}

\maketitle

\begin{abstract}
  The idea of slicing divergences has been proven to be successful when comparing two probability measures in various machine learning applications including generative modeling, and consists in computing the expected value of a `base divergence' between \emph{one-dimensional random projections} of the two measures.  However, the topological, statistical, and computational consequences of this technique have not yet been well-established. In this paper, we aim at bridging this gap and derive various theoretical properties of sliced probability divergences. First, we show that slicing preserves the metric axioms and the weak continuity of the divergence, implying that the sliced divergence will share similar topological properties. We then precise the results in the case where the base divergence belongs to the class of integral probability metrics. On the other hand, we establish that, under mild conditions, the sample complexity of a sliced divergence does not depend on the problem dimension. We finally apply our general results to several base divergences, and illustrate our theory on both synthetic and real data experiments.
\end{abstract}

\section{Introduction}

Most inference methods in implicit generative modeling (IGM), such as
generative adversarial networks \cite{goodfellow2014generative} and
variational auto-encoders \cite{kingma2013auto}, rely on the use
of a particular divergence in order to be able to discriminate probability
distributions. Recent advances in this field have illustrated that the
choice of this divergence  is of
crucial importance %
and can lead to very different practical and theoretical properties
\cite{Arjovsky2017,bousquet2017optimal,gulrajani2017improved,tolstikhin2018wasserstein,arbel2019maximum}. In
this context, `sliced' probability divergences, such as
Sliced-Wasserstein \cite{bonneel2015sliced}, or Sliced-Cram\'{e}r
\cite{kolouri2020sliced}, have become increasingly popular.%

This slicing strategy has been essentially motivated by two main purposes. The first purpose is that some probability
divergences are only defined to compare measures supported on one-dimensional spaces (e.g., Cram\'{e}r
distance, \cite{Cramer1928}); hence, the slicing operation allows the use of such divergences to multivariate distributions \cite{Tabor2018,kolouri2020sliced}. The second purpose arises when the computational complexity of a divergence becomes excessive when comparing measures on high-dimensional spaces, but can efficiently be computed in the univariate case (e.g., the Wasserstein distance between one-dimensional distributions admits a closed-form analytical which can easily be approximated). The slicing operation then leverages these advantages originally available in one dimension to define divergences achieving computational efficiency on multivariate settings \cite{bonneel2015sliced,deshpande2019max,paty2019subspace,kolouri2018sliced,Kolouri2019,vayer2019}.

Even though various sliced divergences have successfully been deployed in practical applications, their theoretical properties have not yet been well-understood. Existing results are largely restricted to the specific case of the Sliced-Wasserstein (SW) distance: it has been shown that SW satisfies the metric axioms \cite{Bonnotte2013}, the convergence in SW is equivalent to the convergence in Wasserstein distance \cite{2019arXiv190604516N,bayraktar2019strong}, and the estimators obtained by minimizing SW are consistent \cite{2019arXiv190604516N}. Besides, some properties of SW have only been characterized for specific settings, in particular its statistical benefits observed in practice \cite{deshpande2018generative,deshpande2019max}. %
In this paper, we aim to bridge this gap by investigating the theoretical properties of sliced probability divergences from a general point of view: since such divergences are all characterized via the same slicing operation, we explore in depth the topological and statistical implications of this operation. Specifically, we consider a generic base divergence $\bDelta$ between one-dimensional probability measures, and define its sliced version, denoted by $\sdelta[]$, which operates on multivariate settings.

We first establish several topological properties of $\sdelta[]$. Thanks to our general approach, our findings can directly be applied to any instance of sliced divergence, including those motivated by the two aforementioned purposes. Specifically, we show that slicing preserves the metric properties: if $\bDelta$ is a metric, so is $\sdelta[]$ (Proposition~\ref{thm:metric}). We then focus on finer topological properties of $\sdelta[]$ and show in \Cref{thm:weak_conv} that, if the convergence in $\bDelta$ implies the weak convergence of measures (or conversely), then slicing preserves this property, \ie~the convergence in $\sdelta[]$ implies the weak convergence of measures (or conversely). We also consider the case when $\bDelta$ is an integral probability metric \cite{Muller1997} and identify sufficient conditions for $\sdelta[]$ to be upper-bounded by $\bDelta$, which implies that $\sdelta[]$ induces a weaker topology (Theorem~\ref{thm:weak_topo}). Similarly, we identify sufficient conditions such that $\bDelta$ and $\sdelta[]$ are strongly equivalent (\Cref{cor:ipm_ub2}), meaning that $\bDelta$ is upper- and lower-bounded by $\sdelta[]$. %

Then, we derive the following statistical properties of $\sdelta[]$: we prove that the `sample complexity' of $\sdelta[]$ is proportional to the %
sample complexity of $\bDelta$ for one-dimensional measures, and does not depend on the dimension $d$ (Theorems~\ref{thm:sliced_sample_complexity}, \ref{thm:sliced_rateofconv}). %
This property explains why \emph{any} $\sdelta[]$ motivated by the second purpose offers statistical benefits when the original divergence suffers from the curse of dimensionality. However, this comes with a caveat: we show that, if one approximates the expectation over the random projections that appears in $\sdelta[]$ with a Monte Carlo average, which is the most common practice, then an additional variance term appears in the sample complexity and can limit the performance of $\sdelta[]$ in high dimensions (Theorem~\ref{thm:projection_complexity}). Our results agree with the recent empirical observations reported in \cite{deshpande2019max,Kolouri2019} and provide a better understanding for them.

We illustrate all our theoretical findings on various examples, which demonstrate their applicability.
In particular, our general topological analysis allows us to establish a novel result for the Sliced-Cramér distance. We also derive a sample complexity result for SW which has never been shown before, under different assumptions on the measures to be compared. We then consider Sinkhorn divergences \cite{feydy19}, whose sample complexity is known to have an exponential dependence on the dimension $d$ and regularization parameter $\veps$ \cite{Genevay19}, and introduce its sliced version, referred to as the Sliced-Sinkhorn divergence. We prove that this new divergence has several merits: we derive its sample complexity by combining our general results with recent work \cite{Genevay19,Mena2019}, and obtain rates that do not depend on $d$ nor on $\veps$. We also show that this divergence improves the worst-case computational complexity bounds of Sinkhorn divergences in $\Rd$. Finally, we support our theory with numerical experiments on synthetic and real data.

\vspace{-5pt}
\section{Preliminaries and Technical Background} \label{sec:background}
\vspace{-5pt}

\paragraph{Notations.} For $d \in \bbN^*$, let $\msx$ be a closed and measurable subset of $\Rd$ and $\mcbb(\msx)$ its Borel $\sigma$-algebra for the induced topology. $\calP(\msx)$ stands for the set of probability measures on $(\msx,\mcbb(\msx))$, and $\calP_p(\msx) = \set{ \mu \in \mcp(\msx)}{\int_{\msx} \norm{x}^p \rmd\mu(x) < \plusinfty }$ is the set of probability measures on $(\msx,\mcbb(\msx))$ with finite moment of order $p$. %
Define for any $n \geq 1$, $\hmu_n$ the empirical distribution computed over a sequence of independent and identically distributed (\iid) random variables $\{X_k\}_{k=1}^n$ sampled from $\mu$, by $\hmu_n = (1/n) \sum_{k=1}^n \updelta_{X_k}$, with $\updelta_x$ the Dirac measure at $x$. %
$\mathbb{M}(\msx)$ is the set of real-valued measurable functions on $\msx$, and $\mathbb{M}_b(\msx)$ is the set of bounded functions of $\mathbb{M}(\msx)$. $\sphereD = \set{\theta \in \Rd}{\| \theta \| = 1}$ denotes the unit sphere in $\Rd$, and $\ball{{\bf0}}{R}[d] = \set{x \in \Rd}{ \| x \| < R} $ is the open ball in $\Rd$ of radius $R > 0$ centered around ${\bf0} \in \rset^d$. 

\textbf{Integral Probability Metrics.} 
For any measurable space $\msy$, let $\msf \subset \bbM(\msy)$ and $\mcp_{\msf}(\msy) = \{ \mu \in \mcp(\msy) \, : \, \forall f \in \msf, \, \int_{\msy} \abs{f(y)} \rmd \mu(y) < \plusinfty \}$. The Integral Probability Metric (IPM, \cite{Muller1997}) associated with $\msf$ and denoted by $\bgamma_{\msf}$, is defined for any $\mu,\nu \in \mcp_{\msf}(\msy)$ as
  \begin{equation}
    \label{eq:def_bgamma_F}
          \bgamma_{\msf}(\mu, \nu) =
            \sup_{f \in \msf} \left| \int_{\msy} f(y) \rmd(\mu - \nu)(y)  \right| 
          \eqsp .
  \end{equation}
If $\mu$ or $\nu$ does not belong to $\mcp_{\msf}(\msy)$, we set $\bgamma_{\msf}(\mu,\nu) = \plusinfty$. IPMs are pseudo-metrics \cite{Sriperumbudur09onintegral}: they are non-negative, symmetric, satisfy the triangle inequality and for any $\mu \in \mcp_\msf(\msy)$, $\gamma_\msf(\mu, \mu) = 0$. We recall well-known instances of IPMs below. %

\begin{enumerate}[wide, labelwidth=!, labelindent=0pt, label=(\arabic*), nolistsep]
\setlength\itemsep{.4em}
  \item \emph{Wasserstein distance of order 1.} By the Monge Kantorovich duality theorem \cite[Theorem 5.10]{villani2008optimal}, when $\msf = \setLigne{f : \msy \rightarrow \rset}{\norm{f}_{\Lip} \leq 1}$, where $\norm{f}_{\Lip} = \sup_{x, y \in \msy, x \neq y} \{\fracb{\abs{f(x) - f(y)}}{\norm{x-y}} \}$, $\gamma_\msf$ is the Wasserstein distance of order 1, denoted by $\wassersteinD[1]$. 
  \item \emph{Maximum mean discrepancy.} Let $\msh$ be a reproducing kernel Hilbert space (RKHS) for real-valued functions on $\msy$, and $\msf$ be the unit ball in $\msh$. Then, $\gamma_{\msf}$ defines the MMD in RKHS \cite[Section 2]{Gretton2012}. 
  \item \emph{Cramér distance.} By \cite[Lemma 1]{dedecker_merlevede_2007}, the Cramér distance \cite[Eq.(10)]{kolouri2020sliced} can be written as an IPM.   
\end{enumerate}

In some of our results presented in \Cref{sec:slicedmetrics}, we will assume that the supremum in \eqref{eq:def_bgamma_F} is attained. This property is for example verified for $\wassersteinD[1]$ and MMD, %
by \cite{villani2008optimal} and \cite{Gretton2012} %
respectively.

\textbf{Wasserstein distance, Sinkhorn divergences.}
Arising from the optimal transportation (OT) theory, the Wasserstein distance of order $p \in [1, \infty)$ for any $\mu, \nu \in \calP_p(\Rd)$ %
is defined as \cite[Definition 6.1]{villani2008optimal}
\begin{equation}
\label{eq:def_wasser}
  \wassersteinD[p]^{p}(\mu, \nu) = \inf_{\gamma \in \Gamma(\mu, \nu)} \int_{\Rd \times \Rd} \norm{ x - y }^p \rmd\gamma(x,y) \eqsp,
 \end{equation}
where $\Gamma(\mu, \nu)$ represents the set of probability measures $\gamma$ on $\big(\Rd \times \Rd, \mcb{\Rd \otimes \Rd} \big)$ such that for any $\msa \in \mcb{\Rd}$, $\gamma(\msa \times \Rd) = \mu(\msa)$ and $\gamma(\Rd \times \msa) = \nu(\msa)$. Note that, by strong duality \cite[Theorem 5.10]{villani2008optimal}, $\wassersteinD[1]$ can be characterized by \eqref{eq:def_wasser} or as an IPM \eqref{eq:def_bgamma_F}.

When $\mu$ and $\nu$ are discrete distributions, computing $\wassersteinD[p](\mu, \nu)$ amounts to solving a linear program, so its computational complexity becomes excessive in large-scale applications. By adding an entropic penalization term to \eqref{eq:def_wasser}, one can obtain an approximate solution %
using a simple numerical scheme with significantly lower computational requirements \cite{Cuturi2013}. This yields a regularized Wasserstein cost: for any $\mu, \nu \in \calP_p(\Rd)$ and $\veps \geq 0$,
\begin{equation}
  \wassersteinD[p, \veps](\mu, \nu) = \inf_{\gamma \in \Gamma(\mu, \nu)} \Big\{ \int_{\Rd \times \Rd} \norm{x-y}^p \rmd\gamma(x,y) + \veps \entropyH(\gamma\ |\ \mu\ \otimes\ \nu) \Big\} \eqsp, \label{eq:def_reg_ot}
\end{equation}
where $\entropyH(\gamma\ |\ \mu\ \otimes\ \nu)$ is the relative entropy of the transport plan $\gamma$ with respect to $\mu\ \otimes\ \nu$: if $\gamma$ is absolutely continuous with respect to $\mu\otimes \nu$, $\entropyH(\gamma\ |\ \mu\ \otimes\ \nu) = \int_{\Rd \times \Rd} \log [\left( \rmd \gamma/\rmd \mu\otimes \nu\right)(x, y)] \rmd \gamma(x,y)$,
otherwise, $\entropyH(\gamma\ |\ \mu\ \otimes\ \nu)= \plusinfty$. Building on the regularized Wasserstein cost, \cite{feydy19} defined Sinkhorn divergences for $\mu, \nu \in \calP_p(\Rd)$ and $\veps \geq 0$ as $\sinkD[p, \veps](\mu, \nu) = \wassersteinD[p, \veps](\mu, \nu) - \left\{ \wassersteinD[p, \veps](\mu, \mu) + \wassersteinD[p, \veps](\nu, \nu)\right\}/2$.
These satisfy for any $\mu \in \calP_p(\Rd)$, $\sinkD[p, \veps](\mu, \mu) = 0$ (contrary to $\wassersteinD[p, \veps]$), and %
interpolate between OT (when $\veps \to 0$) and MMD ($\veps \to \infty$). 

\textbf{Sliced-Wasserstein (SW) distance.} When dealing with one-dimensional distributions, \eqref{eq:def_wasser} admits a closed-form solution, which can be efficiently computed. This %
gave rise to another popular tool called SW, which has been successfully used for generative modeling applications \cite{deshpande2018generative,Kolouri_2018_CVPR,csimcsekli2018sliced,wu2017sliced}. %
The main idea is to consider one-dimensional \emph{linear projections} of two high-dimensional measures, then compute the expected $\wassersteinD[p]$ between these %
representations. Formally, the Sliced-Wasserstein distance of order $p \in [1, \infty)$, is defined for any $\mu,\nu \in \mathcal{P}_p(\Rd)$ as 
\begin{equation}
  \swassersteinD[p]^{p}(\mu, \nu) = \int_{\sphereD} \wassersteinD[p]^p(\thss_{\sharp} \mu, \thss_{\sharp} \nu) \rmd\unifS(\ths) \eqsp,
\end{equation}
where $\unifS$ is the uniform distribution on $\sphere^{d-1}$, and for any $\theta \in \sphere^{d-1}$, $\thss : \rset^d \to \rset$ denotes the linear form given by $x \mapsto \ps{\theta}{x}$ with $\ps{\cdot}{\cdot}$ the Euclidean inner-product. For any measurable function $f :\Rd \to \rset$ and $\zeta \in \mcp(\Rd)$, $f_{\sharp}\zeta$ is the push-forward measure of $\zeta$ by $f$, \ie~for any $\msa \in \mcb{\rset}$, $f_{\sharp}\zeta(\msa) = \zeta(f^{-1}(\msa))$, with $f^{-1}(\msa) = \{x \in \Rd \, : \, f(x) \in \msa\}$, 

\vspace{-5pt}
\section{Sliced Probability Divergences} \label{sec:slicedmetrics}
\vspace{-5pt}

In this section, we define the family of Sliced Probability Divergences (SPDs), then we present our theoretical contributions regarding their topological and statistical properties. We provide all the proofs in the supplementary document.

Consider a \emph{`base divergence'} $\bDelta : \calP(\rset) \times \calP(\rset) \to \rset_+ \cup \{\infty\}$ which measures the dissimilarity between two probability measures on $\rset$. We define the Sliced Probability Divergence of order $p \in [1, \infty)$ associated to $\bDelta$, denoted by $\sdelta[p]$, for $\mu, \nu \in \mcp(\Rd)$ as
\begin{equation}
  \label{eq:spd}
  \sdelta[p]^p(\mu, \nu) = \int_{\sphereD} \bDelta^p ( \thsss \mu, \thsss \nu) \rmd\unifS(\theta) \eqsp .
\end{equation}
We assume that $\theta \mapsto \bDelta^p ( \thsss \mu, \thsss \nu)$ is measurable so that \eqref{eq:spd} is well-defined. This can easily be checked if $(\mu',\nu') \mapsto \bDelta (\mu', \nu')$ is continuous for the weak topology on $\mcp(\rset)$, since this implies $\theta \mapsto \bDelta^p ( \thsss \mu, \thsss \nu)$ is continuous. 

In practice, since the integration over $\sphereD$ in \eqref{eq:spd} does not admit an analytical form in general, it is approximated with a simple Monte Carlo scheme (e.g., \cite{bonneel2015sliced,Kolouri2019,vayer2019,kolouri2020sliced}). The Monte Carlo estimate of $\sdelta[p]$ obtained with $L$ random projection directions is defined as 
\begin{equation}
\label{eqref:mc_estimate}
  \hatsdelta[p,L]^p(\mu, \nu) = (1/L) \sum\nolimits_{l=1}^L \bDelta^p\big(\thsssl \mu, \thsssl \nu\big) \eqsp,
\end{equation}
with $\{\theta_l\}_{l=1}^L$ \iid~from $\unifS$ and $\thss_l(x) = \ps{\ths_l}{x}$. Since each term of the sum in \eqref{eqref:mc_estimate} can be computed independently from each other, the approximation of SPDs can be carried out in parallel, which constitutes a nice practical feature. Recent work \cite{paty2019subspace,deshpande2019max} has shown that sampling many projection directions uniformly on the sphere might not be the best strategy, in the sense that some directions can be more helpful than others to discriminate the two distributions at hand. However, the Monte Carlo estimate based on uniform sampling \eqref{eqref:mc_estimate} is the most common method used in practice to approximate sliced divergences, hence we focus on this approximation throughout the rest of the paper. 

\textbf{Topological properties. }
We provide several results to describe the topology induced by SPDs, given the properties of base divergences. We first relate the metric properties of $\bDelta$ and $\sdelta[p]$, $p \in [1, \infty)$.
\begin{proposition}
\label{thm:metric}
  \begin{enumerate}[noitemsep,nolistsep, label=(\roman*), wide, labelwidth=!, labelindent=0pt]
    \item \label{thm:metric_i}  If $\bDelta$ is non-negative (or symmetric), then $\sdelta[p]$ is non-negative (symmetric resp.).
    \item \label{thm:metric_ii} If $\bDelta$ satisfies for $\mu',\nu' \in \calP(\rset)$, $\bDelta(\mu', \nu') = 0$ if and only if $\mu' = \nu'$, then for $\mu,\nu \in \calP(\Rd)$, $\sdelta[p](\mu, \nu) = 0$ if and only if $\mu = \nu$.
    \item \label{thm:metric_iii} If $\bDelta$ is a metric, then $\sdelta[p]$ is a metric.
  \end{enumerate}
\end{proposition}

Next, we extend the result in \cite[Theorem 1]{2019arXiv190604516N}, which showed that the convergence in SW implies the weak convergence of probability measures: we prove that this property holds for the general class of SPDs, but also that the converse implication is true, provided that $\bDelta$ is weakly continuous. Before presenting this result in \Cref{thm:weak_conv}, we recall the definitions of convergence under a probability divergence and weak convergence of probability measures.

\begin{definition} \label{def:cvg}
Let $d \in \nsets$ and $\bfD : \calP(\Rd) \times \calP(\Rd) \to \rset_+ \cup \{\infty\}$ be a probability divergence. Let $\sequencek{\mu_k}$ be a sequence in $\calP(\Rd)$ and $\mu \in \calP(\Rd)$. We introduce two types of convergence below.
\begin{assumptionC} \label{assum:cvg_div}
``$\sequencek{\mu_k}$ converges to $\mu$ under $\bfD$'', \ie~$\lim_{k \rightarrow \plusinfty} \bfD \big( \mu_k, \mu \big ) = 0$.
\end{assumptionC}
\begin{assumptionC} \label{assum:cvg_weak}
``$\sequencek{\mu_k}$ converges weakly to $\mu$'', \ie~for any continuous and bounded function $f : \Rd \to \rset$, 
\vspace{0pt}
$\hspace{26pt} \lim_{k \rightarrow \plusinfty} \int f \rmd \mu_k = \int f \rmd \mu$. 
\end{assumptionC}
\vspace{-5pt}
Hence, the statement ``the convergence under $\bfD$ implies the convergence in $\calP(\Rd)$'' is equivalent to, \cref{assum:cvg_div} implies \cref{assum:cvg_weak} for any $\sequencek{\mu_k}$ and $\mu$ in $\calP(\Rd)$. Conversely, ``the weak convergence in $\calP(\Rd)$ implies the convergence in $\bfD$'' means that \cref{assum:cvg_weak} implies \cref{assum:cvg_div} for any $\sequencek{\mu_k}$ and $\mu$ in $\calP(\Rd)$.
\end{definition}
\vspace{0pt}
\begin{theorem}
\label{thm:weak_conv}
Let $p \in [1, \infty)$ and $\bDelta$ be a non-negative base divergence. \\[-5mm]
  \begin{enumerate}[noitemsep,nolistsep, label=(\roman*), wide, labelwidth=!, labelindent=0pt]
  \item If the convergence under $\bDelta$ implies the weak convergence in $\calP(\rset)$, then the convergence under $\sdelta[p]$ implies the weak convergence in $\calP(\Rd)$. \\[-3mm]
  \item If $\bDelta$ is bounded and the weak convergence in $\calP(\rset)$ implies the convergence under $\bDelta$, then the weak convergence in $\calP(\Rd)$ implies the convergence under $\sdelta[p]$.
\end{enumerate}
\end{theorem}

We now focus on IPMs and formally define Sliced-IPMs, before providing finer topological results.
\begin{definition}
  Let $\tF \subset \mbm_b(\rset)$, $p \in [1, \infty)$. The Sliced Integral Probability Metric %
  of order $p$ associated with $\tF$, denoted by $\sIpm[\tF, p]$, is, for $\mu, \nu \in \mcp(\Rd)$, $(\sIpm[\tF, p])^p(\mu, \nu) = \int_{\sphereD} \bgamma_{\tF}^p(\thsss \mu, \thsss \nu) \rmd \unifS(\theta)$.
\end{definition}
Since $\gamma_{\tF}$ is a pseudo-metric, $\sIpm[\tF, p]$ is a pseudo-metric as well by \Cref{thm:metric}. %
We now identify some regularity conditions on the function classes $\msf$ and $\tF$ such that %
we are able to show that Sliced-IPMs can be bounded above and below by IPMs. %

\begin{theorem} 
\label{thm:weak_topo}
  Let $\tF \subset\mbm_b(\rset)$, $\msf \subset \mbm_b(\rset^d)$, $\set{f : \Rd \rightarrow \rset}{f = \tf \circ \thss, \text{ with } \tf \in \tF, \theta \in \sphereD} \subset \msf$. Then, for any $p \in [1, \infty)$ and $\mu,\nu \in \mcp(\rset^d)$, $\sIpm[\tF, p](\mu, \nu) \leq \bgamma_{\msf}(\mu, \nu)$. 
\end{theorem}
\Cref{thm:weak_topo} states that $\sIpm[\tF, p]$ induces a weaker topology, which is computationally beneficial as argued in \cite{Arjovsky2017}, but also indicates that $\sIpm[\tF, p]$ comes with less discriminative power, which can be restrictive for hypothesis testing applications \cite{Gretton2012}. We now derive a lower-bound on compact domains.
\begin{theorem}
\label{thm:ipm_lb}
  Let $\mu, \nu \in \mathcal{P}(\rset^d)$, with support included in $\ball{{\bf0}}{R}[d]$. Let $\msg \subset \mbm_b(\Rd)$ and suppose that there exists $\Lipg \geq 0$ such that for any $g \in \msg$, $g$ is $\Lipg$-Lipschitz continuous. Consider a class of functions $\tG$ satisfying $\tG \supset \{\tg: \rset \to  \rset \, : \, \text{there exist $x \in \Rd$, $\theta \in \sphere^{d-1}$, $g \in \msg$} \text{ such that } \tg(t)= g(x-\theta t) 
  \text{ for any $t \in \rset$} \}$.
  Furthermore, suppose that $\sIpm[\tG, p]$ is bounded. Then, for any $p \in \coint{1,\plusinfty}$, there exists $C_p > 0$ such that $\ \bgamma_{\msg}(\mu, \nu) \leq C_p~\sIpm[\tG, p](\mu, \nu)^{1/(d+1)}$\ .
\end{theorem}

One can show that the exponent ${1/(d+1)}$ is intrinsic to slicing and hence cannot be avoided. By combining the two theorems, we finally establish a strong equivalence result below, which implies that the convergence of probability measures in $\sIpm[\tG, p]$ is equivalent to the convergence in $\bgamma_{\msg}$. 

\begin{corollary}
\label{cor:ipm_ub2}
  Let $\mu, \nu \in \mathcal{P}(\rset^d)$, with support included in $\ball{{\bf0}}{R}[d]$, and let $\msg \subset \mbm_b(\rset^d)$. Assume that the conditions of \Cref{thm:ipm_lb} are satisfied. Then, for any $p \in \coint{1,\plusinfty}$,  there exists $C_p \geq 0$ independent of $\mu,\nu$ such that \ $\sIpm[\tG, p](\mu, \nu) \leq \bgamma_{\msg}(\mu, \nu) \leq C_p~\sIpm[\tG, p](\mu, \nu)^{1/(d+1)}$.
\end{corollary} 

Our analysis on IPMs builds on \cite[Chapter 5.1]{Bonnotte2013}, which contains analogous results for the Sliced-Wasserstein distance only. %
The novelty of Theorems~\ref{thm:weak_topo} and \ref{thm:ipm_lb} is the identification of the relationships between the function classes $\tF$,$\msf$ and $\tG$,$\msg$, which might provide a useful guideline for practitioners interested in slicing any IPM, and cannot be directly obtained from \cite{Bonnotte2013}.%
We further illustrate these relations in the supplementary document for classical instances of IPMs.

\textbf{Statistical properties. } %
In most practical applications, %
we have at hand finite sets of samples drawn from unknown underlying distributions. An important question is then the bound of the error made when approximating a divergence with finitely many samples: given $\sdelta[p]$ and any $\mu, \nu \in \calP(\Rd)$, our goal is to quantify the \emph{sample complexity} of $\sdelta[p]$, \ie~the convergence rate of $\sdelta[p](\hmu_n, \hnu_n)$ to $\sdelta[p](\mu, \nu)$ according to $n$. We show that the sample complexity of any SPD is proportional to the sample complexity of the base divergence, and more importantly, does not depend on $d$.
\begin{theorem}
\label{thm:sliced_sample_complexity}
  Let $p \in [1, \infty)$. Suppose that $\bDelta^p$ admits the following sample complexity: for any $\mu', \nu' \in \calP(\rset)$ with respective empirical measures $\hmu'_n, \hnu'_n$, $\E \abs{\bDelta^p(\mu', \nu') - \bDelta^p( \hmu_n',  \hnu_n')} \leq \beta(p, n)$.
  Then, for any $\mu, \nu \in \calP(\Rd)$ with respective empirical measures $\hmu_n, \hnu_n$, the sample complexity of $\sdelta[p]$ is given by \ $\E \abs{\sdelta[p]^p(\mu, \nu) - \sdelta[p]^p(\hmu_n, \hnu_n)} \leq \beta(p, n)$.
\end{theorem} 
If $\bDelta$ is a bounded pseudo-metric and we have a direct control over the convergence rate of empirical measures in $\bDelta$, we can further derive the following result. 
\begin{theorem}
\label{thm:sliced_rateofconv}
  Let $p \in [1, \infty)$. Assume that for any $\mu' \in \calP(\rset)$ with empirical measure $\hmu'_n$, $\E | \bDelta^p( \hmu'_n,  \mu') | \leq \alpha(p, n)$. Then, for any $\mu \in \calP(\Rd)$ with empirical measure $\hmu_n$, we have $\E \big| \sdelta[p]^p(\hmu_n, \mu) \big| \leq \alpha(p, n)$.
  Besides, if $\bDelta$ is non-negative, symmetric, and satisfies the triangle inequality, then $\E \abs{\sdelta[p](\mu, \nu) - \sdelta[p](\hmu_n, \hnu_n)} \nonumber \leq 2~\alpha(p,n)^{1/p}$.
\end{theorem}

Our results show that slicing leads to a dimension-free convergence rate while carrying out useful topological properties of the base divergence (e.g., metric axioms, weak convergence). If the focus is on sustaining such topological properties, then the improvement in the convergence rate is meaningful. On the other hand, slicing also results in less discriminant divergences, as we mentioned for IPMs (\Cref{thm:weak_topo}), and in such a case, the improvement in the rate might be less significant. More analysis is required to understand the potential reduction in the discriminative power, and we leave it out of scope of this study.

In practice, SPDs also induce an approximation error due to the Monte Carlo estimate \eqref{eqref:mc_estimate}. We use the term \emph{projection complexity} to refer to the convergence rate of $\hatsdelta[p,L]$ to $\sdelta[p]$ as a function of the number of projections $L$. Hence, the \emph{overall complexity} $\big|\hatsdelta[p,L](\hmu_n, \hnu_n)-\sdelta[p](\mu, \nu)\big|$ is bounded by the sum of the sample and the projection complexities.

\begin{theorem}
\label{thm:projection_complexity}
  Let $p \in [1, \infty)$ and $\mu, \nu \in \calP(\Rd)$. Then, the error made with the Monte Carlo estimation of $\sdelta[p]$ can be bounded as follows
  \begin{equation}
    \defEns{\E \big| \hatsdelta[p, L]^p(\mu, \nu) - \sdelta[p]^p(\mu, \nu) \big|}^2 \leq L^{-1} \int_{\sphere^{d-1}} \defEns{\bDelta^p(\thsss \mu, \thsss \nu) - \sdelta[p]^p(\mu, \nu)}^2 \rmd \unifS(\theta)  \eqsp .
  \end{equation}
\end{theorem}
By definition of $\sdelta[p]^p(\mu, \nu)$, \Cref{thm:projection_complexity} illustrates that the quality of the Monte Carlo estimates is impacted by the number of projections as well as the variance of the evaluations of the base divergence. This behavior has previously been empirically observed in different scenarios \cite{deshpande2019max,Kolouri2019,paty2019subspace}, and paved the way for the `max-sliced' distances. We additionally provide in the supp. document, finite-sample guarantees on the quality of the Monte Carlo estimates, using \Cref{thm:sliced_sample_complexity,thm:projection_complexity}.

\vspace{-5pt}
\section{Applications} \label{sec:applications}
\vspace{-5pt}

In this section, to further illustrate the significance of our general topological and statistical results, we apply these to specific sliced divergences and present the interesting properties that we obtained. In particular, we will introduce a novel divergence based on Sinkhorn divergences, and provide theoretical results that emphasize its statistical and computational advantages.

First, \Cref{thm:weak_conv} can be applied to various base divergences (e.g., see those listed in \cite[Theorem 6]{Gibbs2002}) and foster interesting applications. In particular, we focus on the Sliced-Cramér distance (SC, \cite{Tabor2018,kolouri2020sliced}) and establish theoretical guarantees which, to the best of our knowledge, have not been proved before: we show that convergence under SC implies weak convergence in $\calP(\bbR^d)$, and the converse is true for measures supported on a compact space. Our general result also applies to the broader class of Sliced-IPMs, assuming a density property for the space of functions associated with the base IPM. We provide the formal statements and proofs of these results in the supplementary document.

Then, we derive the sample complexity of $\swassersteinD[p]$ under different moment conditions. While previous works have illustrated the statistical benefits of SW, our next corollary establishes a novel result: \cite{deshpande2019max} derived the sample complexity for Gaussian distributions only, \cite{2019arXiv190604516N} studied the estimators obtained by minimizing SW, and \cite{manole2019minimax} provided confidence intervals which partially cover our result. 

\begin{corollary}
  \label{thm:smpl_cpx_sw}
  Let $p \in [1, \infty)$, $q > p$, and $\mu, \nu \in \calP_q(\Rd)$ with corresponding empirical measures $\hmu_n, \hnu_n$. We use the notation $M_q^{1/q}(\mu,\nu) = M_q^{1/q}(\mu) + M_q^{1/q}(\nu)$, where $M_q(\zeta)$ refers to the moment of order $q$ of $\zeta \in \calP_q(\Rd)$. Then, there exists a constant $C_{p,q}$ depending on $p, q$ such that
  \begin{align}
    \E \big| \swassersteinD[p](\hmu_n, \hnu_n) - \swassersteinD[p](\mu, \nu) \big| \leq C_{p,q}^{1/p} M_q^{1/q}(\mu,\nu) \left\{ 
    \begin{array}{ll}
      n^{-1/(2p)} & \mbox{ if } q > 2p, \\
      n^{-1/(2p)} \log(n)^{1/p} & \mbox{ if } q = 2p, \\
      n^{-(q-p)/(pq)} & \mbox{ if } q \in (p,2p),
    \end{array}
    \right. \label{eq:correct_smpl_cpx_sw}
  \end{align}
\end{corollary}

We now introduce a new family of probability divergences obtained by slicing the regularized OT cost and Sinkhorn divergences, and called Sliced-Sinkhorn divergences (SSD): for $p \in [1, \infty)$, $\veps \geq 0$ and $\mu, \nu \in \calP_p(\Rd)$,
\begin{equation}
   \swassersteinD[p, \veps](\mu, \nu) = \int_{\sphere^{d-1}} \wassersteinD[p, \veps](\thsss \mu, \thsss \nu)\ \rmd\unifS(\ths),\ \ssinkD[p, \veps](\mu, \nu) = \int_{\sphere^{d-1}} \sinkD[p, \veps](\thsss \mu, \thsss \nu)\ \rmd\unifS(\ths) \label{eq:normalized_ss}
\end{equation}
We show that these divergences enjoy interesting statistical and computational properties. %
For clarity purposes, our results are only presented for $\swassersteinD[p, \veps]$, but also apply for $\ssinkD[p, \veps]$. Since $\wassersteinD[p, \veps]$ is not an IPM, we first derive a topological property analogous to \Cref{thm:weak_topo}. %

\begin{theorem} 
\label{thm:weak_topo_sink}
  Let $p \in [1, \infty)$ and $\veps \geq 0$. For any $\mu, \nu \in \calP_p(\Rd)$,\ \ $\swassersteinD[p, \veps](\mu, \nu) \leq \wassersteinD[p, \veps](\mu, \nu)$.
\end{theorem}
Next, we show that on compact domains, while the sample complexity of regularized OT exponentially worsens as $\veps$ decreases \cite[Theorem 3]{Genevay19}, the sample complexity of SSD does not depend on $\veps$.
\begin{theorem} \label{thm:smpl_cpx_ssink}
  Let $\msx$ be a compact subset of $\Rd$, $p \in [1, \infty)$ and $\mu, \nu \in \calP(\msx)$, with respective empirical instanciations $\hmu_n, \hnu_n$. Then, there exists a constant $C(\mu, \nu)$ that depends on the moments of $\mu$ and $\nu$, such that $\E \big| \swassersteinD[p, \veps](\hmu_n, \hnu_n) - \swassersteinD[p, \veps](\mu, \nu) \big| \leq \diam(\msx)C(\mu, \nu) n^{-1/2}$.

\end{theorem}

In practice, we approximate SSD by using
\eqref{eqref:mc_estimate}. The estimator corresponds to
randomly picking a finite set of directions and solving, for each direction, a
regularized OT problem in $\bbR$. To obtain solutions associated to
the regularized Wasserstein cost \eqref{eq:def_reg_ot}, a method which is now standard is
the Sinkhorn's algorithm
(\cite{franklin1989scaling}; more details in the supp. document,
\cite[Section 4.2]{peyre2019computational}). In particular, if we use the squared
Euclidean ground cost and consider the empirical measures
$\hmu_n, \hnu_n$ on $\bbR^d$ associated to the observations
$(x_i)_{i=1}^n, (y_j)_{j=1}^n$ respectively, computing
$\wassersteinD[p, \veps](\hmu_n, \hnu_n)$ has a worst-case convergence
rate that depends on
$C(\hmu_n, \hnu_n) = \max_{i,j \in \{1,\dots,n\}} \| x_i - y_j \|^2 /
\veps$ (see also~\cite{altschuler2017near} for a sublinear rate with a
better constant, still depending on this quantity). The rate for
$\wassersteinD[p, \veps](\thsss \hmu_n, \thsss \hnu_n)$, with
$\theta \in \sphereD$, then depends on
$C(\thsss \hmu_n, \thsss \hnu_n) = \max_{i,j \in \{1,\dots,n\}} \|
\ps{\theta}{x_i - y_j} \|^2 / \veps$. We show that with high
probability, $C(\thsss \hmu_n, \thsss \hnu_n)$ is smaller than
$C(\hmu_n, \hnu_n)$ by a factor of $d$ at least, unless $n$ grows
super-polynomially with $d$. Our result,
combined with the parallel computation of \eqref{eqref:mc_estimate},
implies that slicing the regularized OT may lead to significant
computational benefits.

\begin{proposition} 
\label{prop:contract}
Let $(x_i)_{i=1}^n$ %
be a set of vectors in $\Rd$ such that $\max_{i,j} \Vert x_i - x_j\Vert_2^2 \leq R^2$, and $\theta$ chosen uniformly at random on $\mathbb{S}^{d-1}$. Then for $\delta \in (0,1]$, it holds with probability $1-\delta$, $\max_{i,j} \vert \langle \theta, x_i-x_j\rangle\vert^2 \leq \frac{2R^2}{d}\log(\sqrt{2\pi}n^2/\delta)$.%
\end{proposition}

Finally, we note that an advantage of the Sinkhorn divergence over the Wasserstein distance is that the former is always differentiable \cite[Proposition 2]{feydy19} while the latter is not. This property, which is crucial in differential programming pipelines, suggests that SSD is potentially better-behaved than SW in tasks such as generative modeling. We leave its analysis to future work.

\vspace{-5pt}
\section{Experiments}
\vspace{-5pt}

We present the numerical experiments that we conducted to illustrate our theoretical findings, and we provide the code to reproduce them\footnote{See \url{https://github.com/kimiandj/sliced_div}}.

\begin{wrapfigure}{r}{0.43\textwidth}
  \vspace{-2pt} 
  \centering
    \includegraphics[width=\linewidth]{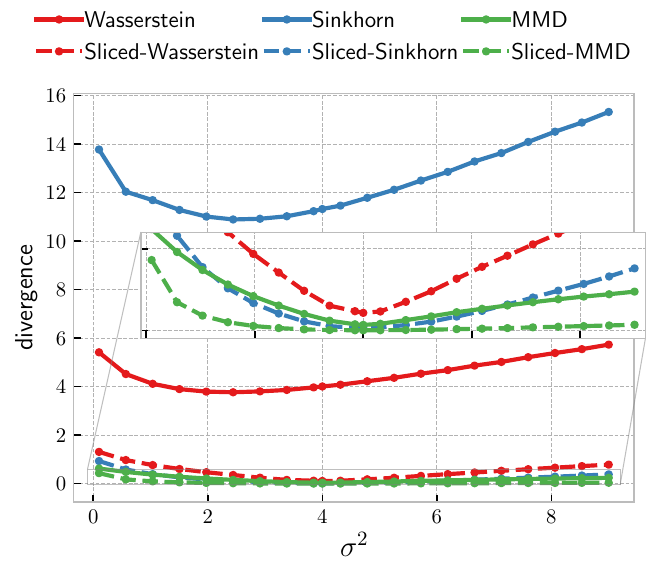}
  \vspace{-16pt}
  \caption{%
  (Sliced-)Divergences between two sets of 1000 samples in $\rset^{10}$ \iid~from $\calN({\bf0}, 4\bfI)$ and $\calN({\bf0}, \sigma^2\bfI)$, for varying $\sigma^2$. %
  }
  \label{fig:lowerbound}
  \vspace{-6pt}
\end{wrapfigure}

We first verify that IPMs and Sinkhorn divergences are bounded below by their sliced versions, as demonstrated in \Cref{thm:weak_topo,thm:weak_topo_sink} respectively. Consider $n = 1000$ observations \iid~from $\calN({\bf0}, \sigma_\star^2 \bfI_d)$, with $\sigma_\star^2 = 4$. We generate $n$ \iid~samples from $\calN({\bf0}, \sigma^2 \bfI_d)$ for $\sigma^2$ varying between 0.1 and 9. We compute MMD between the empirical distributions of the observations and the generated datasets, as well as the Wasserstein distance of order 1 and normalized Sinkhorn divergence \eqref{eq:normalized_ss} with order 1 and $\veps=1$. We used a Gaussian kernel for MMD combined with the heuristic proposed in \cite{Gretton2012}, which sets the kernel width to be the median distance over the aggregated data, and we approximated this discrepancy with the biased estimator in \cite[Equation 5]{Gretton2012}. Then, we compute Sliced-Wasserstein, Sliced-Sinkhorn and Sliced-MMD. Each of these sliced divergences was approximated with a Monte Carlo estimate based on 50 randomly picked projections. \Cref{fig:lowerbound} reports the divergences against $\sigma^2$ for $d = 10$. Results are averaged over 10 runs, and for clarity reasons, we do not plot the error bands (based on the 10th-90th percentiles) as these were very tight. The curves for Wasserstein, Sinkhorn and MMD are above their respective sliced version's ones, as predicted by our theoretical bounds. This figure also illustrates the statistical benefits induced by slicing: all sliced divergences attain their minimum at $\sigma_\star$, while Wasserstein and Sinkhorn fail at this. This observation is in line with \cite{bellemare2017cramer}, where the authors showed that both the minimum point and gradients of the Wasserstein distance have a bias, which can be prominent unless $n$ is large enough. MMD performs well in this task, and this can be explained by its dimension-free sample complexity. In that sense, Sliced-MMD acts more as a sanity-check of our theory, rather than a practical proposal.

\begin{figure}[t!]
\centering
  \subfigure[Projection complexity]{
    \includegraphics[width=.43\linewidth]{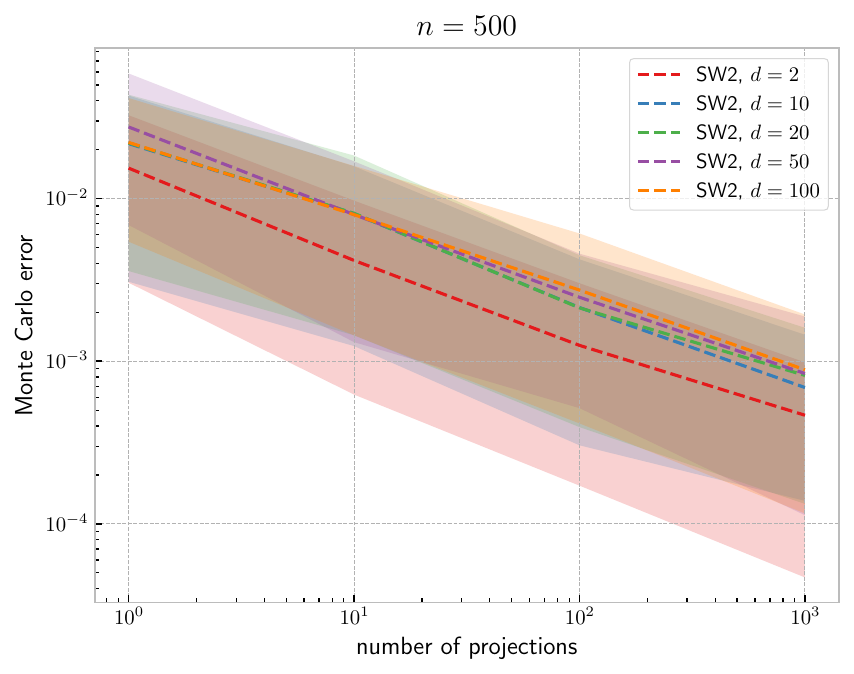}
    \label{fig:proj_cpx_sw}
  }
  \subfigure[Sample complexity]{
    \includegraphics[width=.43\linewidth]{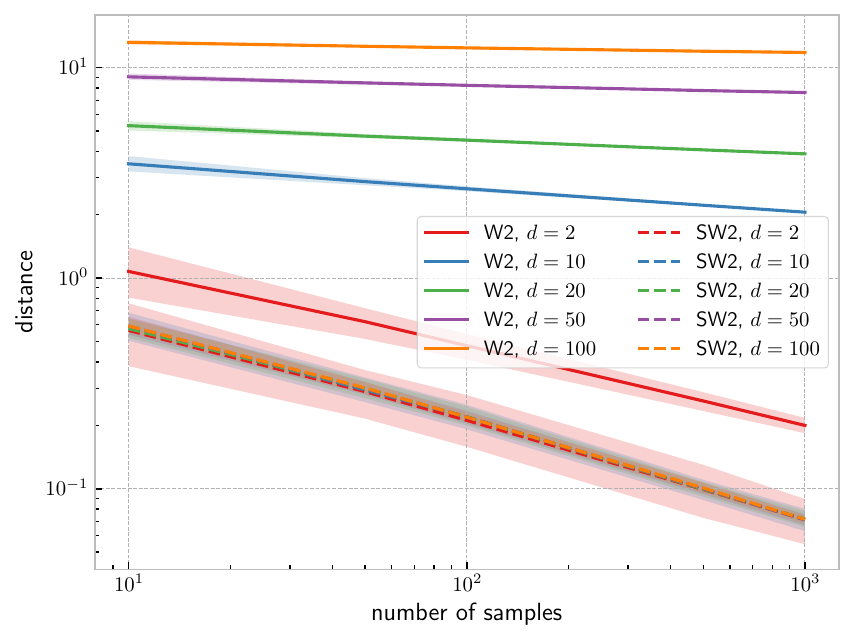}
    \label{fig:sample_cpx_sw}
  }
  \caption{(Sliced-)Wasserstein distances of order 2 between two sets of $n$ samples generated from $\calN({\bf0}, \bfI_d)$ for different $d$, on log-log scale. Results are averaged over 100 runs, and the shaded areas correspond to the 10th-90th percentiles. %
  }
\end{figure}

The next experiments aim at illustrating our statistical properties %
We first analyze the convergence rate of the Monte Carlo estimates (\Cref{thm:projection_complexity}) in a synthetical setting. We consider two sets of $500$ samples \iid~from the $d$-dimensional Gaussian distribution $\calN({\bf0}, \bfI_d)$, and we approximate $\swassersteinD[2]$ between the empirical distributions with a Monte Carlo scheme that uses a high number of projections $L_\star = 10~000$. Then, we compute the Monte Carlo estimate $\hatswassersteinD[2,L]$ obtained with $L < L_\star$ random projections. \Cref{fig:proj_cpx_sw} shows the absolute difference of $\hatswassersteinD[2, L]$ and $\hatswassersteinD[2, L_\star]$ against $L$, for different values of dimension $d$. We observe that the Monte Carlo error indeed shrinks to zero when we increase the number of projections, with a convergence rate of order $L^{-1/2}$. 

Then, we illustrate the sample complexity of Sliced-Wasserstein and Sliced-Sinkhorn (\Cref{thm:smpl_cpx_sw} and \Cref{thm:smpl_cpx_ssink}, respectively). We consider two sets of $n$ samples \iid~from $\calN({\bf0}, \bfI_d)$, and we compute $\wassersteinD[2]$ and $\sinkD[2,\veps]$ and their sliced versions approximated with 100 random projections. %
We analyze the convergence rate for different $n$ and dimensions $d$. We also study the influence of the regularization parameter $\veps$ for Sinkhorn divergences. %
\Cref{fig:sample_cpx_sw} reports the Wasserstein and Sliced-Wasserstein distances vs. $n$, for $d$ between 2 and 100. We observe that, as opposed to $\wassersteinD[2]$, the convergence rate of $\swassersteinD[2]$ does not depend on the dimension, therefore $\swassersteinD[2]$ converges faster than $\wassersteinD[2]$ when the dimension increases. \Cref{subfig:sample_cpx_ssink_a,subfig:sample_cpx_ssink_b} show Sinkhorn and Sliced-Sinkhorn divergences vs. $n$, and respectively study the influence of $d$ and $\veps$ on the convergence rate. As predicted by the theory, Sliced-Sinkhorn offers more `robustness' than Sinkhorn: its convergence rate does not depend on the dimension nor on the regularization coefficient. To illustrate \Cref{prop:contract}, we plot on \Cref{subfig:sample_cpx_ssink_c} the number of iterations when the convergence of Sinkhorn's algorithm is reached, as a function of $d$. For Sliced-Sinkhorn, this number is an average over the number of projections used in the approximation. Our experiment emphasizes the computational advantages of Sliced-Sinkhorn, since its number of iterations remains the same with the increasing dimension, while it grows exponentially for Sinkhorn.

\begin{figure}[t!]
\centering
  \subfigure[Influence of dimension]{
    \includegraphics[width=.3\textwidth]{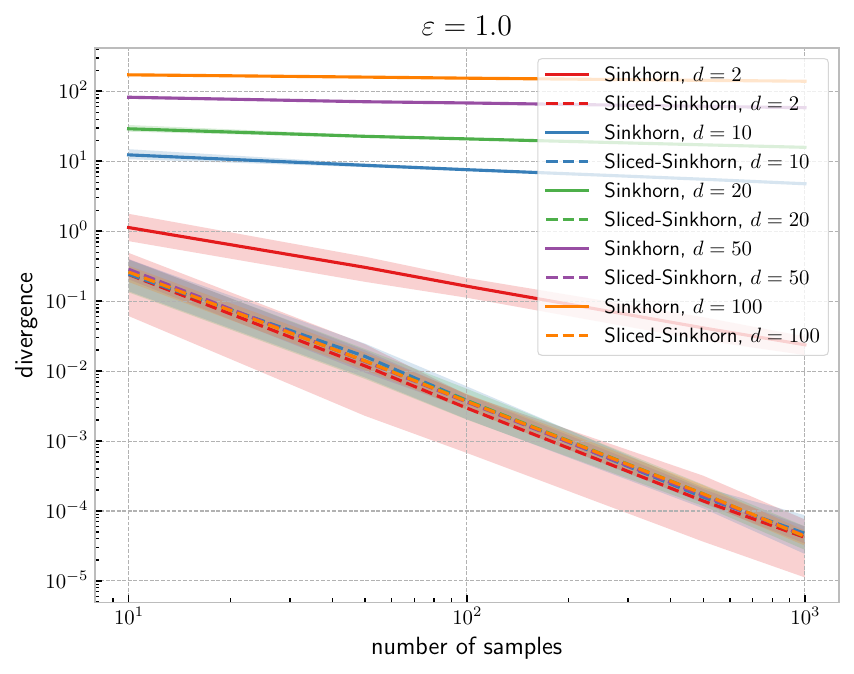}
    \label{subfig:sample_cpx_ssink_a}
  } 
  \subfigure[Influence of regularization]{
    \includegraphics[width=.3\textwidth]{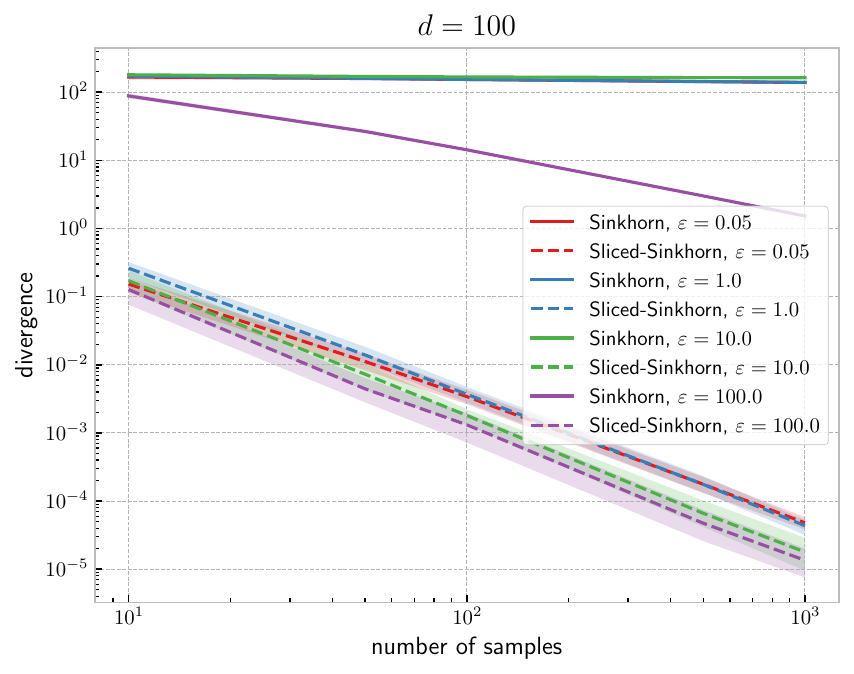}
    \label{subfig:sample_cpx_ssink_b}
  }
  \subfigure[Study of number of iterations]{
    \includegraphics[width=.3\textwidth]{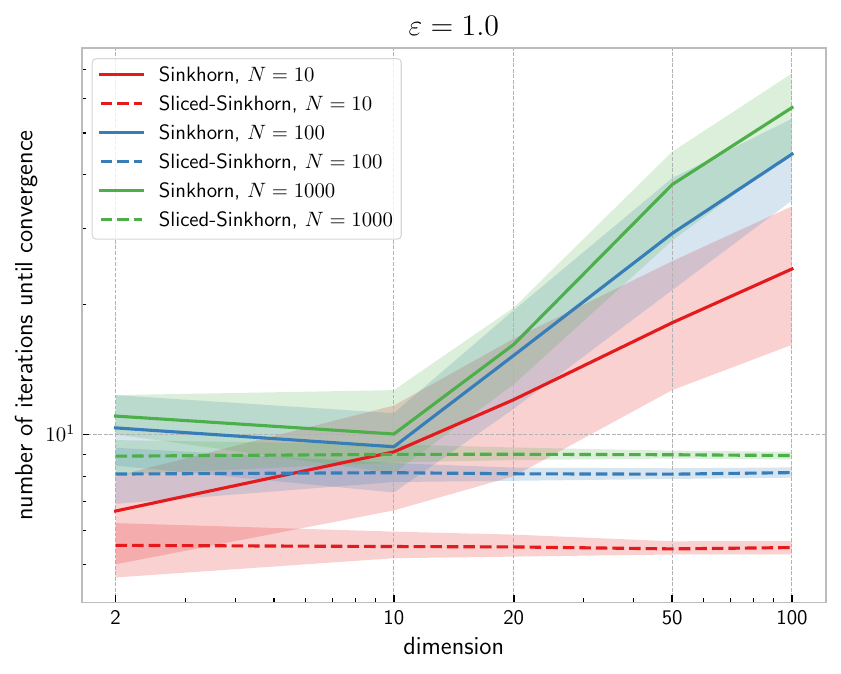}
    \label{subfig:sample_cpx_ssink_c}
  }
  \vspace{-8pt}
  \caption{(Sliced-)Sinkhorn divergences between two sets of $n$ samples generated from $\calN({\bf0}, \bfI_d)$ for different values of $n$, dimension $d$, and regularization coefficient $\veps$. Results are averaged over 100 runs, and the shaded areas correspond to the 10th-90th percentiles. All plots have a log-log scale.}
  \label{fig:sample_cpx_ssink}
  \vspace{-15pt}
\end{figure}

Our last experiment operates on real data and is motivated by the two-sample testing problem \cite{Gretton2012}, whose goal is to determine whether two sets of samples were generated from the same distribution or not. This is useful for various applications, including data integration, where we wish to understand that two datasets were drawn from the same distribution in order to merge them. 
In this context, we run the following experiment: for different values of $n$, we randomly select two subsets of $n$ samples from the same dataset, and we compute the Wasserstein and Sliced-Wasserstein distances (of order 2) between the empirical distributions, as well as the Sinkhorn and Sliced-Sinkhorn divergences ($\veps=1$). The sliced divergences are approximated with 10 random projections. We use the MNIST \cite{lecun-mnisthandwrittendigit-2010} and CIFAR-10 \cite[Chapter 3]{Krizhevsky09learningmultiple} datasets, %
and we report the divergences %
against $n$, and the mean execution time for the computation of Sinkhorn and Sliced-Sinkhorn, on \Cref{fig:realexp}. The sliced divergences perform the best, in the sense that they need less samples to converge to zero. Besides, Sliced-Sinkhorn is faster than Sinkhorn in terms of execution time (which was expected, given our discussion above \Cref{prop:contract}), and the difference is even more visible for a high number of samples. For example, for $n = 2500$ on MNIST or $n = 1000$ on CIFAR-10, Sliced-Sinkhorn is almost 130 times faster than for Sinkhorn on average. 

\begin{figure*}[t!]
\centering
  \subfigure[Results on MNIST]
  {
    \includegraphics[width=.23\textwidth]{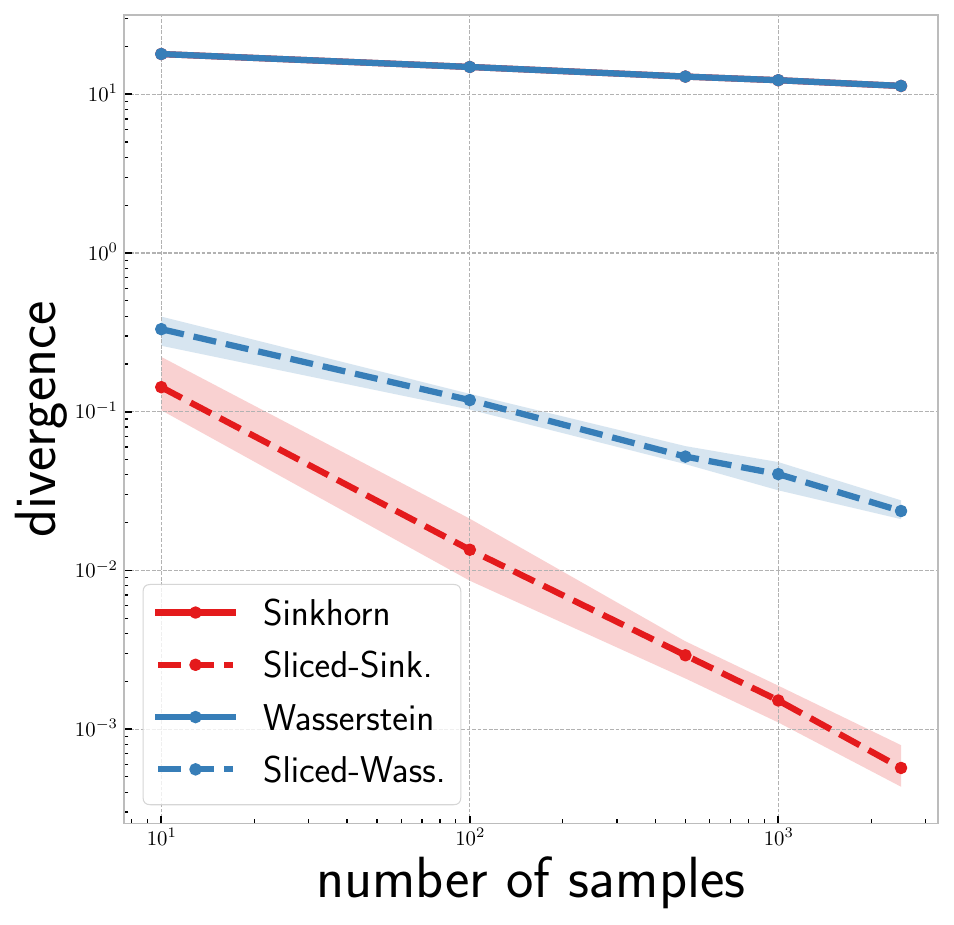}
    \includegraphics[width=.23\textwidth]{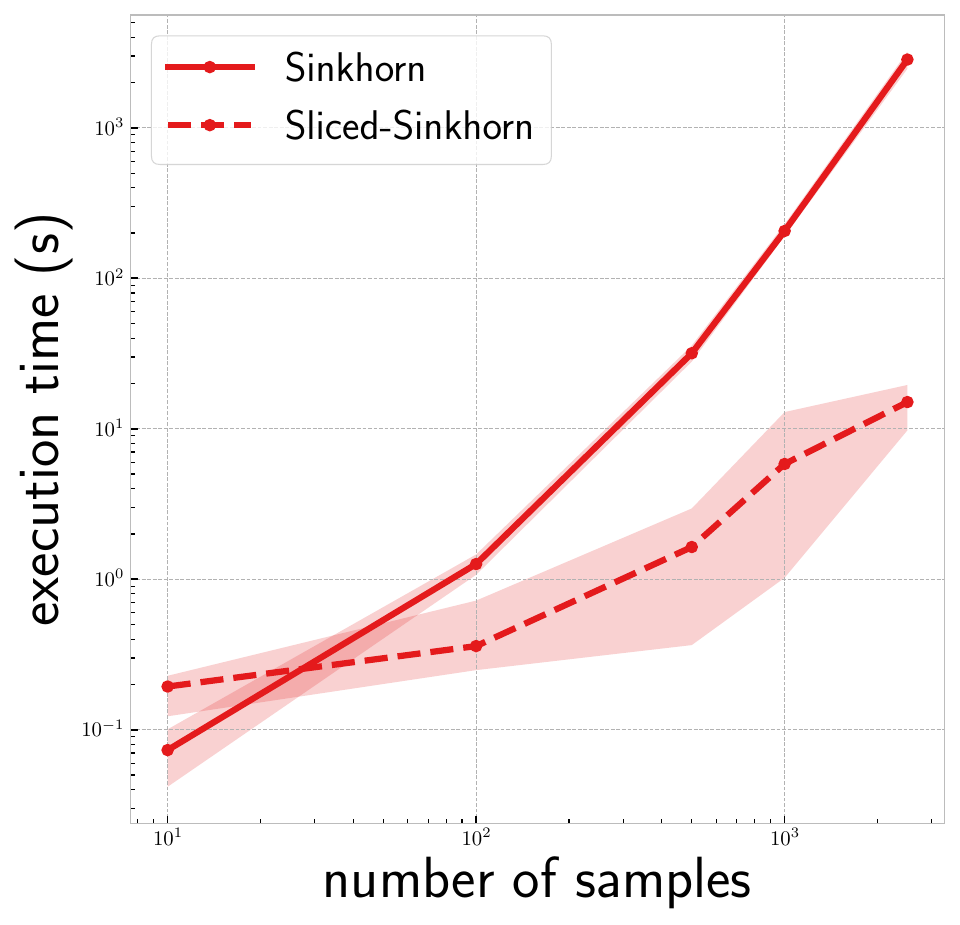}
    \label{subfig:mnist}
  }
  \subfigure[Results on CIFAR-10]
  {
    \includegraphics[width=.23\textwidth]{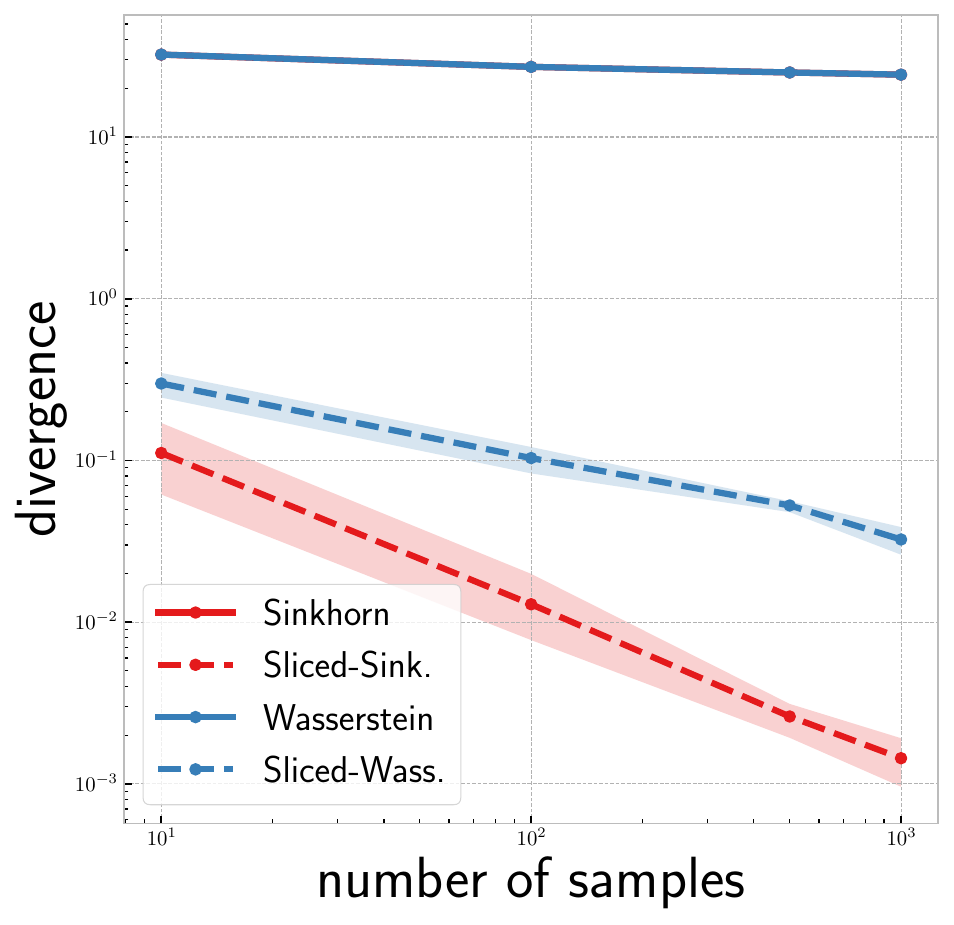}
    \includegraphics[width=.23\textwidth]{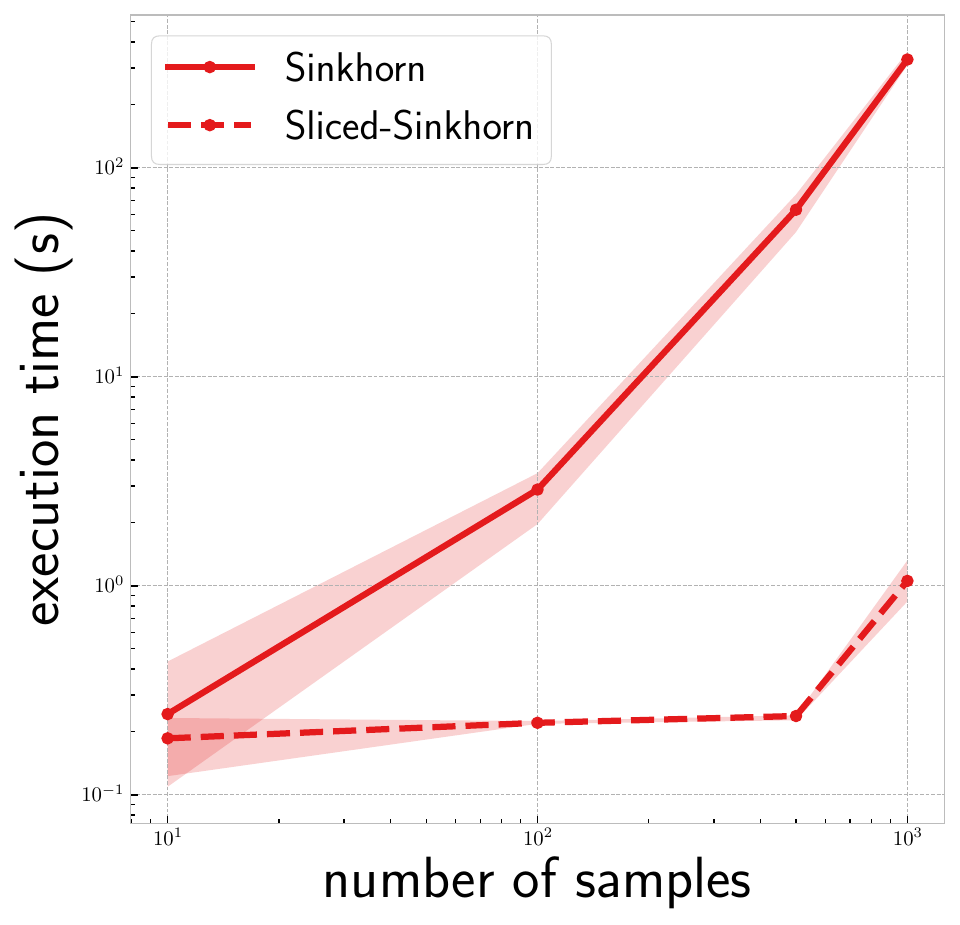}
    \label{subfig:cifar}
  }
  \vspace{-8pt}
  \caption{(Sliced-)Wasserstein and (Sliced-)Sinkhorn ($\veps=1$) between two random subsets of $n$ samples of real datasets, for different values of $n$. Results are averaged over 10 runs, and the shaded areas correspond to the 10th-90th percentiles. All plots have a log-log scale.}
  \vspace{-5pt}
  \label{fig:realexp}
\end{figure*}

\vspace{-5pt}
\section{Conclusion}
\vspace{-5pt}

In this study, we considered sliced probability divergences, which have been increasingly popular in machine learning applications. %
We derived theoretical results about their induced topology as well as their statistical efficiency in terms of number of samples and projections, and we empirically illustrated our findings on different setups. Specifically, we proved that the preserved topology and dimension-free sample complexity are intrinsic to slicing. Since this was unclear in the previous literature, which combined slicing with a specific distance, our unified treatment of these results brings insight to the properties of particular instances used in practice.
The gains in statistical efficiency could be explained by an ability of slicing to overlook irrelevant characteristics of the distributions. An important question for future work is then to understand precisely what geometrical features are well preserved by the slicing operation. %
Another interesting future direction is to extend our analysis to the recently proposed `max-sliced' \cite{deshpande2019max} and `generalized' sliced divergences \cite{Kolouri2019}.

\section*{Broader Impact}

This paper is focused on the theoretical properties of sliced probability divergences, which have become increasingly popular in recent years due to their applications on implicit generative modeling. Our analysis uncovers the topological and statistical consequences of the slicing operation, and aims at providing answers to the question \emph{``When and why do sliced divergences perform well in practice?''}. We believe that our theory would provide useful guidelines for practitioners working in this field, in terms of designing new sliced divergences as well as obtaining a better understanding on the existing sliced divergences. %
Our contributions are mainly theoretical, and we believe these will not pose any negative or positive ethical or societal consequence in the broad sense. 

\begin{ack}
This work is partly supported by the French National Research Agency (ANR) as a part of the FBIMATRIX project (ANR-16-CE23-0014) and by the industrial chair Machine Learning for Big Data from T\'el\'ecom Paris. Alain Durmus acknowledges support from the Polish National Science Center grant (NCN UMO-2018/31/B/ST1/00253). The authors are grateful to Christos Tsirigotis for the fruitful discussion which motivated this work, and to an anonymous reviewer who gave an argument to improve the initial bound in \Cref{prop:contract}. 
\end{ack}

\small
\bibliography{main}

\begin{thebibliography}{10}

\bibitem{goodfellow2014generative}
Ian Goodfellow, Jean Pouget-Abadie, Mehdi Mirza, Bing Xu, David Warde-Farley,
  Sherjil Ozair, Aaron Courville, and Yoshua Bengio.
\newblock Generative adversarial nets.
\newblock In {\em Advances in neural information processing systems}, pages
  2672--2680, 2014.

\bibitem{kingma2013auto}
Diederik~P Kingma and Max Welling.
\newblock Auto-encoding variational bayes.
\newblock In {\em International Conference on Learning Representations (ICLR)},
  2014.

\bibitem{Arjovsky2017}
Martin Arjovsky, Soumith Chintala, and L{\'e}on Bottou.
\newblock {W}asserstein generative adversarial networks.
\newblock In Doina Precup and Yee~Whye Teh, editors, {\em Proceedings of the
  34th International Conference on Machine Learning}, volume~70 of {\em
  Proceedings of Machine Learning Research}, pages 214--223, International
  Convention Centre, Sydney, Australia, 06--11 Aug 2017. PMLR.

\bibitem{bousquet2017optimal}
Olivier Bousquet, Sylvain Gelly, Ilya Tolstikhin, Carl-Johann Simon-Gabriel,
  and Bernhard Schoelkopf.
\newblock {From optimal transport to generative modeling: the VEGAN cookbook}.
\newblock {\em arXiv preprint arXiv:1705.07642}, 2017.

\bibitem{gulrajani2017improved}
Ishaan Gulrajani, Faruk Ahmed, Martin Arjovsky, Vincent Dumoulin, and Aaron~C
  Courville.
\newblock {Improved training of Wasserstein GANs}.
\newblock In {\em Advances in Neural Information Processing Systems}, pages
  5767--5777, 2017.

\bibitem{tolstikhin2018wasserstein}
Ilya Tolstikhin, Olivier Bousquet, Sylvain Gelly, and Bernhard Schoelkopf.
\newblock Wasserstein auto-encoders.
\newblock In {\em International Conference on Learning Representations}, 2018.

\bibitem{arbel2019maximum}
Michael Arbel, Anna Korba, Adil Salim, and Arthur Gretton.
\newblock Maximum mean discrepancy gradient flow.
\newblock In {\em Advances in Neural Information Processing Systems}, 2019.

\bibitem{bonneel2015sliced}
Nicolas Bonneel, Julien Rabin, Gabriel Peyr{\'e}, and Hanspeter Pfister.
\newblock Sliced and {R}adon {W}asserstein barycenters of measures.
\newblock {\em Journal of Mathematical Imaging and Vision}, 51(1):22--45, 2015.

\bibitem{kolouri2020sliced}
Soheil Kolouri, Nicholas~A. Ketz, Andrea Soltoggio, and Praveen~K. Pilly.
\newblock Sliced {C}ramer synaptic consolidation for preserving deeply learned
  representations.
\newblock In {\em International Conference on Learning Representations}, 2020.

\bibitem{Cramer1928}
Harald Cramér.
\newblock On the composition of elementary errors.
\newblock {\em Scandinavian Actuarial Journal}, 1928(1):141--180, 1928.

\bibitem{Tabor2018}
Jacek Tabor, Szymon Knop, Przemyslaw Spurek, Igor~T. Podolak, Marcin Mazur, and
  Stanis{l}aw Jastrz{k{e}}bski.
\newblock Cramer-wold autoencoder.
\newblock {\em CoRR}, abs/1805.09235, 2018.

\bibitem{deshpande2019max}
Ishan Deshpande, Yuan-Ting Hu, Ruoyu Sun, Ayis Pyrros, Nasir Siddiqui, Sanmi
  Koyejo, Zhizhen Zhao, David Forsyth, and Alexander Schwing.
\newblock Max-sliced {W}asserstein distance and its use for gans.
\newblock In {\em IEEE Conference on Computer Vision and Pattern Recognition},
  2019.

\bibitem{paty2019subspace}
Fran{\c{c}}ois-Pierre Paty and Marco Cuturi.
\newblock Subspace robust {W}asserstein distances.
\newblock In {\em International Conference on Machine Learning}, 2019.

\bibitem{kolouri2018sliced}
Soheil Kolouri, Phillip~E. Pope, Charles~E. Martin, and Gustavo~K. Rohde.
\newblock Sliced {W}asserstein auto-encoders.
\newblock In {\em International Conference on Learning Representations}, 2019.

\bibitem{Kolouri2019}
Soheil Kolouri, Kimia Nadjahi, Umut Simsekli, Roland Badeau, and Gustavo Rohde.
\newblock Generalized sliced {W}asserstein distances.
\newblock In H.~Wallach, H.~Larochelle, A.~Beygelzimer, F.~d'Alch\'{e} Buc,
  E.~Fox, and R.~Garnett, editors, {\em Advances in Neural Information
  Processing Systems 32}, pages 261--272. Curran Associates, Inc., 2019.

\bibitem{vayer2019}
Titouan Vayer, R{\'e}mi Flamary, Romain Tavenard, Laetitia Chapel, and Nicolas
  Courty.
\newblock {Sliced Gromov-Wasserstein}.
\newblock In {\em {NeurIPS 2019 - Thirty-third Conference on Neural Information
  Processing Systems}}, volume~32, Vancouver, Canada, December 2019.

\bibitem{Bonnotte2013}
Nicolas Bonnotte.
\newblock {\em Unidimensional and Evolution Methods for Optimal
  Transportation}.
\newblock PhD thesis, Paris 11, 2013.

\bibitem{2019arXiv190604516N}
Kimia Nadjahi, Alain Durmus, Umut \c{S}im{\textcommabelow s}ekli, and Roland
  Badeau.
\newblock {Asymptotic Guarantees for Learning Generative Models with the
  Sliced-Wasserstein Distance}.
\newblock In {\em Advances in Neural Information Processing Systems}, 2019.

\bibitem{bayraktar2019strong}
Erhan Bayraktar and Gaoyue Guo.
\newblock Strong equivalence between metrics of {W}asserstein type, 2019.

\bibitem{deshpande2018generative}
Ishan Deshpande, Ziyu Zhang, and Alexander Schwing.
\newblock Generative modeling using the sliced {W}asserstein distance.
\newblock In {\em Proceedings of the IEEE Conference on Computer Vision and
  Pattern Recognition}, pages 3483--3491, 2018.

\bibitem{Muller1997}
Alfred Müller.
\newblock Integral probability metrics and their generating classes of
  functions.
\newblock {\em Advances in Applied Probability}, 29(2):429–443, 1997.

\bibitem{feydy19}
Jean Feydy, Thibault S\'{e}journ\'{e}, Fran\c{c}ois-Xavier Vialard, Shun-ichi
  Amari, Alain Trouve, and Gabriel Peyr\'{e}.
\newblock Interpolating between optimal transport and {MMD} using {S}inkhorn
  divergences.
\newblock In Kamalika Chaudhuri and Masashi Sugiyama, editors, {\em Proceedings
  of Machine Learning Research}, volume~89 of {\em Proceedings of Machine
  Learning Research}, pages 2681--2690. PMLR, 16--18 Apr 2019.

\bibitem{Genevay19}
Aude Genevay, L\'{e}na\"{i}c Chizat, Francis Bach, Marco Cuturi, and Gabriel
  Peyr\'{e}.
\newblock Sample complexity of {S}inkhorn divergences.
\newblock In {\em Proceedings of Machine Learning Research}, pages 1574--1583,
  2019.

\bibitem{Mena2019}
Gonzalo Mena and Jonathan Niles-Weed.
\newblock Statistical bounds for entropic optimal transport: sample complexity
  and the central limit theorem.
\newblock In H.~Wallach, H.~Larochelle, A.~Beygelzimer, F.~d'Alch\'{e} Buc,
  E.~Fox, and R.~Garnett, editors, {\em Advances in Neural Information
  Processing Systems 32}, pages 4543--4553. Curran Associates, Inc., 2019.

\bibitem{Sriperumbudur09onintegral}
Bharath~K. Sriperumbudur, Kenji Fukumizu, Arthur Gretton, Bernhard Schölkopf,
  and Gert R.~G. Lanckriet.
\newblock On integral probability metrics, $\phi$-divergences and binary
  classification, 2009.

\bibitem{villani2008optimal}
C{\'e}dric Villani.
\newblock {\em Optimal transport: old and new}, volume 338.
\newblock Springer Science \& Business Media, 2008.

\bibitem{Gretton2012}
Arthur Gretton, Karsten~M. Borgwardt, Malte~J. Rasch, Bernhard Sch\"{o}lkopf,
  and Alexander Smola.
\newblock A kernel two-sample test.
\newblock {\em J. Mach. Learn. Res.}, 13(null):723–773, March 2012.

\bibitem{dedecker_merlevede_2007}
Jérôme Dedecker and Florence Merlevède.
\newblock The empirical distribution function for dependent variables:
  asymptotic and nonasymptotic results in ${\mathbb l}^p$.
\newblock {\em ESAIM: Probability and Statistics}, 11:102–114, 2007.

\bibitem{Cuturi2013}
Marco Cuturi.
\newblock Sinkhorn distances: Lightspeed computation of optimal transport.
\newblock In C.~J.~C. Burges, L.~Bottou, M.~Welling, Z.~Ghahramani, and K.~Q.
  Weinberger, editors, {\em Advances in Neural Information Processing Systems
  26}, pages 2292--2300. Curran Associates, Inc., 2013.

\bibitem{Kolouri_2018_CVPR}
Soheil Kolouri, Gustavo~K. Rohde, and Heiko Hoffmann.
\newblock Sliced {W}asserstein distance for learning gaussian mixture models.
\newblock In {\em The IEEE Conference on Computer Vision and Pattern
  Recognition (CVPR)}, June 2018.

\bibitem{csimcsekli2018sliced}
Antoine Liutkus, Umut {\c{S}}im{\c{s}}ekli, Szymon Majewski, Alain Durmus, and
  Fabian-Robert Stoter.
\newblock Sliced-{W}asserstein flows: Nonparametric generative modeling via
  optimal transport and diffusions.
\newblock In {\em International Conference on Machine Learning}, 2019.

\bibitem{wu2017sliced}
Jiqing Wu, Zhiwu Huang, Wen Li, Janine Thoma, and Luc Van~Gool.
\newblock Sliced wasserstein generative models.
\newblock In {\em IEEE Conference on Computer Vision and Pattern Recognition},
  2019.

\bibitem{Gibbs2002}
Alison~L. Gibbs and Francis~Edward Su.
\newblock On choosing and bounding probability metrics.
\newblock {\em International Statistical Review}, 70(3):419--435, 2002.

\bibitem{manole2019minimax}
Tudor Manole, Sivaraman Balakrishnan, and Larry Wasserman.
\newblock Minimax confidence intervals for the sliced {W}asserstein distance,
  2019.

\bibitem{franklin1989scaling}
Joel Franklin and Jens Lorenz.
\newblock On the scaling of multidimensional matrices.
\newblock {\em Linear Algebra and its applications}, 114:717--735, 1989.

\bibitem{peyre2019computational}
Gabriel Peyr{\'e}, Marco Cuturi, et~al.
\newblock Computational optimal transport.
\newblock {\em Foundations and Trends{\textregistered} in Machine Learning},
  11(5-6):355--607, 2019.

\bibitem{altschuler2017near}
Jason Altschuler, Jonathan Niles-Weed, and Philippe Rigollet.
\newblock Near-linear time approximation algorithms for optimal transport via
  sinkhorn iteration.
\newblock In {\em Advances in Neural Information Processing Systems}, pages
  1964--1974, 2017.

\bibitem{bellemare2017cramer}
Marc~G Bellemare, Ivo Danihelka, Will Dabney, Shakir Mohamed, Balaji
  Lakshminarayanan, Stephan Hoyer, and R{\'e}mi Munos.
\newblock The {C}ramer distance as a solution to biased {W}asserstein
  gradients.
\newblock {\em arXiv preprint arXiv:1705.10743}, 2017.

\bibitem{lecun-mnisthandwrittendigit-2010}
Yann LeCun and Corinna Cortes.
\newblock {MNIST} handwritten digit database.
\newblock 2010.

\bibitem{Krizhevsky09learningmultiple}
Alex Krizhevsky.
\newblock Learning multiple layers of features from tiny images.
\newblock Technical report, 2009.

\bibitem{Bogachev2007}
V.I. Bogachev.
\newblock {\em Measure Theory}.
\newblock Number vol.~1 in Measure Theory. Springer Berlin Heidelberg, 2007.

\bibitem{kallenberg:1997}
O.~Kallenberg.
\newblock {\em Foundations of modern probability}.
\newblock Probability and its Applications (New York). Springer-Verlag, New
  York, 1997.

\bibitem{Billingsley1999}
Patrick Billingsley.
\newblock {\em Convergence of probability measures}.
\newblock Wiley Series in Probability and Statistics: Probability and
  Statistics. John Wiley \& Sons Inc., New York, second edition, 1999.
\newblock A Wiley-Interscience Publication.

\bibitem{folland:1999}
G.~B. Folland.
\newblock {\em Real analysis}.
\newblock Pure and Applied Mathematics (New York). John Wiley \& Sons, Inc.,
  New York, second edition, 1999.
\newblock Modern techniques and their applications, A Wiley-Interscience
  Publication.

\bibitem{sra2016directional}
Suvrit Sra.
\newblock Directional statistics in machine learning: a brief review, 2016.

\bibitem{ambrosio2008gradient}
L.~Ambrosio, N.~Gigli, and G.~Savare.
\newblock {\em Gradient Flows: In Metric Spaces and in the Space of Probability
  Measures}.
\newblock Lectures in Mathematics. ETH Z{\"u}rich. Birkh{\"a}user Basel, 2008.

\bibitem{Fournier2015}
Nicolas Fournier and Arnaud Guillin.
\newblock {On the rate of convergence in {W}asserstein distance of the
  empirical measure}.
\newblock {\em {Probability Theory and Related Fields}}, 162(3-4):707, August
  2015.

\bibitem{Genevay2016}
Aude Genevay, Marco Cuturi, Gabriel Peyr\'{e}, and Francis Bach.
\newblock Stochastic optimization for large-scale optimal transport.
\newblock In D.~D. Lee, M.~Sugiyama, U.~V. Luxburg, I.~Guyon, and R.~Garnett,
  editors, {\em Advances in Neural Information Processing Systems 29}, pages
  3440--3448. Curran Associates, Inc., 2016.

\bibitem{wainwright2019high}
Martin~J Wainwright.
\newblock {\em High-dimensional statistics: A non-asymptotic viewpoint},
  volume~48.
\newblock Cambridge University Press, 2019.

\end{thebibliography}
\bibliographystyle{unsrt}

\newpage
\appendix  

\setcounter{equation}{0}
\setcounter{figure}{0}
\setcounter{table}{0}
\setcounter{definition}{0}
\setcounter{theorem}{0}
\setcounter{lemma}{0}
\setcounter{proposition}{0}
\setcounter{corollary}{0}
\renewcommand{\theequation}{S\arabic{equation}}
\renewcommand{\thefigure}{S\arabic{figure}}
\renewcommand{\thetable}{S\arabic{table}}
\renewcommand{\thedefinition}{S\arabic{definition}}
\renewcommand{\thetheorem}{S\arabic{theorem}}
\renewcommand{\thelemma}{S\arabic{lemma}}
\renewcommand{\theproposition}{S\arabic{proposition}}
\renewcommand{\thecorollary}{S\arabic{corollary}}

\section{Postponed proofs for \Cref{sec:slicedmetrics}} \label{supp:sec:proofs_sliceddiv}

\subsection{Proof of \Cref{thm:metric}}
\begin{proof}[Proof of \Cref{thm:metric}]
\ref{thm:metric_i}  The fact that $\sdelta[p]$ is non-negative (or symmetric) if $\bDelta$ is, immediately follows from the definition of $\sdelta[p]$ \eqref{eq:spd}.

\ref{thm:metric_ii} Assume that $\bDelta$ satisfies the identity of indiscernibles, \ie~for $\mu',\nu' \in \calP(\rset)$, $\bDelta(\mu', \nu') = 0$ if and only if $\mu' = \nu'$. For any $\mu \in \calP(\Rd)$ and $\ths \in \sphereD$, $\bDelta(\thsss \mu, \thsss \mu) = 0$, therefore $\sdelta[p](\mu, \mu) = 0$ by its definition \eqref{eq:spd}. Now, consider $\mu$, $\nu \in \calP(\Rd)$ such that $\sdelta[p](\mu,\nu)=0$. Then, by the definition of $\sdelta[p]$ \eqref{eq:spd}, we have $\bDelta ( \thsss \mu, \thsss \nu) = 0$ for $\unifS$-almost every ($\unifS$-a.e.) $\theta \in \sphere^{d-1}$, therefore $\thsss \mu = \thsss \nu$ for $\unifS$-a.e. $\theta \in \sphere^{d-1}$. Next, we use the same technique as in \cite[Proposition 5.1.2]{Bonnotte2013}: for any measure $\xi \in \calP(\rset^s)$ ($s \geq 1$), $\calF[\xi]$ denotes the Fourier transform of $\xi$ and is defined as, for any $w \in \rset^s$, 
\begin{equation}
  \calF[\xi](w) = \int_{\rset^s} e^{-\rmi \ps{w}{x}} \rmd\xi(x) \eqsp. 
\end{equation}

Then, by using \eqref{eq:ft_equal} and the property of pushforward measures, we have for any $t \in \rset$ and $\theta \in \sphereD$, 
\begin{equation}
  \calF[\thsss \mu](t) = \int_\rset e^{-\rmi tu} \rmd\thsss\mu(u) = \int_{\Rd} e^{-\rmi t\ps{\theta}{x}} \rmd\mu(x) = \calF[\mu](t\theta) \eqsp. \label{eq:ft_equal}
\end{equation}

Since for $\unifS$-a.e. $\theta \in \sphereD$, $\thsss \mu = \thsss \nu$ thus $\calF[\thsss \mu] = \calF[\thsss \nu]$, we obtain $\calF[\mu] = \calF[\nu]$. By the injectivity of the Fourier transform, we conclude that $\mu = \nu$.

\ref{thm:metric_iii} Suppose $\bDelta$ is a metric. Based on the previous results,  to show that $\sdelta[p]$ is a metric, all we need to prove here is that it verifies the triangle inequality.
Let  $\mu$, $\nu,\xi \in \calP(\Rd) $. Using that $\bDelta$ satisfies the triangle inequality and the Minkowski inequality in $\rml^p(\sphere^{d-1},\unifS)$, we get
  \begin{align}
  \sdelta[p](\mu, \nu) &= \left \{ \int_{\sphereD} \bDelta^p \big( \thsss \mu, \thsss \nu \big) \rmd \unifS(\theta) \right\}^{1/p} \nonumber \\
  &\leq \Big\{ \int_{\sphereD} \Big[ \bDelta \big( \thsss \mu, \thsss \xi \big) + \bDelta \big( \thsss \xi, \thsss \nu \big) \Big]^p \rmd \unifS(\theta)\Big\}^{1/p} \nonumber \\
  &\leq 
  \left\{ \int_{\sphereD}  \bDelta^p\big( \thsss \mu, \thsss \xi \big) \rmd \unifS(\theta) \right\}^{1/p} + \left\{ \int_{\sphereD} \bDelta^p \big( \thsss \xi, \thsss \nu \big) \rmd \unifS(\theta) \right\}^{1/p} \nonumber \\
  &\leq \sdelta[p](\mu,\xi) + \sdelta[p](\xi,\nu) \eqsp.
  \end{align}
\end{proof}

\subsection{Proof of \Cref{thm:weak_conv}} 

We start by proving \Cref{lem:cvg_sdiv_implies_wc} below, which extends \cite[Lemma S13]{2019arXiv190604516N} to the more general class of Sliced Probability Divergences. 

\begin{lemma}
\label{lem:cvg_sdiv_implies_wc}
  Consider $\sequencek{\mu_k}$ a sequence in $\calP(\Rd)$ satisfying $\lim_{k \rightarrow \infty} \sdelta[1](\mu_k, \mu) = 0$, with $\mu \in \calP(\Rd)$, and assume that the convergence in $\bDelta$ implies the weak convergence in $\calP(\rset)$. Then, there exists an increasing function $\phi : \nset \rightarrow \nset$ such that the subsequence $\sequencek{\mu_{\phi(k)}}$ converges weakly to $\mu$.  
\end{lemma}

\begin{proof}
  We assume that $\lim_{k \rightarrow \infty} \sdelta[1](\mu_k, \mu) = 0$, \ie: 
  \begin{equation}
  \label{eq:weak_conv_1}
    \lim_{k \rightarrow \infty} \int_{\sphereD} \bDelta(\thsss \mu_k, \thsss \mu) \rmd \unifS(\theta) = 0
  \end{equation}
  By \cite[Theorem 2.2.5]{Bogachev2007}, \eqref{eq:weak_conv_1} implies that, there exists an increasing function $\phi : \nset \rightarrow \nset$ such that for $\unifS$-a.e. $\theta \in \sphereD$, $\lim_{k \rightarrow \infty} \bDelta(\thsss \mu_{\phi(k)}, \thsss \mu) = 0$. Since $\bDelta$ is assumed to imply weak convergence in $\calP(\rset)$, then, for $\unifS$-a.e. $\theta \in \sphereD$, $\sequencek{\thsss \mu_{\phi(k)}}$ converges weakly to $\thsss \mu$. By Lévy's characterization \cite[Theorem 4.3]{kallenberg:1997}, we have for $\unifS$-a.e. $\theta \in \sphereD$ and any $s \in \mathbb{R}$, 
  \begin{equation}
      \lim_{k \rightarrow \infty} \Phi_{\thsss \mu_{\phi(k)}}(s) = \Phi_{\thsss \mu}(s) \eqsp,
    \end{equation}
  where $\Phi_\nu$ is the characteristic function of $\nu \in \mcp(\rset^s)$ ($s \geq 1$) and is defined as: for any $v \in \rset^s$, $\Phi_\nu(v) = \int_{\mathbb{R}^s} \rme^{\rmi \langle v, w \rangle} \rmd\nu(w)$. Therefore, for Lebesgue ($\Leb$)-almost every $z \in \mathbb{R}^d$,
  \begin{equation} \label{eqn:cvg_characteristic}
    \lim_{k \rightarrow \infty} \Phi_{\mu_{\phi(k)}}(z) = \Phi_{\mu}(z)  \eqsp. 
  \end{equation}

  We now use \eqref{eqn:cvg_characteristic} to show that $\sequencek{\mu_{\phi(k)}}$ converges weakly to $\mu$. By \cite[Problem 1.11, Chapter 1]{Billingsley1999}, this boils down to proving that, for any $f: \rset^d \to \rset$ continuous with compact support, 
  \begin{equation}
    \label{eq:convo_0}
    \lim_{k \rightarrow \infty} \int_{\mathbb{R}^d} f(z) \rmd\mu_{\phi(k)}(z) = \int_{\mathbb{R}^d} f(z) \rmd\mu(z) \eqsp. 
  \end{equation}

  Consider $\sigma > 0$ and a continuous function $f : \rset^d \to \rset$ with compact support. We introduce the function $f_\sigma$ defined as: for any $x \in \mathbb{R}^d$,
  \begin{equation}
    f_\sigma(x) = (2\uppi \sigma^2)^{-d/2}  \int_{\mathbb{R}^d} f(x-z) \exp\parenthese{-\|z\|^2/{(2\sigma^2)}} \rmd z = f \ast g_\sigma(x) \eqsp,
  \end{equation}
  where $\ast$ denotes the convolution product, and $g_\sigma$ is the density of the $d$-dimensional Gaussian with zero mean and covariance matrix $\sigma^2 \bfI_d$. First, we prove that \eqref{eq:convo_0} holds with $f_{\sigma}$ in place of $f$. The characteristic function associated to a $d$-dimensional Gaussian random variable $G$ with zero mean and covariance matrix $(1/\sigma^2) \bfI_d$ is given by: for any $z \in \rset^d$, $\expe{\rme^{\rmi \ps{z}{G}}} = \rme^{- \norm[2]{z}/(2\sigma^2)}$. By plugging this in the definition of $f_\sigma$ and using Fubini's theorem, we obtain for any $k \in \nset$,
  \begin{align}
  \int_{\mathbb{R}^d} f_\sigma(z) \rmd\mu_{\phi(k)}(z) &= \int_{\mathbb{R}^d} \int_{\mathbb{R}^d} f(w) g_\sigma(z-w) \rmd w \rmd \mu_{\phi(k)}(z) \nonumber \\
  &= (2\uppi \sigma^2)^{-d/2} \int_{\mathbb{R}^d} \int_{\mathbb{R}^d} f(w) \int_{\mathbb{R}^d} \rme^{\rmi \langle z-w, x \rangle}  g_{1/\sigma}(x) \rmd x \rmd w \rmd \mu_{\phi(k)}(z) \nonumber \\
  &= (2\uppi \sigma^2)^{-d/2} \int_{\mathbb{R}^d} \int_{\mathbb{R}^d} f(w) \rme^{-\rmi\langle w,x \rangle } g_{1/\sigma}(x) \Phi_{\mu_{\phi(k)}}(x)  \rmd x \rmd w \nonumber \\
  &= (2\uppi \sigma^2)^{-d/2} \int_{\mathbb{R}^d}   \mathcal{F}[f](x) g_{1/\sigma}(x) \Phi_{\mu_{\phi(k)}}(x) \rmd x \eqsp, \label{eq:convo_1}
  \end{align}
  where $\mathcal{F}[f](x) = \int_{\rset^d} f(w) \rme^{- \rmi \ps{w}{x}} \rmd w$ is the Fourier transform of $f$. Since the support of $f$ is assumed to be compact, $\mathcal{F}[f]$ exists and is bounded by $\int_{\mathbb{R}^d} \abs{f(w)} \rmd w < \plusinfty$, therefore, for any $k \in \mathbb{N}$ and $x \in \mathbb{R}^d$,
  \begin{equation}
    \left| \mathcal{F}[f](x) g_{1/\sigma}(x) \Phi_{\mu_{\phi(k)}}(x) \right| \leq  g_{1/\sigma}(x) \int_{\mathbb{R}^d} |f(w)| \rmd w \eqsp.
  \end{equation}

  We can prove with similar techniques that \eqref{eq:convo_1} holds with $\mu$ in place of $\mu_{\phi(k)}$, \ie:
  \begin{equation}
    \label{eq:convo_2}
    \int_{\mathbb{R}^d} f_\sigma(z) \rmd \mu(z) =  (2\uppi \sigma^2)^{-d/2} \int_{\mathbb{R}^d}  \mathcal{F}[f](x) g_{1/\sigma}(x) \Phi_{\mu}(x) \rmd x \eqsp.
  \end{equation}

  Using \eqref{eqn:cvg_characteristic}, \eqref{eq:convo_1}, \eqref{eq:convo_2} and Lebesgue's Dominated Convergence Theorem, we obtain:
  \begin{align} 
    \lim_{k \rightarrow \infty} (2\uppi \sigma^2)^{-d/2} \int_{\mathbb{R}^d} \mathcal{F}[f](x) g_{1/\sigma}(x) \Phi_{\mu_{\phi(k)}}(x) \rmd x &= (2\uppi \sigma^2)^{-d/2} \int_{\mathbb{R}^d} \mathcal{F}[f](x) g_{1/\sigma}(x) \Phi_\mu(x) \rmd x \eqsp, \nonumber \\
    \ie,~\lim_{k \rightarrow \infty} \int_{\mathbb{R}^d} f_{\sigma}(z) \rmd \mu_{\phi(k)}(z) &= \int_{\mathbb{R}^d} f_\sigma(z) \rmd\mu(z) \eqsp. \label{eqn:cvg_l1_f_sigma} 
  \end{align}

  We can now prove \eqref{eq:convo_0}: for any $\sigma > 0$, 
  \begin{align}
    &\left| \int_{\mathbb{R}^d} f(z) \rmd \mu_{\phi(k)}(z) - \int_{\mathbb{R}^d} f(z) \rmd \mu(z) \right| \\
    &\leq 2\sup_{z \in \mathbb{R}^d} \left| f(z) - f_\sigma(z) \right| + \left| \int_{\mathbb{R}^d} f_\sigma(z) \rmd \mu_{\phi(k)}(z) - \int_{\mathbb{R}^d} f_{\sigma}(z) \rmd \mu(z) \right| \eqsp. 
  \end{align}

  By \eqref{eqn:cvg_l1_f_sigma}, we deduce that for any $\sigma >0$,
  \begin{equation}
    \limsup_{k \to \plusinfty} \left| \int_{\Rd} f(z) \rmd \mu_{\phi(k)}(z) - \int_{\mathbb{R}^d} f(z) \rmd \mu(z) \right| \leq 2\sup_{z \in \mathbb{R}^d} \left| f(z) - f_\sigma(z) \right| \eqsp,
  \end{equation}
  and since $\lim_{\sigma \rightarrow 0} \sup_{z \in \mathbb{R}^d} | f(z) - f_\sigma(z) | = 0$ \cite[Theorem 8.14-b]{folland:1999}, we conclude that $\sequencek{\mu_{\phi(k)}}$ converges weakly to $\mu$.

\end{proof}

We can now prove \Cref{thm:weak_conv}.

\begin{proof}[Proof of \Cref{thm:weak_conv}]

  Let $p \in [1, \infty)$ and $\sequencek{\mu_k}$ be a sequence of probability measures in $\calP(\Rd)$. 

  First, suppose $\sequencek{\mu_k}$ converges weakly to $\mu \in \calP(\Rd)$. By the continuous mapping theorem, since for any $\theta \in \sphereD$, $\thss$ is a bounded linear form thus continuous, then $\sequencek{\thsss \mu_k}$ converges weakly to $\thsss \mu$. Therefore, according to our assumption on $\bDelta$, for any $\theta \in \sphereD$,
  \begin{equation}
  \label{eq:equiv_weakconv_0}
    \lim_{k \rightarrow \infty} \bDelta(\thsss \mu_k, \thsss \mu) = 0 \eqsp .
  \end{equation}

  Besides, $\bDelta$ is assumed to be non-negative and bounded. Hence, there exists $M > 0$ such that, for any $k \in \Nset$,
  \begin{equation}
  \label{eq:equiv_weakconv_1}
    \bDelta^p(\thsss \mu_k, \thsss \mu) \leq M \eqsp .
  \end{equation}

  Using \eqref{eq:equiv_weakconv_0}, \eqref{eq:equiv_weakconv_1} and the bounded convergence theorem, we obtain
  \begin{align}
  \label{eq:equiv_weakconv_2}
    \lim_{k \rightarrow \infty} \sdelta[p]^p(\mu_k, \mu) = \lim_{k \rightarrow \infty} \int_{\sphereD} \bDelta^p(\thsss \mu_k, \thsss \mu) \rmd \unifS(\theta) = \int_{\sphereD} 0^p~\rmd \unifS(\theta) = 0 \eqsp .
  \end{align}

  Since the mapping $t \mapsto t^{1/p}$ is continuous on $\rset+$ (and can be applied to $\sdelta[p]^p$, which is non-negative by the non-negativity of $\bDelta$ and \Cref{thm:metric}), then \eqref{eq:equiv_weakconv_2} implies $\lim_{k \rightarrow \infty} \sdelta[p](\mu_k, \mu) = 0$.

  Now, let us prove the other implication, \ie~$\lim_{k \rightarrow \infty} \sdelta[p]\big( \mu_k, \mu \big) = 0$ implies the weak convergence of $\sequencek{\mu_k}$ to $\mu$, given the assumptions on $\bDelta$. This result is a generalization of \cite[Theorem 1]{2019arXiv190604516N}, and is proved analogously, using \Cref{lem:cvg_sdiv_implies_wc}: consider $\sequencek{\mu_k}$ and $\mu$ in $\calP(\Rd)$ such that
  \begin{equation}
    \label{eqn:assum_sdiv_lim_zero}
    \lim_{k \rightarrow \infty} \sdelta[p](\mu_k, \mu) = 0 \eqsp,   
  \end{equation}
  and suppose $(\mu_k)_{k \in \nset}$ does not converge weakly to $\mu$. Therefore, $\lim_{k \rightarrow \infty} \dist_{\mcp}(\mu_k, \mu) \neq 0$, where $\dist_{\mcp}$ is the L\'evy-Prokhorov metric, \ie~there exists $\epsilon > 0$ and a subsequence $\sequencek{\mu_{\psi(k)}}$ with $\psi : \bbN \rightarrow \bbN$ increasing, such that for any $k \in\nset$,
    \begin{equation} 
      \label{eqn:lim_not_zero}
    \dist_{\mcp}(\mu_{\psi(k)}, \mu) > \epsilon \eqsp.
    \end{equation}
  
  On the other hand, an application of H\"older's inequality on $\sphereD$ gives for any $\mu, \nu$ in $\calP(\Rd)$,
  \begin{equation}
  \label{eq:weak_conv_2}
    \sdelta[1](\mu, \nu) \leq \sdelta[p](\mu, \nu) \eqsp . 
  \end{equation}

  Then, by \eqref{eqn:assum_sdiv_lim_zero}, $\lim_{k \rightarrow \infty} \sdelta[1](\mu_{\psi(k)}, \mu) = 0$. Since we assume the convergence in $\bDelta$ implies the weak convergence in $\calP(\rset)$, \Cref{lem:cvg_sdiv_implies_wc} gives us: there exists a subsequence $\sequencek{\mu_{\phi(\psi(k))}}$ with $\phi : \bbN \rightarrow \bbN$ increasing such that $\sequencek{\mu_{\phi(\psi(k))}}$ converges weakly to $\mu$. This is equivalent to $\lim_{k \rightarrow \infty} \dist_{\mcp}(\mu_{\phi(\psi(k))}, \mu) = 0$, which contradicts \eqref{eqn:lim_not_zero}. We conclude that \eqref{eqn:assum_sdiv_lim_zero} implies the weak convergence of $\sequencek{\mu_k}$ to $\mu$.

\end{proof}

\subsection{Proof of \Cref{thm:weak_topo}} \label{supp:subsec:proof_weak_conv}

\begin{proof}[Proof of \Cref{thm:weak_topo}]

Let $p \in [1, \infty)$ and $\mu, \nu \in \calP(\Rd)$.
\begin{align}
  (\sIpm[\tF, p])^p(\mu, \nu) &= \int_{\sphereD} \bgamma_{\tF}^p(\thsss \mu, \thsss \nu) \rmd \unifS(\theta) \\
  &= \int_{\sphereD} \left\{ \sup_{\tf \in \tF} \left| \int_{\rset} \tf(t)~\rmd (\thsss \mu - \thsss \nu)(t) \right|  \right\}^p \rmd \unifS(\theta) \\
  &= \int_{\sphereD} \left| \int_{\rset} \tf^*(t) \rmd (\thsss \mu - \thsss \nu) (t) \right|^p \rmd \unifS(\theta) \\
  &= \int_{\sphereD} \left| \int_{\Rd} \tf^* \big( \thss(x) \big) \rmd (\mu - \nu) (x) \right|^p \rmd \unifS(\theta) \eqsp, \label{eq:weaktopo_1}
\end{align}
with $\tf^* = \argmax_{\tf \in \tF} \left| \int_{\rset} \tf(t) \rmd \thsss \mu(t) - \int_{\rset} \tf(t) \rmd \thsss \nu(t) \right|$, which is assumed to exist. Note that \eqref{eq:weaktopo_1} results from applying the property of pushforward measures. 

By definition of $\msf$, for any $\theta \in \sphereD$, there exists $f_\theta^* \in \msf$ such that $f_\theta^* = \tf^* \circ \thss$. Therefore, we obtain
\begin{align}
  (\sIpm[\msf, p])^p(\mu, \nu) &= \int_{\sphereD} \left| \int_{\Rd} f_\theta^*(x) \rmd (\mu - \nu)(x) \right|^p \rmd \unifS(\theta) \\
  &\leq \int_{\sphereD} \left\{ \sup_{f \in \msf} \left| \int_{\Rd} f(x) \rmd (\mu - \nu)(x) \right| \right\}^p \rmd \unifS(\theta) \\
  &= \bgamma_{\msf}^p(\mu, \nu)  \int_{\sphereD} \rmd \unifS(\theta) = \bgamma_{\msf}^p(\mu, \nu) \eqsp ,
\end{align}
which completes the proof.

\end{proof}

Informally, the condition on the function classes in \Cref{thm:weak_topo} requires that $\msf$ and $\tF$ should be linked to each other in the way that $\msf$ should be large enough to contain the composition of \emph{all} elements of $\tF$ with \emph{all} possible linear forms $\thss$ for $\ths \in \sphereD$. Let us illustrate this condition by considering the Wasserstein distance of order 1. In this case, $\msf$ is the set of $1$-Lipschitz functions from $\rset^d$ to $\rset$, and $\tF$ is the set of $1$-Lipschitz functions from $\rset$ to $\rset$. Then, the condition on $\msf$ boils down to showing that the composition of any $\tf \in \tF$ with any linear projection $\thss$ results in a $1$-Lipschitz function in $\rset^d$, which is simply true since $\tf$ is $1$-Lipschitz and $\|\theta\|=1$ for all $\theta \in \sphereD$. 

In the next three corollaries, we formally prove that \Cref{thm:weak_topo} holds for the Wasserstein distance of order 1 $\wassersteinD[1]$, total variation distance $\TV$ and maximum mean discrepancy $\MMD[]$. We denote by $\swassersteinD[1]$, $\sTV[p]$ and $\sMMD[p]$ the respective sliced versions of these IPMs with order $p \in [1, \infty)$.

\begin{corollary} \label{cor:weak_topo_w1}
  Let $p \in [1, \infty)$. For any $\mu, \nu \in \calP_1(\Rd)$,\; $\swassersteinD[1](\mu, \nu) \leq \wassersteinD[1](\mu, \nu)$.
\end{corollary}

\begin{proof}
  Choose $\tF = \setLigne{\tf : \rset \rightarrow \rset}{\| \tf \|_{\Lip} \leq 1}$, where $\| \tf \|_{\Lip} = \sup_{x, y \in \Rd, x \neq y} \big\{\fracb{\big|\tf(x) - \tf(y)\big|}{\norm{x-y}} \big\}$. Let $f : \Rd \to \rset$ such that $f = \tf \circ \thss$ with $\tf \in \tF, \ths \in \sphereD$. Then, by using the Cauchy-Schwarz inequality and the definition of $\tF$, we have for any $x,y \in \Rd$,
  \begin{equation}
    \abs{f(x) - f(y)} = \big| \tf \big( \thss(x) \big) - \tf \big( \thss(y) \big) \big| \leq \big| \ps{\ths}{x-y} \big| \leq \| \thss \|~\| x-y \| \leq \| x-y \| \eqsp .
  \end{equation}

  Therefore, $f \in \msf = \setLigne{f : \Rd \rightarrow \rset}{\| f \|_{\Lip} \leq 1}$. \Cref{cor:weak_topo_w1} follows from the application of \Cref{thm:weak_topo} along with the definition of $\wassersteinD[1]$.

\end{proof}

Note that \Cref{cor:weak_topo_w1} is not a new result: the fact that $\swassersteinD[p]$ is bounded above by $\wassersteinD[p]$ for $p \in [1, \infty)$ was established in \cite[Proposition 5.1.3]{Bonnotte2013}. While their result is proved using the primal formulation of the OT problem, we used the dual formulation available for $p = 1$ to illustrate the applicability of \Cref{thm:weak_topo}. Our result is thus consistent with the existing results in the literature.

\begin{corollary}
  Let $p \in [1, \infty)$. For any $\mu, \nu \in \calP(\Rd)$,
  \begin{equation}
    \sTV[p](\mu, \nu) \leq \TV(\mu, \nu) \eqsp .
  \end{equation}
\end{corollary}

\begin{proof}
  Choose $\tF = \big\{\tf : \rset \rightarrow \rset, \; \| \tf \|_\infty \leq 1 \big\}$, and let $f : \Rd \to \rset$ such that $f = \tf \circ \thss$ with $\tf \in \tF, \ths \in \sphereD$. Then,
  \begin{equation}
    \norminf{f} = \| \tf \circ \thss \|_{\infty} = \sup_{x \in \Rd} \big| \tf \big( \thss(x) \big) \big| \leq \sup_{t \in \rset} \big| \tf(t) \big| = \| \tf \|_{\infty} \leq 1 \eqsp,
  \end{equation}
  hence, $f \in \msf = \set{f : \rset^d \rightarrow \rset}{\norminf{f} \leq 1}$. We obtain the final result by using \Cref{thm:weak_topo} and the definition of TV.

\end{proof}

\begin{corollary}
  Let $\tF \subset \mbm_b(\rset)$ be the unit ball of the RKHS with reproducing kernel $\tilde{k}$, and $k$ be the positive definite kernel such that for any $x_i, x_j \in \Rd$, 
  \begin{align}
    k(x_i, x_j) = \int_{\sphereD} \tilde{k}\big( \thss(x_i), \thss(x_j) \big) \rmd \unifS(\ths) \eqsp . 
  \end{align}
  Define $\msf \subset \mbm_b(\rset^d)$ as the unit ball of the RKHS whose reproducing kernel $\hat{k}$ satisfies $k - \hat{k}$ is positive definite. Then, for any $p \in [1, \infty)$ and $\mu, \nu \in \calP(\Rd)$,
  \begin{equation}
    \sMMD[p](\mu, \nu ; \tF) \leq \MMD[](\mu, \nu ; \msf) \eqsp ,
  \end{equation}
  where $\MMD[](\cdot, \cdot~;~\msf')$ and $\sMMD[p](\cdot, \cdot~;~\msf')$  respectively denote the MMD and the Sliced-MMD of order $p$ in the RKHS whose unit ball is $\msf'$.

  In particular, this property holds for
  \vspace{-2mm}
  \begin{enumerate}[label=(\roman*), itemsep=0mm]
    \item Linear kernels: $\tilde{k}(t_i, t_j) = t_i t_j$ for $t_i, t_j \in \rset$, and $\hat{k}(x_i, x_j) = x_i^\top x_j / d'$ for $x_i, x_j \in \rset$ and $d' \geq d$.
    \item Radial basis function (RBF) kernels: let $h \geq 0$, $\tilde{k}(t_i,t_j) = e^{-|t_i-t_j|^2/h}$ for $t_i, t_j \in \rset$, and $\hat{k}(x_i,x_j) = e^{- \|x_i-x_j\|^2/h}$ for $x_i, x_j \in \Rd$.
  \end{enumerate}
\end{corollary}

\begin{proof}
  Define $\tF$ as the unit ball of an RKHS whose reproducing kernel is denoted by $\tilde{k}$. Then, any $\tf \in \tF$ satisfies
  \begin{align}
      \| \tilde{f}\|^2_{\tF} = \sum_{i=1}^n \sum_{j=1}^n \alpha_i \alpha_j \tilde{k}(t_i, t_j) \leq 1, \label{eq:mmd_weaktopo_1}
  \end{align}
  where $n \in \bbN^*$, $\alpha_1, \dots, \alpha_n \in \rset$ and $t_1, \dots, t_n \in \rset$.

  Consider $f : \Rd \to \rset$ such that $f = \tilde{f} \circ \theta^*$ with $\tf \in \tF$ and $\ths \in \sphereD$. By \eqref{eq:mmd_weaktopo_1}, we have 
  \begin{align}
      \sum_{i=1}^n \sum_{j=1}^n \alpha_i \alpha_j \tilde{k}\big( \thss(x_i), \thss(x_j) \big) \leq 1 \label{eq:mmd_weaktopo_1bis}
  \end{align}

  The integration of \eqref{eq:mmd_weaktopo_1bis} over $\sphereD$ give us
  \begin{align}
      \int_{\sphereD} \sum_{i=1}^n \sum_{j=1}^n \alpha_i \alpha_j \tilde{k}\big( \thss(x_i), \thss(x_j) \big) \rmd \unifS(\theta) &\leq \int_{\sphereD} 1~\rmd \unifS(\ths) \\
      \ie, \;\; \sum_{i=1}^n \sum_{j=1}^n \alpha_i \alpha_j \int_{\sphereD} \tilde{k}\big( \thss(x_i), \thss(x_j) \big) \rmd \unifS(\theta) &\leq 1 \eqsp . \label{eq:mmd_weaktopo_1ter}
  \end{align}

  Define $k : \Rd \times \Rd \to \rset$ as $k(x_i, x_j) = \int_{\sphereD} \tilde{k}\big( \thss(x_i), \thss(x_j) \big) \rmd \unifS(\theta)$ for $x_i, x_j \in \Rd$. Since $\tilde{k}$ is positive definite, so is $k$. By the Moore-Aronszajn theorem, there exists a unique RKHS with reproducing kernel $k$. Therefore, \eqref{eq:mmd_weaktopo_1ter} means that $f$ is in the unit ball of the RKHS associated with $k$.

  Additionally, consider a positive definite kernel $\hat{k} : \Rd \times \Rd \to \rset$ such that $k-\hat{k}$ is positive definite on $\Rd$. In other words, the following holds for any $n \in \bbN$, $v_1, \dots, v_n \in \rset$ and $x_1, \dots, x_n \in \Rd$,
  \begin{align}
      \sum_{i=1}^n \sum_{j=1}^n  v_i v_j \{ k(x_i,x_j) -\hat{k}(x_i,x_j) \} \geq 0 \eqsp . 
  \end{align}

  Then, by \eqref{eq:mmd_weaktopo_1ter}, we obtain $\sum_{i=1}^n \sum_{j=1}^n \alpha_i \alpha_j \hat{k}(x_i,x_j)  \leq 1$.

  Therefore, any $f$ defined as $f = \tf \circ \ths$ with $\tf \in \tF$ and $\ths \in \sphereD$ is in the unit ball of the RKHS associated with $\hat{k}$, which we denote by $\msf$. By using \Cref{thm:weak_topo} and the definition of MMD, we obtain the desired result: for any $p \in [1, \infty)$ and $\mu, \nu \in \calP(\Rd)$,
  \begin{align}
    \sMMD[p](\mu, \nu ; \tF) \leq \MMD[](\mu, \nu ; \msf) \eqsp . \label{eq:mmd_relation}
  \end{align}

  Next, we show that this result holds for two popular choices of kernels. First, we choose $\tilde{k}$ as the linear kernel: $\tilde{k}(t_i,t_j) = t_i t_j$ for $t_i, t_j \in \rset$. Define $\hat{k}$ as a rescaled version of the linear kernel in $\Rd$: $\hat{k}(x_i,x_j) = x_i^\top x_j/d'$ for $x_i, x_j \in \Rd$ and $d' \geq d$. Then, for any $n \in \bbN$, $v_1, \dots, v_n \in \rset$ and $x_1, \dots, x_n \in \Rd$,
  \begin{align}
    \sum_{i=1}^n \sum_{j=1}^n v_i v_j \{ k(x_i, x_j) - \hat{k}(x_i, x_j) \}  &= \sum_{i=1}^n \sum_{j=1}^n v_i v_j \Big\{ \int_{\sphereD} \ths(x_i) \ths(x_j) \rmd \unifS(\ths) - x_i^\top x_j / d' \Big\} \\
    &= \sum_{i=1}^n \sum_{j=1}^n v_i v_j \Big\{ x_i^\top \Big(\int_{\sphereD}\theta \theta^\top \rmd \unifS(\theta) \Big) x_j - x_i^\top x_j/d' \Big\} \\
    &= \sum_{i=1}^n \sum_{j=1}^n v_i v_j x_i^\top x_j \Big(1/d - 1/d' \Big) \geq 0 \eqsp , \label{eq:mmd_linear_weaktopo}
  \end{align}
  where \eqref{eq:mmd_linear_weaktopo} results from $\sum_{i=1}^n \sum_{j=1}^n v_i v_j x_i^\top x_j \geq 0$ (the linear kernel is positive definite) and $d' \geq d$. We conclude that \eqref{eq:mmd_relation} holds with $\tF$ defined as the unit ball of the RKHS associated with the linear kernel $\tilde{k}(t_i, t_j) = t_i t_j$ for $t_i, t_j \in \rset$, and $\msf$ being the unit ball of the RKHS associated with the rescaled linear kernel $\hat{k}(x_i, x_j) = x_i^\top x_j / d'$ for $x_i, x_j \in \Rd$ and $d' \geq d$. 

  We conclude that \eqref{eq:mmd_relation} holds with $\tF$ defined as the unit ball of the RKHS associated with the linear kernel $\tilde{k}(t_i, t_j) = t_i t_j$ for $t_i, t_j \in \rset$, and $\msf$ being the unit ball of the RKHS associated with the rescaled linear kernel $\hat{k}(x_i, x_j) = x_i^\top x_j / d$ for $x_i, x_j \in \Rd$. 

  We focus now on RBF kernels: let $h \geq 0$ and choose $\tilde{k}(t_i,t_j) = e^{-|t_i-t_j|^2/h}$ for $t_i, t_j \in \rset$, and $\hat{k}(x_i,x_j) = e^{- \|x_i-x_j\|^2/h}$ for $x_i, x_j \in \Rd$. We have for any $x_i, x_j \in \Rd$,
  \begin{align}
    k(x_i,x_j) &= \int_{\sphereD} \tilde{k} \big( \ths(x_i), \ths(x_j) \big) \rmd \unifS(\theta) = \int_{\sphereD} e^{- |\theta^\top x_i - \theta^\top x_j|^2 / h}~\rmd \unifS(\theta) \\
    &= \int_{\sphereD} e^{- |\theta^\top( x_i - x_j)|^2/h}~\rmd \unifS(\theta) \\
    &= \int_{\sphereD} e^{(- \|x_i-x_j\|^2 / h) (\theta^\top ( x_i - x_j ) / \|x_i-x_j\|)^2} \rmd \unifS(\theta) \\
    &= M\left(\frac1{2}, \frac{d}{2}, - \frac{\|x_i-x_j\|^2}{h} \right) \eqsp , \label{eq:mmd_rbf_weaktopo_1}
  \end{align}
  where $M(a,c,\kappa)$ stands for the confluent hypergeometric function evaluated at $a,c,\kappa \in \rset$, and appears in the normalizing constant of the multivariate Watson distribution: see \cite[Section 2.3]{sra2016directional} for more details. 

  $M$ satisfies the following property
  \begin{align}
      M\left( \frac1{2}, \frac{d}{2}, - \frac{\|x_i-x_j\|^2}{h} \right) = e^{- \|x_i-x_j\|^2 / h}~M\left(\frac{d-1}{2}, \frac{d}{2},  \frac{\|x_i-x_j\|^2}{h} \right) \eqsp . \label{eq:mmd_rbf_weaktopo_2}
  \end{align}
  
  Since $\|x_i-x_j\|^2 / h \geq 0$ and $\kappa \mapsto M(\cdot,\cdot, \kappa)$ is increasing, we have 
  \begin{align}
      M\left( \frac{d-1}{2}, \frac{d}{2},  \frac{\|x_i-x_j\|^2}{h} \right) \geq M \left( \frac{d-1}{2}, \frac{d}{2},  0 \right) = M \left( \frac{1}{2}, \frac{d}{2},  0 \right) = 1 \eqsp. \label{eq:mmd_rbf_weaktopo_3}
  \end{align}

  Finally, by using \eqref{eq:mmd_rbf_weaktopo_1} and \eqref{eq:mmd_rbf_weaktopo_2}, we obtain: for any $n \in \bbN$, $v_1, \dots, v_n \in \rset$ and $x_1, \dots, x_n \in \Rd$,
  \begin{align}
    \sum_{i=1}^n \sum_{j=1}^n v_i v_j \{ k(x_i,x_j)- \hat{k}(x_i,x_j) \} &= \sum_{i=1}^n \sum_{j=1}^n v_i v_j \Biggl[  M \left( \frac1{2}, \frac{d}{2}, - \frac{\|x_i-x_j\|^2}{h} \right) -  e^{- \|x_i-x_j\|^2/{h}} \Biggr] \\
      &= \sum_{i=1}^n \sum_{j=1}^n v_i v_j e^{-\|x_i-x_j\|^2 / h} \Biggl[ M\left(\frac{d-1}{2}, \frac{d}{2},  \frac{\|x_i-x_j\|^2}{h} \right) - 1 \Biggr] \\
      & \geq 0 \eqsp ,
  \end{align}
  where the last line follows from \eqref{eq:mmd_rbf_weaktopo_3} and $\sum_{i=1}^n \sum_{j=1}^n v_i v_j e^{-\|x_i-x_j\|^2 / h} \geq 0$ (RBF kernels are positive definite). We conclude that $k - \hat{k}$ is positive definite, hence \eqref{eq:mmd_relation} holds for RBF kernels.

\end{proof}

\subsection{Proof of \Cref{thm:ipm_lb}}

\begin{proof}[Proof of \Cref{thm:ipm_lb}]

We start by upper bounding the distance between two regularized measures. Denote by $\supp(\zeta)$ the support of the function $\zeta$. Let $\varphi : \rset \to \rset_+^*$ be a smooth and even function verifying $\supp(\varphi) \subset \ccint{-1,1}$ and $\int_{\rset} \varphi(t) \rmd \Leb(t) = 1$. Define $\varphi_{\lambda}(x) = \lambda^{-d} \varphi(\norm{x}/\lambda)/\calA(\sphere^{d-1})$, with $\calA(\sphere^{d-1})$ denoting the surface area of the $d$-dimensional unit sphere: $\calA(\sphere^{d-1}) = 2\pi^{d/2}/\Gamma(d/2)$, where $\Gamma$ is the gamma function. Denote by $\calF[f]$ the Fourier transform of any function $f$ defined on $\rset^s$ ($s \geq 1$), given by: for any $x \in \rset ^s$, $\calF[f](x) = \int_{\rset^s} f(w) e^{-\rmi \ps{w}{x}} \rmd w$. Let $g \in \msg$. By the isometry properties of the Fourier transform and the definition of $\varphi_\lambda$, we have
\begin{align}
  \int_{\Rd} g(x) \rmd (\mu_\lambda - \nu_\lambda)(x) = \int_{\Rd} \calF[g] (w) \left\{\calF[\mu](w) - \calF[\nu](w)\right\} \calF[\varphi](\lambda w) \rmd w \eqsp,
\end{align}
where $\mu_\lambda = \mu \ast \varphi_\lambda $ and $\nu_\lambda = \nu \ast \varphi_\lambda $. By representing $w$ with its polar coordinates $(r,\theta) \in [0,\infty) \times \sphereD$, we obtain
\begin{align}
  \int_{\Rd} g(x) \rmd (\mu_\lambda - \nu_\lambda)(x) = \int_{\sphereD} \int_0^\infty \calF[g](r \theta) \left\{ \calF[\mu](r \theta) - \calF[\nu](r \theta) \right\} \calF[\varphi](\lambda r) r^{d-1} \rmd r \rmd \unifS(\theta) \eqsp.
\end{align}

Since $g$ is a real function, $\calF[g]$ is an even function, hence
\begin{align}
  &\int_{\Rd} g(x) \rmd (\mu_\lambda - \nu_\lambda)(x) \\
  &= \frac1{2} \int_{\sphereD} \int_\rset \calF[g](r \theta) \left\{ \calF[\mu](r \theta) - \calF[\nu](r \theta) \right\} \calF[\varphi](\lambda r) \abs{r}^{d-1} \rmd r \rmd \unifS(\theta) \\
  &= \frac1{2} \int_{\sphereD} \int_\rset \calF[g](r \theta) \left\{ \calF[\thsss\mu] (r) - \calF[\thsss\nu] (r) \right\} \calF[\varphi](\lambda r) \abs{r}^{d-1} \rmd r \rmd \unifS(\theta) \label{eq:lb_proof_0} \\
  &= \frac1{2} \int_{\sphereD} \int_\rset \int_{-R}^R \calF[g](r \theta) e^{- \rmi r u} \rmd(\thsss\mu - \thsss\nu)(u) \calF[\varphi](\lambda r) \abs{r}^{d-1}  \rmd r \rmd \unifS(\theta) \label{eq:lb_proof_1} \\
  &= \frac1{2} \int_{\sphereD} \int_\rset \int_{\rset^d} \int_{-R}^R g(x) e^{- \rmi r(u + \langle \theta, x \rangle) } \left\{ \rmd(\thsss\mu - \thsss\nu)(u) \right\} \calF[\varphi](\lambda r) \abs{r}^{d-1} \rmd x \rmd r \rmd \unifS(\theta) \eqsp, 
\end{align}
where \eqref{eq:lb_proof_0} follows from \eqref{eq:ft_equal}, \eqref{eq:lb_proof_1} results from the definition of the Fourier transform 
and the fact that $u \in  [-R,R]$, and in the last line, we used the definition of the Fourier transform and Fubini's theorem. By making the change of variables $x \to x - u \theta$, we obtain
\begin{align}
  &\int_{\Rd} g(x) \rmd (\mu_\lambda - \nu_\lambda)(x) \\
  &= \frac1{2} \int_{\sphereD} \int_\rset \int_{\rset^d} \int_{-R}^R  g(x - u \theta) e^{- \rmi r \langle \theta, x \rangle } \rmd(\thsss\mu - \thsss\nu)(u) \calF[\varphi](\lambda r) \abs{r}^{d-1}  \rmd x \rmd r \rmd \unifS(\theta) \eqsp.
\end{align}
Since we assumed $\supp(\mu)$, $\supp(\nu)$ are included in $B_d({\bf0}, R)$, then $\supp(\mu_\lambda)$, $\supp(\mu_\lambda)$ are in $B_d({\bf0}, R + \lambda)$, and the domain of $x \mapsto g(x - u \theta)$ must be contained in $B_d({\bf0},2R+\lambda)$. By Fubini's theorem and the definition of $\tG$, we have
\begin{align}
&\Biggl|\int_{\Rd} g(x) \rmd (\mu_\lambda - \nu_\lambda)(x) \Biggr| \\
&\leq \frac1{2}  \int_\rset \int_{B_d({\bf0},2R+\lambda)}\int_{\sphereD} \Bigl| \int_{-R}^R g(x - u \theta) \rmd(\thsss\mu - \thsss\nu)(u)  e^{-\rmi r \langle \theta, x \rangle }  \calF[\varphi](\lambda r) \abs{r}^{d-1} \Bigr| \rmd \unifS(\theta) \rmd x \rmd r  \\
&\leq \frac1{2}  \int_\rset \int_{B_d({\bf0},2R+\lambda)}\int_{\sphereD}  \bgamma_{\tG}(\thsss \mu, \thsss \nu) \Bigl|e^{-\rmi r \langle \theta, x \rangle }  \calF[\varphi](\lambda r) \abs{r}^{d-1} \Bigr| \rmd \unifS(\theta) \rmd x \rmd r  \\
&\leq C (2R+\lambda)^d  \int_{\sphereD}  \bgamma_{\tG}(\thsss \mu, \thsss \nu) \rmd \unifS(\theta)  \int_\rset \lambda^{-d} \Bigl|\calF[\varphi](r) \abs{r}^{d-1} \Bigr| \rmd r \label{eq:lb_proof_2} \\
&\leq C (2R+\lambda)^d \lambda^{-d}  \Biggl(\int_{\sphereD}  \bgamma_{\tG}^p(\thsss \mu, \thsss \nu) \rmd \unifS(\theta) \Biggr)^{1/p}  \int_\rset \Bigl|\calF[\varphi](r) |r|^{d-1} \Bigr| \rmd r \label{eq:lb_proof_3} \\
&\leq C_1 (2R+\lambda)^d \lambda^{-d} \sIpm[\tG, p](\mu, \nu) \eqsp, \label{eq:lb_proof_3b}
\end{align}

where in \eqref{eq:lb_proof_2}, $C > 0$ and does not depend on $\mu$ and $\nu$, \eqref{eq:lb_proof_3} results from applying H\"older's inequality on $\sphereD$ if $p > 1$, and in \eqref{eq:lb_proof_3b}, $C_1 = C\int_\rset \Bigl|\calF[\varphi](r) |r|^{d-1} \Bigr| \rmd r$.

By using the definition of $\bgamma_{\msg}$ and \eqref{eq:lb_proof_3b}, we obtain
\begin{align}
  \label{eqn:ipm_lb_partial}
  \bgamma_{\msg}(\mu_\lambda, \nu_\lambda) = \sup_{g \in \msg} \Bigl| \int_{\Rd} g(x) \rmd (\mu_\lambda - \nu_\lambda)(x) \Bigr|  \leq C_1 (2R+\lambda)^d \lambda^{-d}  \sIpm[\tG, p](\mu, \nu) \eqsp .
\end{align}

We now relate $\bgamma_{\msg}(\mu_\lambda, \nu_\lambda)$ with $\bgamma_{\msg}(\mu, \nu)$. We start with the following estimate
\begin{align}
  &\int_{\Rd} g(x) \rmd(\mu - \nu)(x) - \bgamma_{\msg}(\mu_\lambda, \nu_\lambda) \\
  &\leq \int_{\Rd} g(x) \rmd(\mu - \nu)(x) - \int_{\Rd} g(x) \rmd(\mu_\lambda - \nu_\lambda)(x) \\
  &\leq \int_{\Rd} \bigl| g(x) - (\varphi_\lambda \ast g)(x) \bigr| \rmd\mu(x) + \int_{\Rd} \bigl| g(x) - (\varphi_\lambda \ast g)(x) \bigr| \rmd\nu(x) \label{eq:lb_proof_4} \\
\end{align}

Since we assumed any $g \in \msg$ is $\Lipg$-Lipschitz continuous, we can bound the integrand in \eqref{eq:lb_proof_4} as follows: for $x \in \Rd$,
\begin{align}
  \bigl| g(x) - (\varphi_\lambda \ast g)(x) \bigr| &= \Bigl| \lambda^{-d} \int_{\Rd} \big( g(x) - g(y) \big) \varphi\big( (x-y) / \lambda \big)  \rmd y \Bigr| \\
  &\leq \lambda^{-d} \int_{\Rd} \bigl| g(x) - g(y) \bigr| \varphi\big( (x-y) / \lambda \big)  \rmd y \\
  &\leq \Lipg \lambda^{-d+1} \int_{\Rd} \| x-y \| \lambda^{-1} \varphi\big( (x-y) / \lambda \big)  \rmd y \\
  &\leq \Lipg \lambda^{-d+1} \int_{\Rd} \| u \| \lambda^{-1} \varphi\big( u / \lambda \big)  \rmd u \leq \Lipg \lambda \int \| z \| \varphi(z)  \rmd z \eqsp.
\end{align}

Hence, by denoting by $M_1(\varphi)$ the moment of order 1 of $\varphi$, \eqref{eq:lb_proof_4} is bounded by
\begin{align}
  \int_{\Rd} g(x) \rmd(\mu - \nu)(x) - \bgamma_{\msg}(\mu_\lambda, \nu_\lambda) \leq 2 \Lipg M_1(\varphi) \lambda \eqsp.
\end{align}

Taking the supremum of both sides over $\msg$ gives us 
\begin{align}
  \bgamma_{\msg}(\mu, \nu) - \bgamma_{\msg}(\mu_\lambda, \nu_\lambda) \leq 2 \Lipg M_1(\varphi) \lambda \eqsp .
\end{align}

By combining the above inequality with \eqref{eqn:ipm_lb_partial}, we get
\begin{align}
  \bgamma_{\msg}(\mu, \nu) &\leq C_1 (2R+\lambda)^d \lambda^{-d}  \sIpm[\tG, p](\mu, \nu) + 2 \Lipg M_1(\varphi) \lambda \\
  &\leq C_2 \lambda \Bigl((2R+\lambda)^d \lambda^{- (d+1)}  \sIpm[\tG, p](\mu, \nu) + 1 \Bigr) \eqsp ,
\end{align}
with $C_2$ satisfying $C_2 \geq C_1$ and $C_2 \geq 2 \Lipg M_1(\varphi)$. Finally, by choosing $\lambda = R^{d /(d+1)}  \sIpm[\tG, p](\mu, \nu)^{1 /(d+1)}$ and using the hypothesis that $\sIpm[\tG, p]$ is bounded, we obtain
\begin{align}
  \bgamma_{\msg}(\mu, \nu) &\leq C_2 R^{d /(d+1)}  \sIpm[\tG, p](\mu, \nu)^{1 /(d+1)} \Bigl( (2R+\lambda)^d R^{-d}   + 1 \Bigr) \\
  &\leq C_p  \sIpm[\tG, p](\mu, \nu)^{1 /(d+1)} ,
\end{align}
for some $C_p > 0$, as desired. This concludes the proof.

\end{proof}

As with \Cref{thm:weak_topo}, \Cref{thm:ipm_lb} assumes that the function classes $\msg$ and $\tG$ are linked to each other and sufficiently regular. The condition on $\msg$ is verified with $\wassersteinD[1]$ (simply by definition) and MMD (provided that the reproducing kernel is Lipschitz-continuous, which holds on compact spaces for classical choices of kernels), but not with TV. On the other hand, the second condition requires $\tG$ to be large enough to contain \emph{any} possible slice $g(x-u\theta)$ for any $g \in \msg$. %

\subsection{Proof of \Cref{cor:ipm_ub2}}

\begin{proof}[Proof of \Cref{cor:ipm_ub2}]
The desired result is obtained as a direct application of \Cref{thm:weak_topo,thm:ipm_lb}.

\end{proof}

\subsection{Proof of \Cref{thm:sliced_sample_complexity}}

\begin{proof}[Proof of \Cref{thm:sliced_sample_complexity}]
  Let $p \in [1, \infty)$ and $\mu, \nu$ in $\calP(\Rd)$ with respective empirical measures $\hmu_n, \hnu_n$. By using the definition of $\sdelta[p]$, the triangle inequality and the assumption on the sample complexity of $\bDelta^p$, we have
  \begin{align}
    \E \abs{\sdelta[p]^p(\mu, \nu) - \sdelta[p]^p(\hmu_n, \hnu_n)} &= \E\left|\int_{\sphereD} \big\{ \bDelta^p(\thsss \mu, \thsss \nu) - \bDelta^p(\thsss \hmu_n, \thsss \hnu_n) \big\} \rmd \unifS(\theta) \right| \\
    &\leq \E\left\{ \int_{\sphereD} \big| \bDelta^p(\thsss \mu, \thsss \nu) - \bDelta^p(\thsss \hmu_n, \thsss \hnu_n) \big| \rmd \unifS(\theta) \right\} \\
    &\leq \int_{\sphereD} \E \big| \bDelta^p(\thsss \mu, \thsss \nu) - \bDelta^p(\thsss \hmu_n, \thsss \hnu_n) \big| \rmd \unifS(\theta) \\
    &\leq \int_{\sphereD} \beta(p, n) \rmd \unifS(\theta) = \beta(p, n) \eqsp ,
  \end{align}
  which completes the proof.

\end{proof}

\subsection{Proof of \Cref{thm:sliced_rateofconv}}

\begin{proof}[Proof of \Cref{thm:sliced_rateofconv}]
  Let $p \in [1,\infty)$ and $\mu \in \calP(\Rd)$ with corresponding empirical measure $\hmu_n$. By using the definition of $\sdelta[p]$, the triangle inequality and the assumed convergence rate of empirical measures in $\bDelta^p$, we obtain the convergence rate in $\sdelta[p]$ as follows
  \begin{align}
    \E\abs{\sdelta[p]^p(\hmu_n, \mu)} &= \E \abs{\int_{\sphereD} \bDelta^p(\thsss \hmu_n, \thsss \mu) \rmd \unifS(\theta)} \leq \E \left\{ \int_{\sphereD} \abs{\bDelta^p(\thsss \hmu_n, \thsss \mu)} \rmd \unifS(\theta) \right\} \\
    &\leq \int_{\sphereD} \E\abs{\bDelta^p(\thsss \hmu_n, \thsss \mu)} \rmd \unifS(\theta) \leq \int_{\sphereD} \alpha(p, n) \rmd \unifS(\theta) = \alpha(p, n) \eqsp . \label{eq:rateofconv_sliced}
  \end{align}

  Additionally, if we assume that $\bDelta$ satisfies non-negativity, symmetry and the triangle inequality, then $\sdelta[p]$ also verifies these three properties by \Cref{thm:metric}, and we can derive its sample complexity: for any $\mu, \nu$ in $\calP(\Rd)$ with respective empirical measures $\hmu_n, \hnu_n$, the triangle inequality give us
  \begin{align}
    \abs{\sdelta[p](\mu, \nu) - \sdelta[p](\hmu_n, \hnu_n)} \leq \sdelta[p](\hmu_n, \mu) + \sdelta[p](\hnu_n, \nu) \label{eq:sliced_sample_complexity_0}
  \end{align}

  By taking the expectation of \eqref{eq:sliced_sample_complexity_0} with respect to $\hmu_n, \hnu_n$, we obtain
  \begin{align}
    \E \abs{\sdelta[p](\mu, \nu) - \sdelta[p](\hmu_n, \hnu_n)} &\leq \E \abs{\sdelta[p](\hmu_n, \mu)} + \E \abs{\sdelta[p](\hnu_n, \nu)} \\
    &\leq \left\{ \E \abs{\sdelta[p]^p(\hmu_n, \mu)} \right\}^{1/p} + \left\{ \E \abs{\sdelta[p]^p(\hnu_n, \nu)}\right\}^{1/p} \label{eq:sliced_sample_complexity_1} \\
    &\leq \alpha(p,n)^{1/p} + \alpha(p,n)^{1/p} = 2 \alpha(p,n)^{1/p} \eqsp , \label{eq:sliced_sample_complexity_2}
  \end{align}

  where \eqref{eq:sliced_sample_complexity_1} results from applying H\"older's inequality on $\sphereD$ if $p > 1$, and \eqref{eq:sliced_sample_complexity_2} follows from the convergence rate result in \eqref{eq:rateofconv_sliced}.

\end{proof}

\subsection{Proof of \Cref{thm:projection_complexity}}

\begin{proof}[Proof of \Cref{thm:projection_complexity}]
  Let $p \in [1, \infty)$ and $\mu, \nu \in \calP(\Rd)$. We recall that $\hatsdelta[p, L](\mu, \nu)$ denotes the approximation of $\sdelta[p](\mu, \nu)$ obtained with a Monte Carlo scheme that uniformly picks $L$ projection directions on $\sphereD$ (cf. Equation \eqref{eqref:mc_estimate} in the main document).

  By using H\"older's inequality and the results on the moments of the Monte Carlo estimation error, we obtain 
  \begin{align}
    \E_{\theta \sim \unifS} \big| \hatsdelta[p, L]^p(\mu, \nu) - \sdelta[p]^p(\mu, \nu) \big| &\leq \big\{ \E_{\theta \sim \unifS} \big| \hatsdelta[p, L]^p(\mu, \nu) - \sdelta[p]^p(\mu, \nu) \big|^2 \big\}^{1/2} \\
    &\leq L^{-1/2} \left\{ \int_{\sphere^{d-1}} \defEns{\bDelta^p(\thsss \mu, \thsss \nu) - \sdelta[p]^p(\mu, \nu) }^2 \rmd \unifS(\theta) \right\}^{1/2} \eqsp,
  \end{align}
  Since $\sdelta[p]^p(\mu, \nu) = \int_{\sphere^{d-1}} \bDelta^p(\thsss \mu, \thsss \nu) \rmd \unifS(\theta)$ by definition, the quantity $\int_{\sphere^{d-1}} \defEns{\bDelta^p(\thsss \mu, \thsss \nu) - \sdelta[p]^p(\mu, \nu) }^2 \rmd \unifS(\theta)$ is the variance of $\bDelta^p(\thsss \mu, \thsss \nu)$ with respect to $\ths \sim \unifS$.

\end{proof}

\subsection{The overall complexity}

We now leverage \Cref{thm:sliced_sample_complexity,thm:projection_complexity} to derive the \emph{overall complexity} of sliced divergences, \ie~the convergence rate of $\hatsdelta[p](\hmu_n, \hnu_n)$ to $\sdelta[p](\mu, \nu)$. This result is useful as it helps understanding the behavior of sliced divergences in most practical applications, where $\sdelta[p](\mu,\nu)$ is approximated using finite sets of samples drawn from $\mu$ and $\nu$ along with Monte Carlo estimates. \\

\begin{corollary} \label{thm:overall_complexity}
  Let $p \in [1, \infty)$ and $\mu, \nu \in \calP(\Rd)$. Denote by $\hmu_n$ (respectively, $\hnu_n$) the empirical distribution computed over a sequence of \iid~random variables $X_{1:n} = \{X_k\}_{k=1}^n$ from $\mu$ (resp., $Y_{1:n} = \{Y_k\}_{k=1}^n$ from $\nu$). Assume $\bDelta^p$ admits the following sample complexity: for any $\mu', \nu' \in \calP(\rset)$ and empirical instantiations $\hmu'_n, \hnu'_n$, $\E[ \abs{\bDelta^p(\mu', \nu') - \bDelta^p( \hmu_n',  \hnu_n')}]\leq \beta(p, n)$. Then,
  \begin{multline}
    \E \big[ |\hatsdelta[p, L]^p(\hmu_n, \hnu_n) - \sdelta[p]^p(\mu, \nu)| \big] \leq \beta(p, n) \\
+    L^{-1/2}~\parentheseDeux{\int_{\sphereD}\expe{ \parenthese{\bDelta^p(\thsss \hmu_n, \thsss \hnu_n) - \sdelta[p]^p(\hmu_n, \hnu_n) }^2} \rmd \unifS(\theta) }^{1/2}  \eqsp,
  \end{multline}
  where $\hatsdelta[p, L]^p(\hmu_n, \hnu_n)$ is defined by \eqref{eqref:mc_estimate}, and $\E$ is the expectation with respect to (\wrt) $X_{1:n}$, $Y_{1:n}$ and $\{\theta_l\}_{l= 1}^L$ \iid~from the uniform distribution on $\sphereD$.
\end{corollary}

\begin{proof}[Proof of \Cref{thm:overall_complexity}]
  Let $p \in [1, \infty)$, $\mu, \nu \in \calP(\Rd)$ and the respective empirical distributions $\hmu_n, \hnu_n$. By the triangle inequality,
  \begin{equation}
    |\hatsdelta[p, L]^p(\hmu_n, \hnu_n) - \sdelta[p]^p(\mu, \nu)| \leq |\hatsdelta[p, L]^p(\hmu_n, \hnu_n) - \sdelta[p]^p(\hmu_n, \hnu_n) | + | \sdelta[p]^p(\hmu_n, \hnu_n) - \sdelta[p]^p(\mu, \nu)| \eqsp .
  \end{equation}

  Therefore, by linearity of expectation, we have
  \begin{align}
    &\E \big[ |\hatsdelta[p, L]^p(\hmu_n, \hnu_n) - \sdelta[p]^p(\mu, \nu)| \big]  \\
    &\leq \E \left[  \E [|\hatsdelta[p, L]^p(\hmu_n, \hnu_n) - \sdelta[p]^p(\hmu_n, \hnu_n) |\big|~X_{1:n}, Y_{1:n}] \right] + \E \left[ | \sdelta[p]^p(\hmu_n, \hnu_n) - \sdelta[p]^p(\mu, \nu)|  \right] \eqsp.   \label{eq:projcpx_0} 
  \end{align}

  We bound the left term in \eqref{eq:projcpx_0}. By \Cref{thm:projection_complexity}, we have
  \begin{align}
    &\E \big[| \hatsdelta[p, L]^p(\hmu_n, \hnu_n) - \sdelta[p]^p(\hmu_n, \hnu_n) |~\big|~X_{1:n}, Y_{1:n} \big] \\
    &\leq L^{-1/2} \left\{ \int_{\sphere^{d-1}} \defEns{\bDelta^p(\thsss \hmu_n, \thsss \hnu_n) - \sdelta[p]^p(\hmu_n, \hnu_n) }^2 \rmd \unifS(\theta) \right\}^{1/2} \eqsp . \label{eq:projcpx_1}
  \end{align}

  By taking the expectation then using Jensen's inequality, we get 
  \begin{align}
    &\E \left[ \E \big[| \hatsdelta[p, L]^p(\hmu_n, \hnu_n) - \sdelta[p]^p(\hmu_n, \hnu_n) |~\big|~X_{1:n}, Y_{1:n} \big] \right] \\
    &\leq L^{-1/2}~\E \left[ \left\{ \int_{\sphereD} \defEns{\bDelta^p(\thsss \hmu_n, \thsss \hnu_n) - \sdelta[p]^p(\hmu_n, \hnu_n) }^2 \rmd \unifS(\theta) \right\}^{1/2} \right] \\
    &\leq L^{-1/2}~\E^{1/2} \left[ \int_{\sphereD} \defEns{\bDelta^p(\thsss \hmu_n, \thsss \hnu_n) - \sdelta[p]^p(\hmu_n, \hnu_n) }^2 \rmd \unifS(\theta) \right] \label{eq:projcpx_2} \eqsp.
  \end{align}

  Next, we bound the right term in \eqref{eq:projcpx_0}: by the sample complexity assumption for $\bDelta^p$ and \Cref{thm:sliced_sample_complexity}, we have
  \begin{equation} \label{eq:projcpx_3}
    \E \left[ | \sdelta[p]^p(\hmu_n, \hnu_n) - \sdelta[p]^p(\mu, \nu)|  \right] \leq \beta(p, n) \eqsp.
  \end{equation}

  Combining \eqref{eq:projcpx_2} and \eqref{eq:projcpx_3} in \eqref{eq:projcpx_0} completes the proof.

\end{proof}

\begin{remark}
  Note that by Fubini's theorem, $\int_{\sphereD}\expeLigne{ \parentheseLigne{\bDelta^p(\thsss \hmu_n, \thsss \hnu_n) - \sdelta[p]^p(\hmu_n, \hnu_n) }^2} \rmd \unifS(\theta)$ (which appears in \Cref{thm:overall_complexity}) is equal to  $\PE[\Var\{ \bDelta^p(\thsss \hmu_n, \thsss \hnu_n) | X_{1:n},Y_{1:n}\}]$, where $\Var$ is the variance \wrt~$X_{1:n}$, $Y_{1:n}$ and $\theta$ (which is distributed according to the uniform distribution on $\sphereD$  and independent of $X_{1:n}, Y_{1:n}$). 
\end{remark}

\section{Postponed proofs for \Cref{sec:applications}} \label{supp:sec:proofs_applications}

\subsection{Applications of \Cref{thm:weak_conv}}

As discussed in \Cref{sec:applications}, we can use the general result in \Cref{thm:weak_conv} to establish novel topological properties for specific sliced probability divergences, for example the Sliced-Cram\'er distance (whose definition is recalled in \Cref{def:slicedcramer}) and the broader class of Sliced-IPMs. We present our results and proofs for these examples below.

\begin{definition}[Cram\'er distance \cite{Cramer1928}]
  Let $p \in [1, \infty)$ and $\mu, \nu \in \calP(\rset)$. Denote by $F_\mu, F_\nu$ the cumulative distribution functions of $\mu, \nu$ respectively. The Cram\'er distance of order $p$ between $\mu$ and $\nu$ is defined by
  \begin{equation}
    \cramerD[p]^p(\mu, \nu) = \int_{\rset} \abs{F_\mu(t) - F_\nu(t)}^p \rmd t \eqsp .
  \end{equation}  
\end{definition}

\begin{definition}[Sliced-Cram\'er distance \cite{kolouri2020sliced}] \label{def:slicedcramer}
  Let $p \in [1, \infty)$ and $\mu, \nu \in \calP(\Rd)$. The Sliced-Cram\'er distance of order $p$ between $\mu$ and $\nu$ is defined by
  \begin{equation}
    \scramerD[p]^p(\mu, \nu) = \int_{\sphereD} \cramerD[p]^p(\thsss \mu, \thsss \nu) \rmd \unifS(\theta) \eqsp .
  \end{equation}
\end{definition}

\begin{corollary}
  Let $p \in [1, \infty)$. For any sequence $\sequencek{\mu_k}$ in $\calP(\Rd)$ and $\mu \in \calP(\Rd)$, $\lim_{k \rightarrow \infty} \scramerD[p] \big( \mu_k, \mu \big ) = 0$ implies $\sequencek{\mu_k}$ converges weakly to $\mu$. Besides, if $\sequencek{\mu_k}$ and $\mu$ are supported on a compact space $\msk \subset \Rd$, then the converse implication holds, meaning that the convergence under $\scramerD[p]$ is equivalent to the weak convergence in $\calP(\msk)$.
\end{corollary}

\begin{proof}
  Let $p \in [1, \infty)$. By H\"older's inequality, for any $\mu', \nu' \in \calP(\rset)$, we have
  \begin{equation} \label{eq:cramer_holder}
    \cramerD[1](\mu', \nu') \leq \cramerD[p](\mu', \nu') \eqsp .
  \end{equation}

  Consider a sequence $\sequencek{\mu'_k}$ in $\calP(\rset)$ and $\mu' \in \calP(\rset)$ such that $\lim_{k \rightarrow \infty} \cramerD[p](\mu'_k, \mu') = 0$. By \eqref{eq:cramer_holder}, this implies $\lim_{k \rightarrow \infty} \cramerD[1](\mu'_k, \mu') = 0$. Since the Cram\'er distance of order 1 is equivalent to the Wasserstein distance of order 1, then by \cite[Theorem 6.8]{villani2008optimal}, the convergence of $\sequencek{\mu'_k}$ to $\mu'$ under $\cramerD[p]$ implies $\sequencek{\mu'_k}$ converges weakly to $\mu'$ in $\calP(\rset)$. By \Cref{thm:weak_conv}, we conclude that the convergence under $\scramerD[p]$ implies the weak convergence in $\calP(\Rd)$.

  We now show the second part of the statement. This result partly follows from slight modifications of the techniques we used in the proof of \Cref{thm:weak_conv}. Consider a compact
  space $\msk' \subset \rset$ and a sequence $\sequencek{\mu'_k}$ in
  $\calP(\msk')$. Suppose that $\sequencek{\mu'_k}$ converges weakly
  to $\mu' \in \calP(\msk')$. Since $F_{\mu'}$ is non-decreasing, it is almost everywhere continuous \wrt~to the Lebesgue
  convergence, and using the Portmanteau theorem, we get that for $\Leb$-almost every
  $t \in \rset$,
  $\lim_{k \rightarrow \infty} F_{\mu'_k}(t) = F_{\mu'}(t)$. Besides,
  for any $k \in \bbN$ and $t \in \msk'$,
  $\absLigne{F_{\mu_k'}(t)} \leq 1$, and since $\msk'$ is compact,
  $\left( \int_{\msk'} 1^p \rmd t \right)^{1/p} < \infty$. By the
  dominated convergence theorem in $\mathrm{L}^p$-spaces, we conclude
  that
  \begin{equation}
    \label{eq:proof_cramer_conv}
    \lim_{k \rightarrow \infty} \left\{ \int_{\msk'} | F_{\mu'_k}(t) - F_{\mu'}(t) |^p \rmd t \right\}^{1/p} = 0 \eqsp ,
  \end{equation}

  in other words, the weak convergence of measures in $\calP(\msk')$, where $\msk'$ is a compact subspace of $\rset$, implies the convergence under $\cramerD[p]$.

  Now, consider a compact space $\msk \subset \Rd$ and a sequence
  $\sequencek{\mu_k}$ in $\calP(\msk)$ which converges weakly to
  $\mu \in \calP(\msk)$. For any $\theta \in \sphereD$, define
  $\msk_\theta = \set{\ps{\theta}{x}}{x \in \msk}$, which is a compact
  subset of $\rset$ (since it is the image of $\msk$ by a continuous
  function) with $\diam(\msk_\theta) \leq \diam(\msk)$ (by the Cauchy-Schwarz inequality). The sequence of pushforward measures $\sequencek{\thsss \mu_k}$ is in $\calP(\msk_\theta)$ and, by the
  continuous mapping theorem, converges weakly to $\thsss \mu \in \calP(\msk_\theta)$. Therefore, by \eqref{eq:proof_cramer_conv}, for any $\theta \in\sphere^{d-1}$,
  \begin{equation}
    \label{eq:proof_cramer_conv_2}
    \lim_{k \rightarrow \infty} \cramerD[p](\thsss \mu_k, \thsss \mu) =
  0 \eqsp.
  \end{equation}
  Besides, for any $\mu,\nu \in \calP(\rset^d)$ with support in $\msk$, and $\theta \in\sphere^{d-1}$,
  \begin{multline}
    \label{eq:2}
    \cramerD[p](\thsss \nu, \thsss \mu) =  \int_{\rset} \abs{F_\nu(t) - F_\mu(t)}^p \rmd t =\int_{\msk_{\theta}} \abs{F_\nu(t) - F_\mu(t)}^p \rmd t\\
    \leq 2^p \diam(\msk_{\theta}) \leq 2^p \diam(\msk) \eqsp.
  \end{multline}
  By \eqref{eq:proof_cramer_conv_2} and the dominated convergence theorem, we finally obtain 
  $\lim_{k \rightarrow \infty} \scramerD[p](\mu_k, \mu) = 0$.

\end{proof}

\begin{corollary}
  Let $p \in [1, \infty)$ and $\tF \subset \mbm_b(\rset)$. Suppose that the space spanned by $\tF$ is dense in the space of continous functions for $\| \cdot \|_\infty$. Then, the convergence under the Sliced Integral Probability Metric of order $p$ associated with $\tF$, $\sIpm[\tF, p]$~, implies the weak convergence in $\calP(\Rd)$. Besides, if $\bgamma_{\tF}$ is bounded, the converge implication holds, \ie~the weak convergence in $\calP(\Rd)$ implies the convergence under $\sIpm[\tF, p]$.
\end{corollary}

\begin{proof}
  By construction of $\tF$ and \cite[Section 5.1]{ambrosio2008gradient}, $\bgamma_{\tF}$ metrizes the weak convergence in $\calP(\rset)$, \ie~the weak convergence in $\calP(\rset)$ is equivalent to the convergence of measures under $\bgamma_{\tF}$. The properties of $\sIpm[\tF, p]$, $p \in [1, \infty)$ result from the application of \Cref{thm:weak_conv}.

\end{proof}

\begin{remark}
  The boundedness assumption for $\bgamma_{\tF}$ is achieved if we additionally suppose that $\tF$ is a uniformly bounded family of functions in $\mbm(\rset)$, which is a mild assumption.
\end{remark}

\subsection{Proof of \Cref{thm:smpl_cpx_sw}}

\begin{lemma}
  \label{lem:rateofconv_wass_pushfor}
  Let $p \in [1, \infty)$ and $\mu' \in \calP(\rset)$ with empirical distribution $\hmu'_n$. Suppose there exists $q > p$ such that the moment of order $q$ of $\mu'$, defined as $M_q(\mu') = \int_{\rset} \abs{t}^q \rmd \mu'(t)$, is bounded above by $K < \infty$. Then, there exists a constant $C_{p,q}$ depending on $p, q$ such that
  \begin{equation}
    \expe{\wassersteinD[p]^p(\hmu'_n, \mu')} \leq C_{p, q} K \left\{ 
    \begin{array}{ll}
      n^{-1/2} & \mbox{ if } q > 2p, \label{eq:rateofconv_wass_pushfor_1} \\
      n^{-1/2} \log(n) & \mbox{ if } q = 2p, \\
      n^{-(q-p)/q} & \mbox{ if } q \in (p,2p).
    \end{array}
    \right.
  \end{equation}
\end{lemma}

\begin{proof}
  This immediately results from \cite[Theorem 1]{Fournier2015}.
\end{proof}

\begin{proof}[Proof of \Cref{thm:smpl_cpx_sw}]
  We first recall that, for any $\xi \in \calP(\rset^s)$ ($s \geq 1$) and $\theta \in \sphereD$, the moment of order $k > 0$ of $\thsss \xi$ is lower than the one associated with $\xi$. Indeed, by using the property of pushforward measures, the Cauchy-Schwarz inequality, and $\norm{\theta} \leq 1$, we have
  \begin{equation}
    M_k(\thsss \xi) = \int_{\rset} \abs{t}^{k} \rmd \thsss \xi(t) = \int_{\Rd} \abs{\ps{\theta}{x}}^k \rmd \xi(x) \leq \int_{\Rd} \norm{x}^k \rmd \xi(x) = M_k(\xi) \eqsp . \label{eq:moments}
  \end{equation}

  Now, let $p \in [1, \infty)$ and $\mu \in \calP_q(\Rd)$ ($q > p$) with empirical distribution $\hmu_n$. Then, by \eqref{eq:moments}, for any $\theta \in \sphereD$, $M_q(\thsss \mu) \leq M_q(\mu) < \infty$, and we can apply \Cref{lem:rateofconv_wass_pushfor} and \Cref{thm:sliced_rateofconv} to derive the convergence rate under $\swassersteinD[p]$~: there exists a constant $C_{p,q}$ such that,
  \begin{align}
    \expe{\swassersteinD[p]^p(\hmu_n, \mu)} &\leq C_{p, q} M_q^{p/q} (\mu) \left\{ 
    \begin{array}{ll}
      n^{-1/2} & \mbox{ if } q > 2p, \\
      n^{-1/2} \log(n) & \mbox{ if } q = 2p, \\
      n^{-(q-p)/q} & \mbox{ if } q \in (p,2p).
    \end{array} \label{eq:convrate_sw}
    \right.
  \end{align}

  Besides, since $\wassersteinD[p]$ is a metric, we can apply \Cref{thm:sliced_rateofconv} to derive the sample complexity of $\swassersteinD[p]$. Consider $\mu, \nu \in \calP_q(\Rd)$ with $q > p$, with respective empirical measures $\hmu_n, \hnu_n$. Then, starting from \eqref{eq:sliced_sample_complexity_1} and using the convergence rate derived in \eqref{eq:convrate_sw}, we obtain the desired result as follows
  \begin{align}
    &\E \abs{\swassersteinD[p](\mu, \nu) - \swassersteinD[p](\hmu_n, \hnu_n)} \\
    &\leq \left\{ \E \abs{\swassersteinD[p]^p(\hmu_n, \mu)} \right\}^{1/p} + \left\{ \E \abs{\swassersteinD[p]^p(\hnu_n, \nu)}\right\}^{1/p} \\
    &\leq C_{p, q}^{1/p} \big(M_q^{1/q} (\mu) + M_q^{1/q} (\nu)\big)  \left\{ 
    \begin{array}{ll}
      n^{-1/(2p)} & \mbox{ if } q > 2p, \\
      n^{-1/(2p)} \log(n)^{1/p} & \mbox{ if } q = 2p, \\
      n^{-(q-p)/(pq)} & \mbox{ if } q \in (p,2p).
    \end{array}
    \right.
  \end{align}
  
\end{proof}

\subsection{Proof of \Cref{thm:weak_topo_sink}}

\begin{proof}[Proof of \Cref{thm:weak_topo_sink}]

Let $p \in [1, \infty)$ and $\veps \geq 0$. We use the reformulation of $\wassersteinD[p,\veps]$ as the maximum of an expectation, as given in \cite[Proposition 2.1]{Genevay2016},
\begin{align}
  \swassersteinD[p, \veps]^p(\mu, \nu) &= \int_{\sphereD} \wassersteinD[p, \veps]^p(\thsss \mu, \thsss \nu) \rmd \unifS(\theta) \\
  &= \int_{\sphereD} \Biggl\{ \max_{\tu,\tv \in \mrC(\rset)}  \mathbb{E}_{\thsss \mu \otimes \thsss \nu} \Bigl[\phi_\veps \big( \tu(\tX),\tv(\tY),\tX,\tY \big) \Bigr] \Biggr\}^p \rmd \unifS(\theta)\eqsp, \label{eqn:1dsink}
\end{align}
where $\mrC(\rset)$ denotes the set of continuous real functions, and $\phi_\veps(t,s,x,y) = t + s - \veps e^{(t + s - \| x - y \|^p)/\veps}$.

Consider for any $\theta \in \sphereD$, $\tu^\star_\theta$, $\tv^\star_\theta$ as the functions attaining the maximum in \eqref{eqn:1dsink}, which exist by \cite[Theorem 4 in the supplementary document]{Genevay19}. We obtain
\begin{align}
  \swassersteinD[p, \veps]^p(\mu, \nu) &= \int_{\sphereD} \Biggl\{  \mathbb{E}_{\thsss \mu \otimes \thsss \nu} \Bigl[\phi_\veps\big(\tu^\star_\theta(\tX), \tv^\star_\theta(\tY), \tX, \tY \big) \Bigr] \Biggr\}^p \rmd \unifS(\theta) \\
  &= \int_{\sphereD} \Biggl\{  \mathbb{E}_{\mu \otimes \nu} \Bigl[\phi_\veps\big(\tu^\star_\theta \circ \thss (X), \tv^\star_\theta \circ \thss (Y), X, Y \big) \Bigr] \Biggr\}^p \rmd \unifS(\theta)\eqsp. \label{eq:proof_thm7_1}
\end{align}
Since for all $\tilde{w} \in \mrC(\rset)$ and $\theta \in \sphereD$, $\tilde{w} \circ \theta^\star \in \mrC(\rset^d)$, we can bound \eqref{eq:proof_thm7_1} as follows
\begin{align}
  \swassersteinD[p, \veps]^p(\mu, \nu) \leq \int_{\sphereD} \Biggl\{  \max_{u,v \in \mrC(\rset^d)} \mathbb{E}_{\mu \otimes \nu} \Bigl[\phi_\veps \big(u(X), v(Y), X, Y \big) \Bigr] \Biggr\}^p \rmd \unifS(\theta) = \wassersteinD[p, \veps]^p(\mu, \nu) \eqsp. \label{eq:proof_thm7_2}
\end{align}

By \Cref{thm:metric}, since $\wassersteinD[p, \veps]$ is non-negative, so is $\swassersteinD[p, \veps]$, and we can apply $t \mapsto t^{1/p}$ on both sides of \eqref{eq:proof_thm7_2} to obtain the final result.

\end{proof}

\subsection{Proof of \Cref{thm:smpl_cpx_ssink}}

\begin{proposition} \label{prop:upperbound_1d_sinkhorn}
    Let $\tilde{\msx}$ be a compact subset of $\rset$, and $\mu', \nu' \in \calP(\tilde{\msx})$ with respective empirical instantiations $\hmu'_n, \hnu'_n$. Let $p \in [1, \infty)$ and $\veps \geq 0$. Then, 
    \begin{equation}
      \abs{\wassersteinD[p, \veps](\hmu'_n, \hnu'_n) - \wassersteinD[p, \veps](\mu', \nu')} \leq 2~\diam(\tilde{\msx}) \left\{ \wassersteinD[1](\mu', \hmu'_n) + \wassersteinD[1](\nu', \hnu'_n) \right\} \eqsp . \label{eq:sink_1d_bound}
    \end{equation}
\end{proposition}

\begin{proof}
  Let $p \in [1, \infty)$, $\veps \geq 0$ and $\tilde{\msx} \subset \rset$ compact. Consider $\mu', \nu' \in \calP(\tilde{\msx})$ with respective empirical distributions $\hmu'_n, \hnu'_n$. We first express the regularized OT cost as the maximum of an expectation \cite[Proposition 2.1]{Genevay2016}
  \begin{align}
    \wassersteinD[p, \veps](\mu', \nu') &= \max_{\tu,\tv \in \mrC(\rset)}  \E_{\mu' \otimes \nu'} \bigl[\phi_\veps\big(\tu(\tX),\tv(\tY),\tX,\tY\big) \bigr] \label{eq:slicedsink_complexity_0} \\
    \wassersteinD[p, \veps](\hmu'_n, \nu') &= \max_{\tu,\tv \in \mrC(\rset)}  \E_{\hmu'_n \otimes \nu'} \bigl[\phi_\veps\big(\tu(\tX),\tv(\tY),\tX,\tY\big) \bigr] \label{eq:slicedsink_complexity_1} \eqsp,
  \end{align}
  where $\phi_\veps(t,s,x,y) = t + s - \veps e^{(t + s - \| x - y \|^2/2)/\veps}$. By \cite[Proposition 1]{Genevay19}, the Sinkhorn potentials $(\tu, \tv)$ are Lipschitz continuous with Lipschitz constant $\diam(\tilde{\msx}) < \infty$. Therefore, by denoting by $\Lip_{\diam(\tilde{\msx})}(\rset)$ the space of $\diam(\tilde{\msx})$-Lipschitz continuous functions defined on $\rset$, \eqref{eq:slicedsink_complexity_0} and \eqref{eq:slicedsink_complexity_1} can be rewritten with the maximization over $\Lip_{\diam(\tilde{\msx})}(\rset)$. 

  We can now use \cite[Proposition 2]{Mena2019} to bound the absolute difference of $\wassersteinD[p, \veps](\mu', \nu')$ and $\wassersteinD[p, \veps](\hmu'_n, \nu')$. We provide the detailed proof below for completeness. 
  By \cite[Proposition 6, Appendix A]{Mena2019}, there exist smooth potentials $(\tu^\star, \tv^\star)$ attaining the maximum in \eqref{eq:slicedsink_complexity_0} such that, for all $\tx, \ty \in \rset$,
  \begin{align}
    \int_\rset \phi_\veps(\tu^\star(\tx), \tv^\star(\ty), \tx, \ty) \rmd \nu'(\ty) &= 1 \;\;\; \mu'\text{-almost surely}, \label{eq:cond1} \\
    \int_\rset \phi_\veps(\tu^\star(\tx), \tv^\star(\ty), \tx, \ty) \rmd \mu'(\tx) &= 1 \;\;\; \nu'\text{-almost surely} \label{eq:cond2} \eqsp .
  \end{align}

  Analogously, there exist smooth optimal potentials $(\tu_n^\star, \tv_n^\star)$ for \eqref{eq:slicedsink_complexity_1} satisfying \eqref{eq:cond1} and \eqref{eq:cond2} where $\tu^\star$, $\tv^\star$ and $\mu'$ are replaced by $\tu_n^\star$, $\tv_n^\star$ and $\hmu'_n$ respectively. 

  The optimality of these potentials give us
  \begin{align}
    &\E_{\mu' \otimes \nu'} \bigl[\phi_\veps(\tu_n^\star(\tX),\tv_n^\star(\tY),\tX,\tY) \bigr] - \E_{\hmu'_n \otimes \nu'} \bigl[\phi_\veps(\tu_n^\star(\tX),\tv_n^\star(\tY),\tX,\tY) \bigr] \\
    &\leq \E_{\mu' \otimes \nu'} \bigl[\phi_\veps(\tu^\star(\tX),\tv^\star(\tY),\tX,\tY) \bigr] - \E_{\hmu'_n \otimes \nu'} \bigl[\phi_\veps(\tu_n^\star(\tX),\tv_n^\star(\tY),\tX,\tY) \bigr] \\
    &\leq \E_{\mu' \otimes \nu'} \bigl[\phi_\veps(\tu^\star(\tX),\tv^\star(\tY),\tX,\tY) \bigr] - \E_{\hmu'_n \otimes \nu'} \bigl[\phi_\veps(\tu^\star(\tX),\tv^\star(\tY),\tX,\tY) \bigr] \eqsp .
  \end{align}
  Therefore, 
  \begin{align}
    &\abs{\wassersteinD[p, \veps](\mu', \nu') - \wassersteinD[p, \veps](\hmu'_n, \nu')} \\
    &= \abs{\E_{\mu' \otimes \nu'} \bigl[\phi_\veps(\tu^\star(\tX),\tv^\star(\tY),\tX,\tY) \bigr] - \E_{\hmu'_n \otimes \nu'} \bigl[\phi_\veps(\tu_n^\star(\tX),\tv_n^\star(\tY),\tX,\tY) \bigr] } \\
    &\leq \abs{\E_{\mu' \otimes \nu'} \bigl[\phi_\veps(\tu^\star(\tX),\tv^\star(\tY),\tX,\tY) \bigr] - \E_{\hmu'_n \otimes \nu'} \bigl[\phi_\veps(\tu^\star(\tX),\tv^\star(\tY),\tX,\tY) \bigr]} \\
    &\;\; + \abs{\E_{\mu' \otimes \nu'} \bigl[\phi_\veps(\tu_n^\star(\tX),\tv_n^\star(\tY),\tX,\tY) \bigr] - \E_{\hmu'_n \otimes \nu'} \bigl[\phi_\veps(\tu_n^\star(\tX), \tv_n^\star(\tY),\tX,\tY) \bigr]} \eqsp . \label{eq:slicedsink_complexity_7}
  \end{align}

  We bound each term of the sum in \eqref{eq:slicedsink_complexity_7} as follows
  \begin{align}
    & \abs{\E_{\mu' \otimes \nu'} \bigl[\phi_\veps(\tu^\star(\tX),\tv^\star(\tY),\tX,\tY) \bigr] - \E_{\hmu'_n \otimes \nu'} \bigl[\phi_\veps(\tu^\star(\tX),\tv^\star(\tY),\tX,\tY) \bigr]} \\
    &= \Big| \int_{\rset} \tu^\star(\tx) \rmd(\mu' - \hmu'_n)(\tx) - \veps \int_{\rset} \int_{\rset} e^{(\tu^\star(\tx) + \tv^\star(\ty) - \abs{\tx - \ty}^2/2) / \veps} \rmd \nu'(\ty) \rmd(\mu' - \hmu'_n)(\tx) \Big|  \\
    &= \Big| \int_{\rset}  \tu^\star(\tx) \rmd(\mu' - \hmu'_n)(\tx) \Big| \leq \sup_{\tu \in \Lip_{\diam(\tilde{\msx})}(\rset)} \Big| \int_{\rset} \tu(\tx) \rmd(\mu' - \hmu'_n)(\tx) \Big| \eqsp , \label{eq:slicedsink_complexity_4}
  \end{align}

  where $\eqref{eq:slicedsink_complexity_4}$ results from \eqref{eq:cond1}. Since for any $f \in \Lip_{\Lipg}(\rset)$ with $\Lipg > 0$, $f / \Lipg \in \Lip_1(\rset)$, \eqref{eq:slicedsink_complexity_4} can be bounded as follows
  \begin{align}
    & \abs{\E_{\mu' \otimes \nu'} \bigl[\phi_\veps(\tu^\star(\tX),\tv^\star(\tY),\tX,\tY) \bigr] - \E_{\hmu'_n \otimes \nu'} \bigl[\phi_\veps(\tu^\star(\tX),\tv^\star(\tY),\tX,\tY) \bigr]} \\
    &\leq \diam(\tilde{\msx})  \sup_{\tu \in \Lip_{1}(\rset)} \Big| \int_{\rset} \tu(\tx) \rmd(\thsss \mu - \thsss \hmu_n)(\tx) \Big| = \diam(\tilde{\msx}) \wassersteinD[1](\mu', \hmu'_n) \eqsp , \label{eq:slicedsink_complexity_5}
  \end{align}

  where \eqref{eq:slicedsink_complexity_5} follows from the dual formulation of the Wasserstein distance of order 1 \cite[Theorem 5.10]{villani2008optimal}. 

  We show with an analogous proof that
  \begin{align}
    \abs{\E_{\mu' \otimes \nu'} \bigl[\phi_\veps(\tu_n^\star(\tX),\tv_n^\star(\tY),\tX,\tY) \bigr] - \E_{\hmu'_n \otimes \nu'} \bigl[\phi_\veps(\tu_n^\star(\tX), \tv_n^\star(\tY),\tX,\tY) \bigr]} \leq  \diam(\tilde{\msx}) \wassersteinD[1](\mu', \hmu'_n) \eqsp ,
  \end{align}

  which leads to the conclusion that
  \begin{equation}
  \abs{\wassersteinD[p, \veps](\mu', \nu') - \wassersteinD[p, \veps](\hmu'_n, \nu')} \leq 2~\diam(\tilde{\msx}) \wassersteinD[1](\mu', \hmu'_n) \eqsp . \label{eq:slicedsink_complexity_6}
  \end{equation}

  By using the triangle inequality and \eqref{eq:slicedsink_complexity_6}, we obtain the final result
  \begin{align}
    \abs{\wassersteinD[p, \veps](\hmu'_n, \hnu'_n) - \wassersteinD[p, \veps](\mu', \nu')} &\leq \abs{\wassersteinD[p, \veps](\mu', \nu') - \wassersteinD[p, \veps](\hmu'_n, \nu')} + \abs{\wassersteinD[p, \veps](\hmu'_n, \nu') - \wassersteinD[p, \veps](\hmu'_n, \hnu'_n)} \\
    &\leq 2~\diam(\tilde{\msx}) \left\{ \wassersteinD[1](\mu', \hmu'_n) + \wassersteinD[1](\nu', \hnu'_n) \right\} \eqsp .
  \end{align}
\end{proof}

\begin{corollary} \label{cor:sample_cpx_1d_sink}
  Let $\tilde{\msx}$ be a compact subset of $\rset$, and $\mu', \nu' \in \calP(\tilde{\msx})$. Denote by $\hmu'_n, \hnu'_n$ their respective empirical instantiations. Let $p \in [1, \infty)$ and $\veps \geq 0$. Then, 
  \begin{equation}
  \E \abs{\wassersteinD[p, \veps](\hmu'_n, \hnu'_n) - \wassersteinD[p, \veps](\mu', \nu')} \leq 2~\diam(\tilde{\msx}) C_q \big[ M_q^{1/q} (\mu') + M_q^{1/q} (\nu') \big] n^{-1/2} \eqsp ,
  \end{equation}
  where $q >2$, $C_q < \infty$ is a constant that depends on $q$, and $M_q(\mu'), M_q(\nu')$ are the moments of order $q$ of $\mu', \nu'$ respectively.  
\end{corollary}

\begin{proof}
  We apply \Cref{prop:upperbound_1d_sinkhorn} and take the expectation of \eqref{eq:sink_1d_bound} with respect to $\tX_{1:n} \sim \hmu'_n$ and $\tY_{1:n} \sim \hnu'_n$
  \begin{equation}
    \E \abs{\wassersteinD[p, \veps](\hmu'_n, \hnu'_n) - \wassersteinD[p, \veps](\mu', \nu')} \leq 2~\diam(\tilde{\msx}) \E \left\{ \wassersteinD[1](\mu', \hmu'_n) + \wassersteinD[1](\nu', \hnu'_n) \right\} \eqsp . \label{eq:sample_cpx_1d_sink_0}
  \end{equation}

  Since $\mu'$ and $\nu'$ are both supported on a compact space, they have infinitely many finite moments. We can then bound \eqref{eq:sample_cpx_1d_sink_0} using the convergence rate of empirical measures in $\wassersteinD[1]$, recalled in \Cref{lem:rateofconv_wass_pushfor}. This concludes the proof.

\end{proof}

\begin{proof}[Proof of \Cref{thm:smpl_cpx_ssink}]
  Let $p \in [1, \infty)$ and $\veps \geq 0$. Consider $\mu, \nu \in \calP(\msx)$ with $\msx \subset \Rd$ compact, and denote by $\hmu_n, \hnu_n$ their respective empirical distributions. 

  Let $\ths \in \sphereD$ and define $\msx_\theta = \set{\ps{\theta}{x}}{x \in \msx}$. $\msx_\theta$ is compact (since $\msx$ is compact and $\theta^\star$ is continuous) and verifies $\diam(\msx_\theta) \leq \diam(\msx)$ (by the Cauchy-Schwarz inequality). Besides, by \eqref{eq:moments}, for any $k > 0$, $M_k(\thsss \mu) \leq M_k(\mu)$ and $M_k(\thsss \nu) \leq M_k(\nu)$. By \Cref{cor:sample_cpx_1d_sink}, there exists $C_q < \infty$ which depends on $q > 2$ such that,
  \begin{align}
    \E \abs{\wassersteinD[p, \veps](\thsss \hmu_n, \thsss \hnu_n) - \wassersteinD[p, \veps](\thsss \mu, \thsss \nu)}  \leq 2~\diam(\msx) C_q \big[ M_q^{1/q} (\mu) + M_q^{1/q} (\nu) \big] n^{-1/2} \eqsp .
  \end{align}

  The sample complexity of $\swassersteinD[p, \veps]$ is finally obtained by applying \Cref{thm:sliced_sample_complexity}.

\end{proof}

\subsection{Proof of \Cref{prop:contract}}

Sinkhorn's algorithm refers to an iterative procedure which operates on empirical distributions as follows: consider a cost matrix $C$ between two sets of $n$ samples, and define the matrix $K$ with $K_{i,j}=\exp(-C_{i,j}/\veps)$ for $1 \leq i, j \leq n$, and initialize $b^{(0)}= {1}\in \mathbb{R}^n$ ; then, compute for $\ell>1$, $a^{(\ell)} = 1 ./ n(Kb^{(\ell-1)})$, $b^{(\ell)} = 1 ./ n(Ka^{(\ell)})$, where $./$ stands for the entry-wise division. This defines a sequence $\gamma^{(\ell)}_{i,j} = a^{(\ell)}_iK_{i,j}b^{(\ell)}_j$, which converges to a solution of \eqref{eq:def_reg_ot} at a linear rate. The convergence rate of Sinkhorn's algorithm is recalled in \Cref{thm:cvgcesinkhorn}. For an extended discussion on this result, we refer to \cite[Section 4.2]{peyre2019computational}.

\begin{theorem}[\cite{franklin1989scaling}]
\label{thm:cvgcesinkhorn}
The iterates $a^{(\ell)}$ and $b^{(\ell)}$ of Sinkhorn's algorithm converge linearly for the Hilbert metric at a rate $1-\tanh(\tau(K)/4)$, with $\tau(K) = \log \max_{i,j,i',j'}\frac{K_{ij}K_{i'j'}}{K_{ij'}K_{i'j}}$. In particular, for the squared-norm cost, \ie~$K_{ij}=\exp(-\Vert x_i -x_j\Vert^2/\veps$), it holds
\begin{align}
\tau(K) \leq 2\max_{i,j}\Vert x_i-x_j\Vert^2 / \veps.
\label{eqn:sink_conv}
\end{align}
\end{theorem}

\begin{proof}[Proof of \Cref{prop:contract}]
For $i,j \in \{1,\dots,n\}$, the function $f_{i,j}: \theta \in \mathbb{S}^{d-1} \mapsto \frac{1}{R}\ps{\theta}{x_i-x_j}$ is $1$-Lipschitz and has median $0$ for $\theta$ uniformly distributed on the unit sphere. Thus, by concentration of measure on the sphere~\cite[Example 3.12]{wainwright2019high}, it holds for $\vareps>0$,
$$
\mathbb{P}\left( \vert f_{i,j}(\theta)\vert \geq \vareps \right) \leq \sqrt{2\pi}\exp(-d\vareps^2/2) \eqsp .
$$
Taking a union bound over the $n(n-1)$ pairs of indices and setting $\tau = (R\vareps)^2$, it follows
$$
\mathbb{P}\left( \max_{i,j} \vert \langle \theta, x_i-x_j\rangle\vert^2 \geq \tau \right) \leq \sqrt{2\pi}n^2\exp(-d\tau/2R^2) \eqsp .
$$
Hence, for any $\delta>0$, it holds with probability $1-\delta$ that $\max_{i,j} \vert \langle \theta, x_i-x_j\rangle\vert^2 \leq \frac{2R^2}{d}\log(\sqrt{2\pi}n^2/\delta)$.
 This argument was suggested to us by an anonymous reviewer.

\end{proof}

\section{Additional experimental results}

All of our experimental findings presented in this paper and its supplementary document can be reproduced with the code that we provided here: \url{https://github.com/kimiandj/sliced_div}. %

In this section, we provide additional results obtained for the synthetical experiments illustrating the sample complexity of Sliced-Wasserstein and Sliced-Sinkhorn divergences: we produce figures analogously to \Cref{fig:sample_cpx_sw,subfig:sample_cpx_ssink_a,subfig:sample_cpx_ssink_b}, with different hyperparameter values. \\[5mm]

\begin{figure*}[ht!]
\centering
  \subfigure[$L = 1$]{
    \includegraphics[width=.314\textwidth]{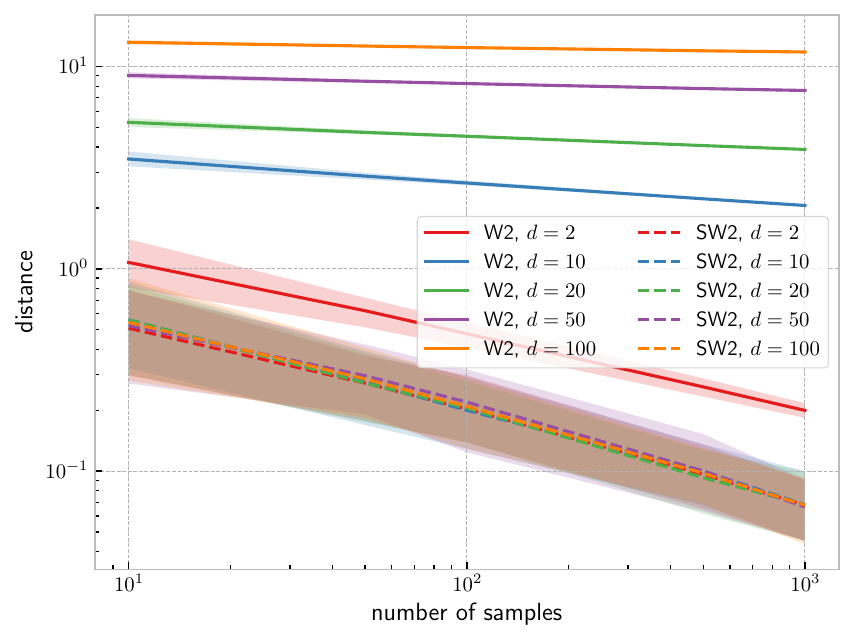}
  }
  \subfigure[$L = 10$]{
    \includegraphics[width=.314\textwidth]{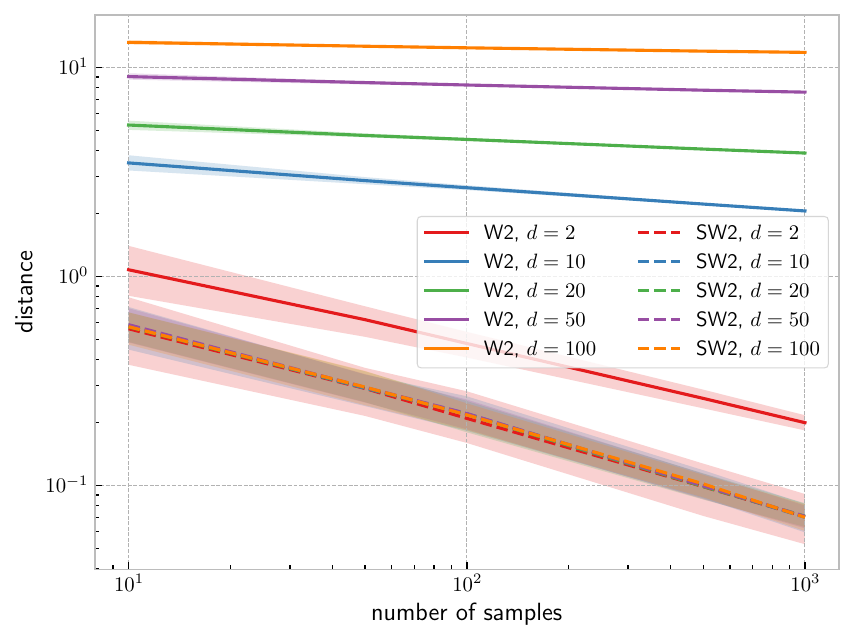}
  }
  \subfigure[$L = 1000$]{
    \includegraphics[width=.314\textwidth]{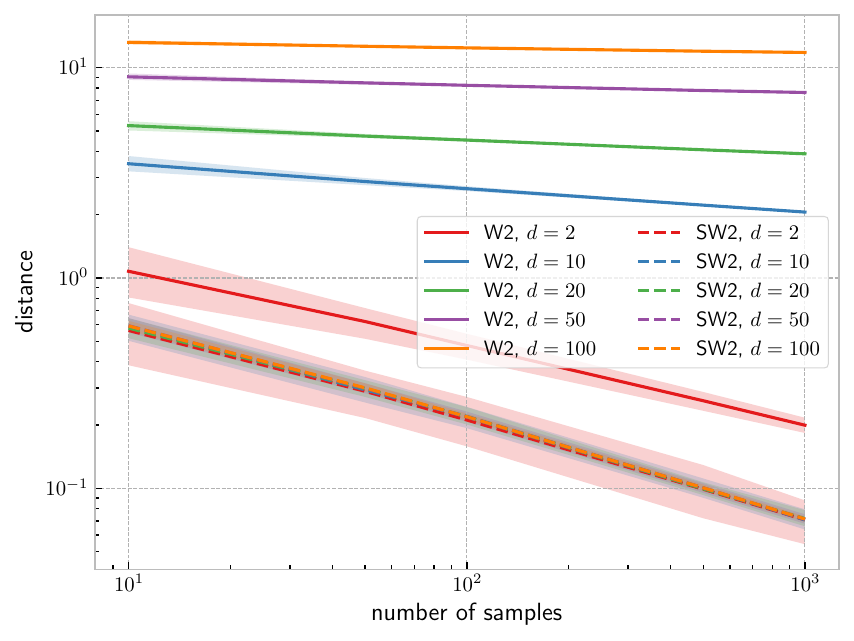}
  }
  \caption{Illustration of \Cref{thm:smpl_cpx_sw}: Wasserstein and Sliced-Wasserstein distances of order 2 between two sets of $n$ samples generated from $\calN({\bf0}, \bfI_d)$ vs. $n$, for different $d$, on log-log scale. $\swassersteinD[2]$ is approximated with $L$ random projections, $L \in \{1, 10, 1000 \}$. Results are averaged over 100 runs, and the shaded areas correspond to the 10th-90th percentiles. \Cref{fig:sample_cpx_sw} shows the results for $L = 100$. }
  \vspace{1cm}
\end{figure*}

\begin{figure*}[ht!]
\centering
  \subfigure[Influence of the data dimension for $\veps \in \{0.05, 10, 100\}$]{
    \includegraphics[width=.314\textwidth]{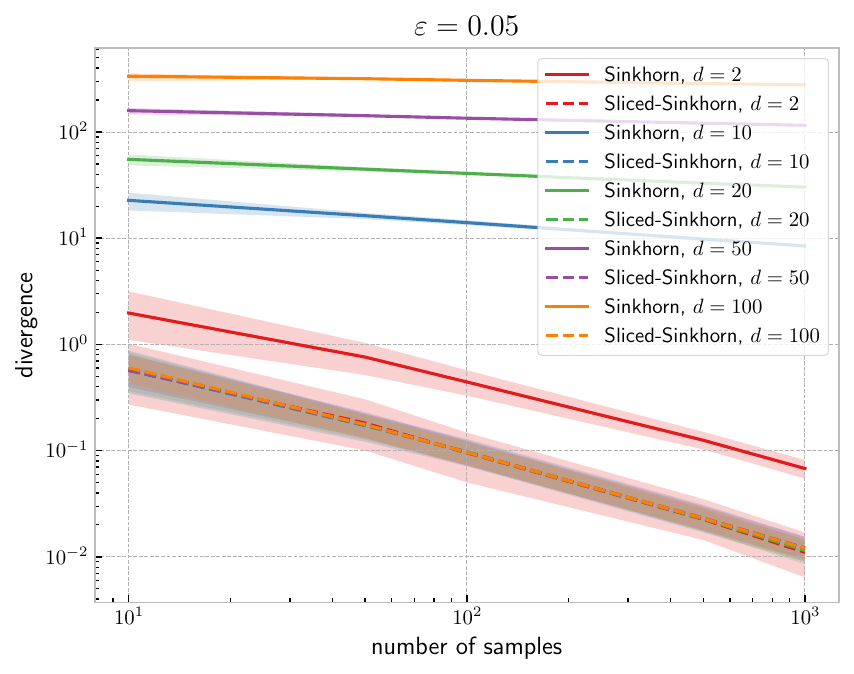}
    \includegraphics[width=.314\textwidth]{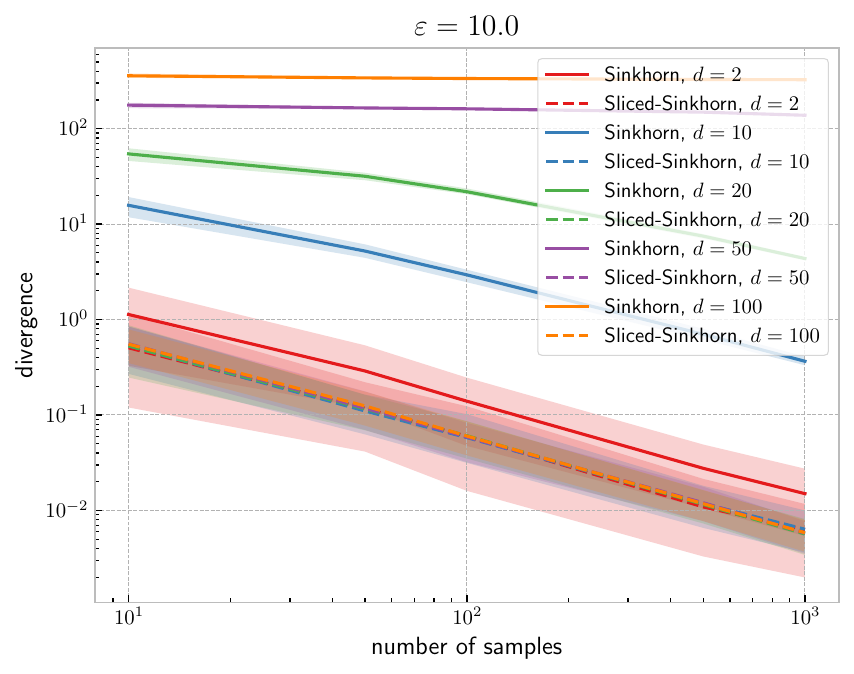}
    \includegraphics[width=.314\textwidth]{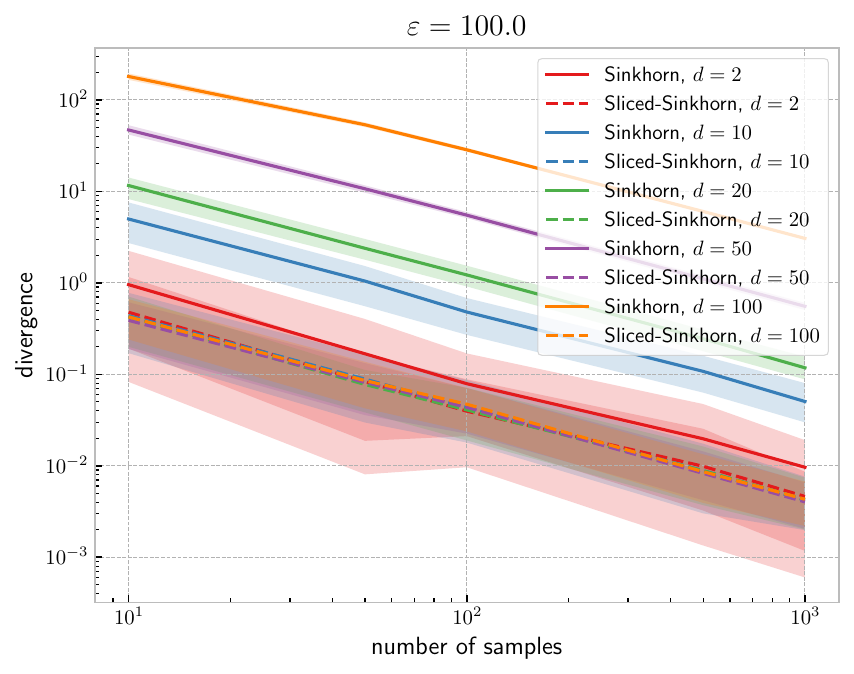}
  }
  \subfigure[Influence of the regularization coefficient for $d \in \{ 2, 10, 50 \}$]{
    \includegraphics[width=.314\textwidth]{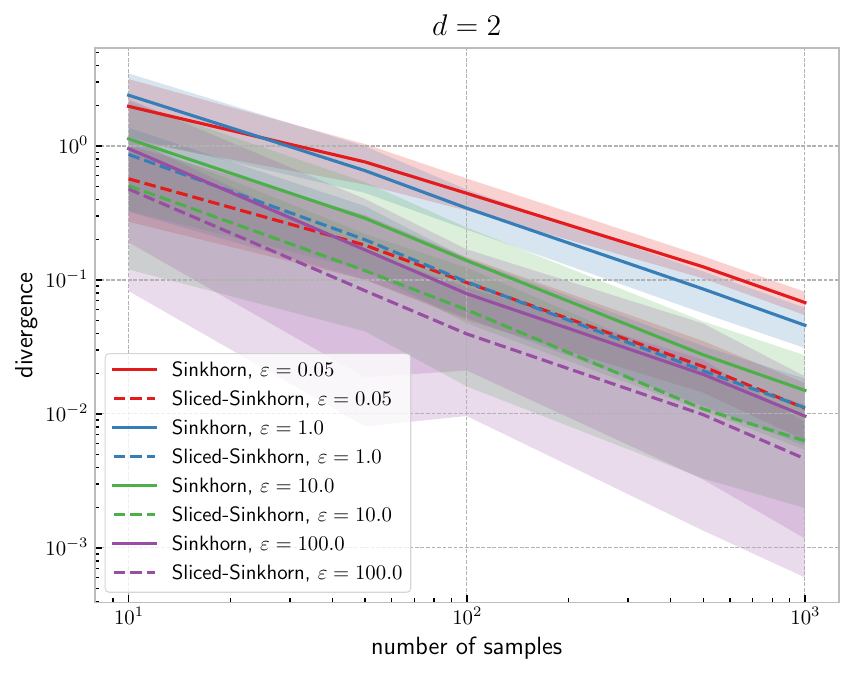}
    \includegraphics[width=.314\textwidth]{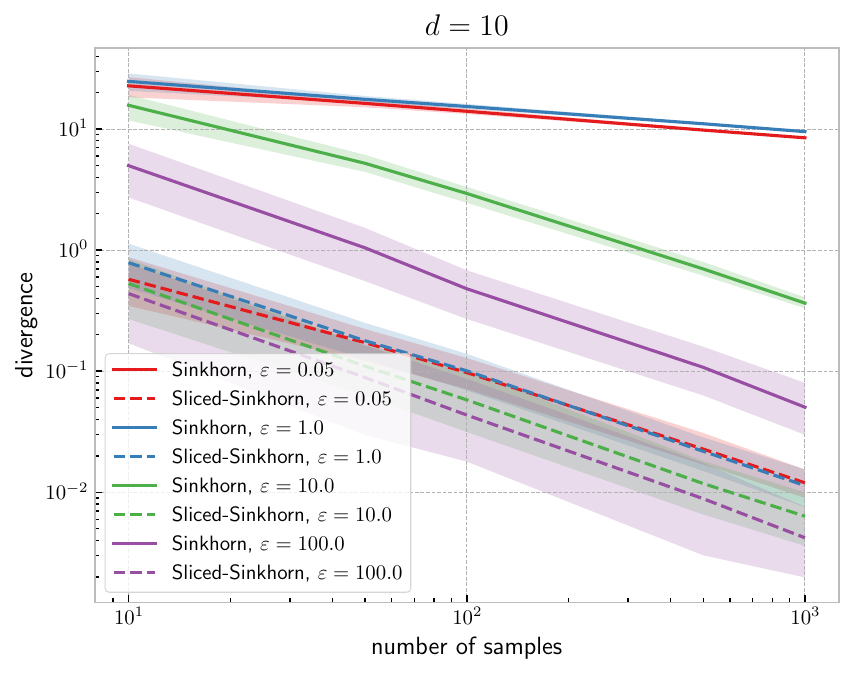}
    \includegraphics[width=.314\textwidth]{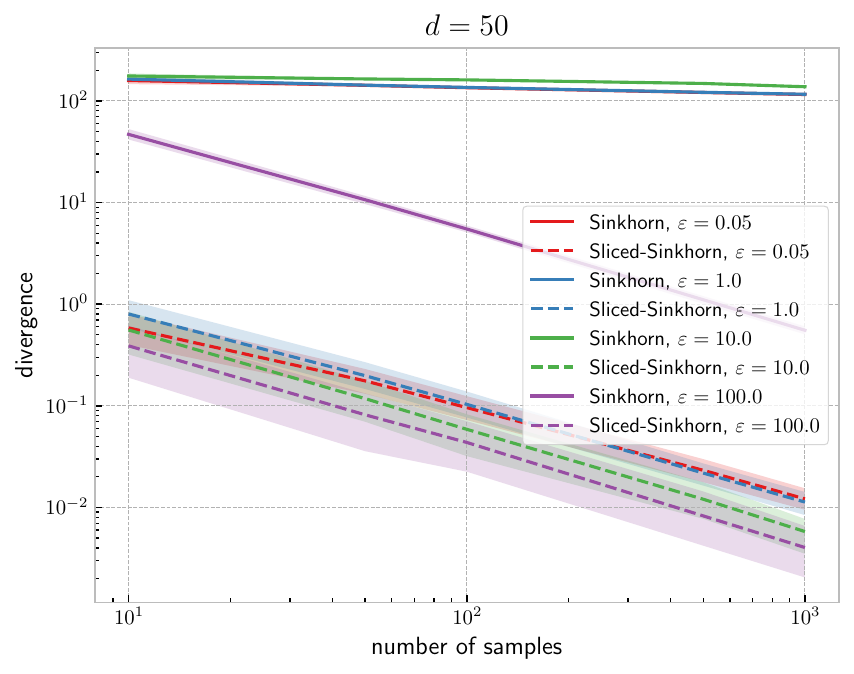}
  }
  \caption{Illustration of \Cref{thm:smpl_cpx_ssink}: Sinkhorn and Sliced-Sinkhorn divergences between two sets of $n$ samples generated from $\calN({\bf0}, \bfI_d)$ for different values of $n$, dimension $d$, and regularization coefficient $\veps$. Sliced-Sinkhorn is approximated with 10 random projections. Results are averaged over 100 runs, and the shaded areas correspond to the 10th-90th percentiles. All plots have a log-log scale. \Cref{subfig:sample_cpx_ssink_a} shows the influence of the dimension for $\veps = 1$, and \Cref{subfig:sample_cpx_ssink_b} shows the influence of the regularization for $d = 100$.}
\end{figure*}

\end{document}